\documentclass{article}

\usepackage{microtype}
\usepackage{graphicx}
\usepackage{booktabs} % for professional tables
\usepackage{changepage} % for adjustwidth
\usepackage{subcaption}
\usepackage{float} % in preamble
\usepackage{dirtytalk}

\usepackage{hyperref}

\usepackage{icml2026}
\icmlshowauthorstrue
\usepackage{amsmath}
\usepackage{amssymb}
\usepackage{mathtools}
\usepackage{amsthm}

\usepackage[capitalize,noabbrev]{cleveref}

\newtheoremstyle{icmlprop}
{3pt}{3pt}{\itshape}{}{\bfseries}{}{.5em}{}
\theoremstyle{icmlprop}

\newtheorem{theorem}{Theorem}[section]
\newtheorem{proposition}[theorem]{Proposition}

\newtheorem{definition}[theorem]{Definition}

\newtheorem{property}{Property}

\usepackage[textsize=tiny]{todonotes}

\icmltitlerunning{Reward-Preserving Attacks For Robust Reinforcement Learning}

\usepackage{comment}

\begin{document}

\twocolumn[
\icmltitle{Reward-Preserving Attacks For Robust Reinforcement Learning}

% unless you've provided the [accepted] option to the icml2026



\begin{icmlauthorlist}
\icmlauthor{Lucas Schott}{isx,mlia}
\icmlauthor{Elies Gherbi}{isx}
\icmlauthor{Hatem Hajri}{safran}
\icmlauthor{Sylvain Lamprier}{leria,mlia}
\end{icmlauthorlist}

\icmlaffiliation{isx}{IRT SystemX, Palaiseau, France}
\icmlaffiliation{mlia}{MLIA, ISIR, Sorbonne Université, Paris, France}
\icmlaffiliation{safran}{Safran Electronics and Defense, Palaiseau, France}
\icmlaffiliation{leria}{LERIA,  Université d'Angers, France}

\icmlcorrespondingauthor{Lucas Schott}{lucas.schott@irt-systemx.fr}
\icmlcorrespondingauthor{Sylvain Lamprier}{sylvain.lamprier@universite-angers.fr}

\icmlkeywords{Robust Reinforcement Learning, Adversarial Training}

\vskip 0.3in
]



\printAffiliationsAndNotice{}  % leave blank if no need to mention equal contribution

\begin{abstract}
    Adversarial training in reinforcement learning (RL) is challenging because perturbations cascade through trajectories and compound over time, making fixed-strength attacks either overly destructive or too conservative. We propose reward-preserving attacks, which adapt adversarial strength so that an $\alpha$ fraction of the nominal-to-worst-case return gap remains achievable at each state. In deep RL, perturbation magnitudes $\eta$ are selected dynamically, using a learned critic $Q((s,a),\eta)$ that estimates the expected return of $\alpha$-reward-preserving rollouts. For intermediate values of $\alpha$, this adaptive training yields policies that are robust across a wide range of perturbation magnitudes while preserving nominal performance, outperforming fixed-radius and uniformly sampled-radius adversarial training.
\end{abstract}

\section{Introduction}

    Adversarial attacks in machine learning refer to deliberately crafted perturbations, that are designed to cause learned models to fail in unexpected or worst-case ways \citep{goodfellow2014explaining}. These attacks reveal fundamental limitations in the generalization and stability of neural networks, motivating a wide range of defenses and training procedures collectively known as \emph{adversarial robustness} \citep{madry2017towards}. While the phenomenon first received attention in supervised image classification, its implications extend far beyond static prediction problems, particularly to sequential decision-making systems such as reinforcement learning (RL) \cite{schott2024robust}.

    When attacking a classifier, one is typically interested in inducing a \emph{pointwise divergence} in its prediction, while remaining undetectable. This makes it possible to specify attack strengths in a relatively simple way, using classical perturbation norms (e.g., $L_0$, $L_1$, $L_2$, or $L_\infty$), which produce imperceptible modifications, at least to a human observer, and do not alter the semantic content of the input \citep{goodfellow2014explaining,carlini2017towards}. In adversarial training for classifiers, such attack strength can moreover be globally tuned, taking into account the model’s current ability to defend itself.

    In the context of adversarial training for \emph{robust RL}, the situation becomes substantially more complex. Beyond a pointwise divergence, the objective is long-term: the agent seeks to \emph{maximize cumulative reward} despite perturbations occurring along entire trajectories. If perturbations become too strong, the agent may be unable to solve the task at all, thereby losing the feedback signal required to improve. More subtly, the agent may remain trapped in a suboptimal region of the state space without a means of escaping it.  
    Beyond these learning difficulties, certain attacks may even make the task unsolvable \emph{even for an optimal policy adapted to the attack}. Conversely, if perturbations are too weak, they fail to induce meaningful distribution shifts or information loss, and adversarial training no longer promotes meaningful robustness \citep{morimoto2005robust,huang2017adversarial,pinto2017robust}.
    %\footnote{An analogous phenomenon appears in trust-region optimization methods such as TRPO \citep{schulman2015trust}, where same update magnitudes can cause catastrophic forgetting in some regions of the parameter space, while remaining overly conservative in others.}

    \begin{figure*}[ht]
        \centering
        \includegraphics[width=0.9\textwidth]{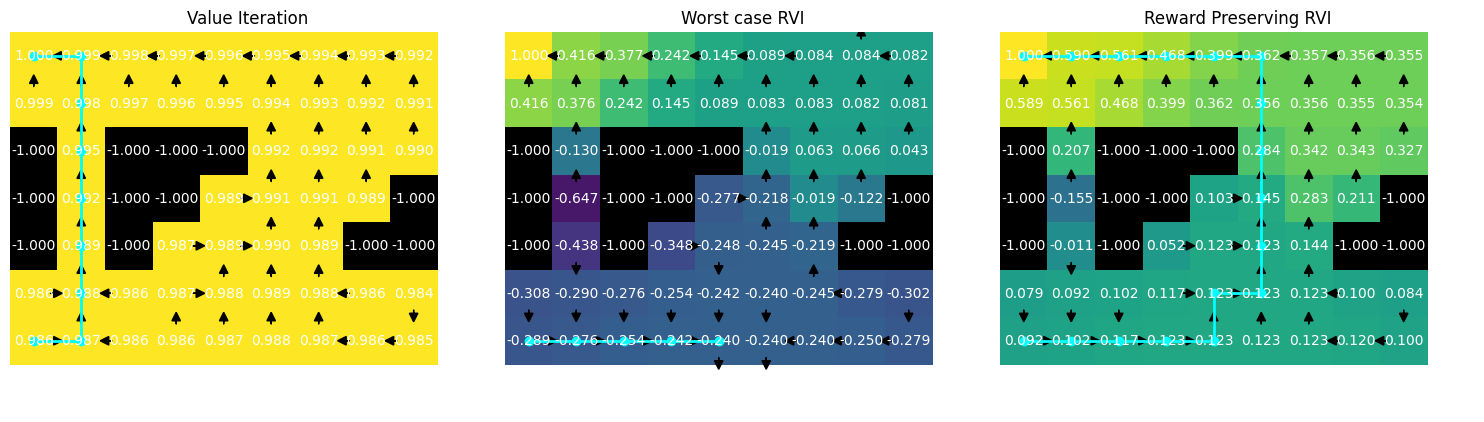}
        \caption{
            Comparison of value functions and induced optimal trajectories on a deterministic GridWorld environment for: 
            \emph{(left)} classical Value Iteration, 
            \emph{(middle)} Robust Value Iteration (RVI), and 
            \emph{(right)} our $\alpha$-reward-preserving extension of RVI ($\alpha=0.3$). 
            The environment contains a single positively rewarded goal state located in the top-left corner, while black cells correspond to terminal states with reward $-1$. A discount factor $\gamma = 0.999$ is used.  
            Robustness is enforced through an uncertainty set $\mathcal{B}$ over transition kernels, where the ambiguity radius $\eta_{\cal B}$ is computed via a Sinkhorn-regularized $W_2$ transportation cost between next-state distributions, with ground costs defined as Euclidean distances between successor states. Implementation details are provided in Appendix~\ref{sec:gridworld}.
            }
        \label{fig:preserve_RVI}
    \end{figure*}

Consider a simple illustrative example involving a narrow bridge, as depicted in the toy gridworld environment from figure \ref{fig:preserve_RVI}. In regions of the environment where many nearby states share similar %long-term
    futures, one may remove nearly all information from the observation vector, or induce strong perturbations of specific components or dynamics, without destroying the agent's ability to recover. However, if the observation or dynamics are  fully corrupted \emph{precisely} when the agent is on the narrow bridge, %and the agent receives no auxiliary signal about its true position, 
    then crossing becomes virtually impossible, even with unlimited training. % (the agent's “vision” is disrupted at the most critical moment).  
    %This reveals two qualitatively different situations:  
    %(1) If crossing the critical point (the bridge) can be avoided via an alternative path with only a small reduction in cumulative reward, a robust strategy should redirect the agent away from the dangerous region, which can be encouraged by attacking specifically at the beginning of the bridge.  
    %(2) If crossing the bridge is \emph{mandatory} for solving the task, then perturbations at that location must remain weak enough not to destroy the existence of a viable recovery strategy; otherwise, adversarial training collapses entirely.
    %This example is illustrated in Figure \ref{fig:preserve_RVI}, which compares classical  Value Iteration (VI), Robust Value Iteration (RVI) \cite{wiesemann2013robust} and our $\alpha$-reward-preserving extension of RVI on a deterministic GridWorld environment with attacks on the dynamics. This ewample will be discussed in details in Section \ref{sec:rew_preserve_attack}.
This highlights two qualitatively different situations: (1) if the critical region can be avoided via an alternative path with only a small reduction in cumulative reward, a robust policy should steer the agent away from danger, which can be encouraged by targeted attacks at the approach to the bridge; (2) if crossing the bridge is \emph{mandatory}, perturbations must remain weak enough to preserve a viable recovery strategy, otherwise adversarial training fails.  Motivated by this, we aim at ensuring %focus on
two complementary forms of robustness in RL policies: \emph{local robustness}, the ability to recover from %misleading
perturbations while maintaining long-term performance; and \emph{global robustness}, the ability to favor safer over excessively risky trajectories without invalidating the task. %To achieve this, we introduce $\alpha$-reward-preserving attacks and foundational components for deep robust RL training. %Experiments demonstrate that our approach yields policies resilient to diverse perturbations in deployment.

%To achieve this, we propose a new approach for robust reinforcement learning, that dynamically adjusts attack magnitudes based on estimated criticallity of the states, with the aim to preserve solvability of the task from any situation.  After presenting background on adversarial training in RL, Section \ref{sec:rewpreserv_intro} introduces our $\alpha$-reward-preserving attacks, and analyzes them in the tabular setting. Then, it presents foundational components for the deep robust RL training described in Section \ref{sec:deep_approx}. Finally, experiments in Section \ref{sec:xps} demonstrate the effectiveness of our approach in training policies that are resilient to various perturbations in their deployment environment.    

To achieve this, we propose a new approach for robust reinforcement learning that dynamically adjusts attack magnitudes based on an estimate of state criticality, with the aim of preserving task solvability from any situation. After giving background on adversarial training in RL, Section~\ref{sec:rewpreserv_intro} introduces our $\alpha$-reward-preserving attacks and analyzes them in the tabular setting. %We then present foundational components for the deep robust RL training described in 
Section~\ref{sec:deep_approx} extends  our $\alpha$-reward-preserving training approach for the deep RL setting.  Finally, experiments in Section~\ref{sec:xps} demonstrate the effectiveness of our approach in training policies that are resilient to diverse perturbations in their deployment environment.

%we propose a new approach for robust reinforcement learning, based on the so-called $\alpha$-reward-preserving attacks we introduce. After giving some background on robust RL via adversarial training, Section  \ref{sec:rewpreserv_intro} introduces and analyzes our reward preserving attacks in the tabular setting. Next, it introduces building components that serve as foundations for our deep robust RL training described in section \ref{sec:deep_approx}. Experiments reported in section \ref{sec:xps} finally demonstrate the relevance of the approx in training policies that are prepared to various perturbations when deployed in their target environment.      

\section{Toward Robust RL with Reward-Preserving Attacks}
\label{sec:rewpreserv_intro}
    
    \subsection{Robust RL through Adversarial Training}
    
        In this work, we consider Markov Decision Process $\Omega = (S, A, T, R, X, O, \gamma)$, where: $S$ is the set of states in the environment; $A$ is the set of actions available to the agent; $T : S \times A \times S \rightarrow [0,1]$ is the stochastic transition function, with $T(s_+|s,a)$ denoting the probability of transitioning to state $s_+$ given state $s$ and action $a$; $R : S \times A \times S \rightarrow \mathbb{R}$ is the reward function, with $R(s_t,a_t,s_{t+1})$ the reward received by the agent for taking action $a_t$ in state $s_t$ and moving to state $s_{t+1}$; $X$ is the set of observations as perceived by the agent; $O : S \times X \rightarrow [0,1]$ is the observation function, with $O(x|s)$ denoting the probability of observing $x \in X$ given state $s$. For simplicity, we consider in the following that $O$ is information preserving (i.e., it is a deterministic bijective mapping from $S$ to $X$),  while it could be extended to partially observable problems (POMDP). We also consider that any $x \in X$ is a vector of $k$ real values $\phi(s)$, with $\phi$ the mapping from $s \in {\cal S}$ to its corresponding encoding vector. 
        
        In $\Omega$, we consider policies $\pi : S \times A \rightarrow [0,1]$, where $\pi(a|s)$ denotes the probability of selecting action $a$ given state $s$. The classical goal in RL is to discover policies that maximize the expected cumulative discounted reward $\mathbb{E}_{\tau \sim \pi^\Omega} [\sum_{t=0}^{|\tau|-1} \gamma^{t}R(s_t, a_t, s_{t+1})]$, with $\tau=\big((s_0,a_0),(s_1,a_1),...,(s_{|\tau|},\textunderscore{})\big)$ a trajectory in $\Omega$ and $\pi^\Omega(\tau)$ the probability of $\tau$ in $\Omega$ when using policy $\pi$ for selecting actions at each step. The discount factor $\gamma \in ]0;1[$ weights the importance of future rewards. We note $V^{\pi, \Omega}(s)=\mathbb{E}_{\tau \sim \pi^\Omega} [\sum_{t=0}^{|\tau|-1} \gamma^{t}R(s_t, a_t, s_{t+1}) | s_0=s]$ the value function associated to $\pi$ for rollouts starting from $s_0$,  $\pi^{*,\Omega}$ the optimal policy in $\Omega$ and $V^{*, \Omega}(s)$ its associated value function. Similarly, $Q^{\pi, \Omega}(s,a)=\mathbb{E}_{\tau \sim \pi^\Omega} [\sum_{t=0}^{|\tau|-1} \gamma^{t}R(s_t, a_t, s_{t+1}) | s_0=s, a_0=a]$ is the state-action value function for $\pi$ in $\Omega$.
        
        For achieving robust agents, adversarial training introduces adversarial agents (or attackers) $\xi$, 
        which produce perturbations for the situations encountered by the protagonist agent $\pi$. 
        The adversarial actions taken by $\xi$ are denoted 
        $A^{\xi,\Omega}=(a_i^{\xi,\Omega})_{i=1}^k \in \prod_{i=1}^k {\cal A}_i^{\xi,\Omega}$, 
        where each $a_i^{\xi,\Omega}$ corresponds to a component (e.g., a parameter of the MDP) 
        and ${\cal A}_i^{\xi,\Omega}$ represents the set of values that the environment $\Omega$ allows for this component.
        Usually, the goal is to consider worst-case distributions in an uncertainty set ${\cal B}$, to promote robustness of the protagonist agent $\pi$. The radius $\eta_{\cal B}$ of ${\cal B}$ is defined according to a given metric (e.g., in term of f-divergence or in term of $L_1$, $L_2$ or $L_\infty$  distances), either on the parameters of the nominal MDP $\Omega$ or directly on its distributions, and determines  the maximal  magnitude of perturbations allowed inside ${\cal B}$. 
        The adversarial training setting can then be formalized as: 
        \begin{align}
            \label{obj_robusttl}
            \pi^* = \arg\max_\pi
            \mathbb{E}_{\tau \sim \pi^{\Omega^{\xi^*}}}[R(\tau)] \nonumber \\
            s.t.\qquad \xi^*=\arg\min_{\xi \in {\cal B}}   \Delta^{\pi,\Omega}(\xi)
        \end{align}
        where  we note $\Omega^\xi$ the environment under attacks from adversary $\xi$. By a slight abuse of notation, $\xi \in {\cal B}$ stands as an attacker that is constrained to produce only perturbed distributions within ${\cal B}$. $\Delta^{\pi,\Omega}(\xi)$ stands as the optimization objective of the adversarial agent given $\pi$ and the training environment $\Omega$. In the following, we consider the setting where the attacker aims at minimizing the cumulative discounted reward:  $\Delta^{\pi,\Omega}(\xi)=\mathbb{E}_{\tau \sim \pi^{\Omega^{\xi^*}}}[R(\tau)]$, while our proposal could be applied for other settings. Given a worst-case objective, the attacker can act on any component of the MDP, ranging from the observation function to the transition kernel.  When acting on the transition function, the Robust Value Iteration algorithm is shown to converge asymptotically toward the optimal policy under the optimal worst-case attack for the tabular setting \citep{wiesemann2013robust}.

    \subsection{Reward-Preserving Attacks}
        \label{sec:rew_preserve_attack}
        
        From formulation in \eqref{obj_robusttl}, we note that the shape of the uncertainty set ${\cal B}$ is highly impactful regarding the optimal policy under $\Omega^{\xi^*}$. It is known that sa-rectangular uncertainty sets (i.e., ${\cal B}$ allows independent perturbations for each state-action pair of the MDP) usually leads to very hard attacks, inducing too difficult or even impossible problems $\Omega^{\xi^*}$, which in turn results in too conservative policies \cite{zouitine2024time}. Hence, many approaches propose to regularize the attacks inside ${\cal B}$, with constrained perturbations. For instance, Active Domain Randomization \cite{mehta2020active} propose to discover the worst-case global parametrization of the MDP rather than relying on sa-rectangular attacks. However, this induces difficult worst-case identification in subsets of ${\cal B}$ that are no longer convex w.r.t. the MDP parametrization (e.g., implying convex hull approximation). Other approaches propose to apply regularization through time (e.g., as in \citep{zouitine2024time}), to act on the timing of the attacks \citep{fan2025less}, or to consider restrictions over the radius of ${\cal B}$ to lower the power of the resulting attacks \citep{ma2018improved}\footnote{Further discussion about related work is given in appendix \ref{sec:related_work}.}. However, these approaches do not allow for local adaptation of attacks regarding the current state area of the agent. We claim this is crucial for effectively coping with both kinds of robustness mentioned in the introduction. To close that gap,  we propose to introduce a new kind of "reward-preserving" attacks, as defined below.     
        
        \begin{definition}[Reward-Preserving Attack]
            \label{def_preserv}
            Given a MDP $\Omega$ and an uncertainty set ${\cal B}$ in that MDP, an attack $\xi \in {\cal B}$ on this MDP is said $\alpha$-Reward-Preserving for a state $s \in {\cal S}$ and an action $a \in {\cal A}(s)$ iff there exists an optimal policy $\pi^{*,\Omega^\xi}$ adapted to $\xi$ such that: $Q^{*,\Omega^\xi}(s,a) \geq $
            %\begin{equation*}
             $Q^{*,\Omega^{\xi^*}}(s,a) + \alpha \left( Q^{*,\Omega}(s,a) - Q^{*,\Omega^{\xi^*}}(s,a) \right)$, 
            %\end{equation*}
            where $Q^{*,\Omega^{\xi^*}}(s,a)$ stands as the value at $s,a$ using 
            the optimal policy $\pi^{*,\Omega^{\xi^*}}$ against the optimal attacker $\xi^*$ for $\Omega$. 
            
            The set of $\alpha$-Reward-Preserving attacks for a state-action pair $(s,a)$ is noted $\Xi_\alpha(s,a)$. 
            An attack $\xi$ is also said $\alpha$-Reward-Preserving for an MDP $\Omega$ iff for any state-action pair $(s,a)$ of $\Omega$, $\xi \in \Xi_\alpha(s,a)$. The set of  $\alpha$-Reward-Preserving attacks for an MDP $\Omega$ is denoted as $\Xi_\alpha(\Omega)$. For any pair $(s,a)$, %state $s$ % \in {\cal S}$
            %and $a$, % \in {\cal A}(s)$, 
            we have $\Xi_\alpha(\Omega) \subseteq \Xi_\alpha(s) \subseteq {\cal B}$.    
        \end{definition}

        That is, an $\alpha$-reward-preserving attack $\xi_\alpha \in \Xi_\alpha(\Omega)$ is an attack that guarantees that a proportion $\alpha$ of the gap between the best expected cumulative reward in the original MDP $\Omega$ and the one in the worst-case modified MDP in ${\cal B}$ remains reachable in the  resulting perturbed MDP $\Omega^{\xi_\alpha}$. Importantly, this definition differs fundamentally from a convex mixture $\alpha\,\Omega + (1-\alpha)\,\Omega^{\xi^*}$, which merely constrains $\Omega^\xi$ to stay close to the nominal MDP $\Omega$, and offers no guarantee that a fixed proportion of the reward gap is preserved at each state--action pair.  Our definition, instead, constrains the optimal Q-values, effectively combining the policies of the nominal and worst-case MDPs at the value-function level—an objective that cannot generally be realized by simple MDP interpolation.

        Given a policy $\pi$, the worst case $\alpha$-reward-preserving attack $\xi_\alpha^{*,\pi}$ is defined for each state $s$ and action $a$ as belonging to: 
        \begin{equation}
            \Xi_\alpha^{*,\pi}(s,a) = \arg\min_{\xi \in \Xi_\alpha(s,a)} Q^{\pi, \Omega^\xi}(s,a)
        \end{equation}
        We also note $\Xi_\alpha^{*,*}(s,a)$ the set of optimal worst-case $\alpha$-reward-preserving attacks set for an optimal policy against them: $\xi_\alpha^{*}(s,a) = \arg\min_{\xi \in \Xi_\alpha(s,a)} \max_\pi Q^{\pi, \Omega^\xi}(s,a)$, and $Q^{*,\Omega^{\xi_\alpha^{*}}}(s,a)$ the associated value function for the corresponding optimal policy. 
        Note that $\Xi_\alpha^{*,*}(s,a)$ is typically not convex in the space of admissible attacks, even under SA-rectangularity, because the mapping $\xi \mapsto Q^{*,\Omega^\xi}(s,a)$ is highly non-linear in the perturbed MDP.  
        As a result, convex combinations of two attacks in $\Xi_\alpha^{*,*}(s,a)$ do not generally satisfy the $\alpha$-reward-preserving constraints. 
        
        %From the above definitions, it is clear that extending the classical Robust Value Iteration \citep{wiesemann2013robust} algorithm to handle $\alpha$-reward-preserving uncertainty sets for dynamics attacks in the tabular setting with a known MDP is not trivial, since it requires optimizing attacks under admissibility conditions that depend on the optimal policies adapted to them. A naive approach might attempt to define the worst-case $\alpha$-reward-preserving Q-values for each state-action pair as $\hat{Q}(s,a) := Q^{*,\Omega^{\xi^*}}(s,a) + \alpha \bigl( Q^{*,\Omega}(s,a) - Q^{*,\Omega^{\xi^*}}(s,a) \bigr)$, using precomputed values $Q^{*,\Omega^{\xi^*}}(s,a)$ from robust value iteration and $Q^{*,\Omega}(s,a)$ from standard value iteration. However, this construction is generally incorrect: the resulting $\hat{Q} (s,a)$ values do not satisfy the Bellman optimality equations for any MDP. In other words, simply taking a linear interpolation between the nominal and worst-case Q-values does not correspond to the optimal value function of any realizable perturbed MDP, and therefore cannot be used directly within a Bellman iteration to build a robust policy properly. Still, section \ref{sec:preservingRVI} discusses this point further and provides an extension attempt of RVI for $\alpha$-reward-preserving attacks, albeit with weaker guarantees, which serves as a foundation for our general approach described in section \ref{sec:deep_approx}.

        Extending classical Robust Value Iteration (RVI) \citep{wiesemann2013robust} to $\alpha$-reward-preserving uncertainty sets in the tabular setting is non-trivial, as it requires optimizing admissible attacks whose constraints depend on the optimal policies adapted to them. A naive approach would define worst-case $\alpha$-reward-preserving Q-values as
\[
\hat{Q}(s,a) := Q^{*,\Omega^{\xi^*}}(s,a) + \alpha \bigl( Q^{*,\Omega}(s,a) - Q^{*,\Omega^{\xi^*}}(s,a) \bigr),
\]
interpolating between nominal and worst-case solutions. However, such $\hat{Q}$ generally fails to satisfy the Bellman optimality equations for any realizable MDP, and therefore cannot be used directly within a Bellman iteration. Section~\ref{sec:preservingRVI} discusses this limitation and presents an extension of RVI with weaker guarantees, which serves as a foundation for our deep robust approach in Section~\ref{sec:deep_approx}.

        %We also remark the Reward Structure Preservation property of  worst-case $\alpha$-reward-preserving attacks for sufficiently large ${\cal B}$, as formalized below (and demonstrated in section \ref{sec:properties}). 

Worst-case $\alpha$-reward-preserving attacks also preserve the reward structure for sufficiently large $\mathcal{B}$ (Section~\ref{sec:properties}).

        \begin{property} \textbf{Reward Structure Preservation} \\
            Given a sufficiently large uncertainty set ${\cal B}$, $Q^{*, \xi^*}$ is equal to a given constant minimal value $Rmin$ for every state $s \in {\cal S}$  and action $a \in {\cal A}(s)$  (i.e., the worst-case attacks fully destroy the reward signal). In that setting, worst-case  $\alpha$-reward-preserving attacks preserve the structure of the reward:
            \begin{align}
                \forall ((s,a),(s',a')) \in ({\cal S}\times {\cal A})^2: Q^{*,\Omega}(s,a) > \nonumber \\
                Q^{*,\Omega}(s',a') \implies Q^{*,\Omega^{\xi_\alpha^{*}}}(s,a) > Q^{*,\Omega^{\xi_\alpha^{*}}}(s',a')
            \end{align}
        \end{property}

Thus, for sufficiently large $\mathcal{B}$, there exists an optimal policy for $\Omega^{\xi^*_\alpha}$ that coincides with the optimal policy of the nominal MDP $\Omega$. In non-tabular settings with stochastic neural policies, %this suggests that 
agents can acquire \emph{local robustness} with such attacks, by learning to recover from complex or misleading situations without biasing the nominal optimal policy (which may occur with classical uncertainty sets). Under observation attacks, this manifests as learning to denoise inputs and reduce sensitivity to isolated perturbations, while under dynamics attacks it encourages actions that keep the agent away from risky states — such as the edge of a bridge —, even when transitions are manipulated. Importantly, such robustness is achieved without altering the optimal policy path, as recovery strategies remain feasible.

However, this does not address \emph{global robustness}: preserving the reward structure does not enable the agent to prefer safer paths over riskier ones when the latter yield higher nominal returns. In contrast, such preferences can be recovered in classical settings with constrained uncertainty sets $\mathcal{B}$, as formalized by the following property (Section~\ref{sec:properties}).

        %However, this does not allow to address \emph{global robustness}. 
        %In particular, %when multiple trajectories incur at most an $1-\alpha$ loss in cumulative reward, 
%preserving the global reward structure %alone 
%does not allow the agent to prefer safer paths (e.g., a wider bridge) over riskier ones, whenever the latter provide higher  returns in the nominal MDP. %However, this does not address the second type of robustness discussed in the introduction, namely global robustness.
        %For instance, even if there exists a path that crosses a wider and therefore safer bridge, enforcing attacks that preserve the global reward structure does not enable the agent to prefer this safer path, as long as both paths incur at most an $\alpha$ loss in cumulative reward.
        %In contrast, i
 %       Nonetheless, in the more classical setting where the uncertainty set ${\cal B}$ is constrained, such global preferences can be recovered, as discussed through the following property for decision change under worst-case $\alpha$-reward-preserving attack (proof given in section \ref{sec:properties}).

\begin{property}[Condition for Preferred State--Action Change] \label{prop:change}
Consider two state--action pairs $(s,a)$ and $(s',a')$ such that
\[
d_\Omega\big((s,a),(s',a')\big)
:= Q^{*,\Omega}(s,a) - Q^{*,\Omega}(s',a') > 0 .
\]
Under a worst-case $\alpha$-reward-preserving attack $\xi^*_\alpha$ defined for a given  uncertainty set ${\cal B}$, the preference between $(s,a)$ and $(s',a')$ is reversed if and only if: $d_{\Omega^{\xi^*}}\big((s',a'),(s,a)\big)
> \frac{\alpha}{1-\alpha}\,
d_\Omega\big((s,a),(s',a')\big)
+ \delta\big((s',a'),(s,a)\big)$, 
where %$d_{\Omega^{\xi^*}}((s',a'),(s,a)) := Q^{*,\Omega^{\xi^*}}(s',a') - Q^{*,\Omega^{\xi^*}}(s,a)$ and
$\delta((s',a'),(s,a)) := (\epsilon_{s',a'}-\epsilon_{s,a})/(1-\alpha)$, with $\epsilon_{s,a}$ denotes the gap between $Q^{*,\Omega^{\xi^*_\alpha}}(s,a)$ and its $\alpha$-reward-preserving lower bound
$\hat Q(s,a) := (1-\alpha) Q^{*,\Omega^{\xi^*_\alpha}}(s,a) + \alpha Q^{*,\Omega}(s,a)$.  
%Moreover, $\delta((s',a'),(s,a)) = \mathcal{O}(\eta_{\mathcal B})$, and vanishes as the uncertainty set $\mathcal B$ shrinks.
While $\delta((s',a'),(s,a)) \to 0$ as $\eta_{\mathcal B}\to 0$, the actual variation of $Q^{*,\Omega^{\xi^*}}(s,a)$ can be amplified by local gaps in successor actions, so $\delta$ variations may be dominated by the effective sensitivity of $Q$ in “dangerous” zones, which induce preference changes under $\alpha$-reward-preserving attacks. 
\end{property}

        In other words, for any given state $s$, the optimal action changes from $a$ to $a'$ under a worst-case $\alpha$-reward-preserving attack whenever the resulting increase in worst-case performance outweighs a proportion $\frac{\alpha}{1-\alpha}$ of the nominal performance loss (assuming that both $Q^{*,\Omega^{\xi_\alpha^*}}$ are close enough to their respective bounds $\hat Q$).  
        This effect also propagates through distant states: $\alpha$-reward-preserving attacks may modify optimal trajectories over long horizons, as safer routes become preferable when their nominal performance loss is offset by improved worst-case robustness within the uncertainty set $\mathcal{B}$.

        We remark that $\alpha$ acts as a weighting factor that balances nominal performance against worst-case performance under attacks in $\mathcal{B}$. 
        When $\alpha < 0.5$, robustness to worst-case scenarios dominates the decision-making process, making worst-case performance more important than nominal performance. 
        Conversely, when $\alpha > 0.5$, the situation is reversed: nominal performance becomes more influential than worst-case robustness.

        This also highlights the importance of the shape of $\mathcal{B}$, as it directly influences the resulting behaviors. 
        Under $\alpha$-reward-preserving attacks with $\alpha > 0$, some amount of the nominal reward signal is always preserved, ensuring that efficient policies can still be learned. 
        However, since deviations toward safer trajectories are only allowed when they improve worst-case values within $\mathcal{B}$, global robustness is encouraged only when the uncertainty set is well structured. 
        In particular, the worst-case performance $Q^{*,\Omega^{\xi^*}}$ should vary on a scale comparable to nominal performance differences so that robustness incentives align with meaningful behavioral changes. To illustrate this, figure \ref{fig:preserve_RVI} compares Value Iteration (VI), Robust Value Iteration (RVI) and our $\alpha$-reward-preserving extension of RVI on a deterministic GridWorld environment with attacks on the dynamics. It highlights a setting, with rather large ${\cal B}$, where classical RVI exhibits strong risk aversion. It becomes excessively conservative and avoids the path to the goal. In contrast, our $\alpha$-reward-preserving extension reshapes the value landscape to downweight excessively pessimistic transitions, enabling the agent to reach the goal while still accounting for model uncertainty.

    \subsection{Magnitude–Direction Decomposition of Perturbations}
        \label{sec:decomposition}
            
        Following observations from the previous section, we need to control the shape of the uncertainty set $\mathcal{B}$ and the scope of the attacks in order to obtain an effective training approach under $\alpha$-reward preserving attacks. To do so, we propose to consider attacks of the nominal MDP $\Omega$ that decouple (i) the choice of the magnitude $\eta \leq \eta_{\cal B}$, where $\eta_{\cal B}$ stands for  the maximal magnitude allowed for attacks (which defines the radius of the convex set $\mathcal{B}(s,a)$ that includes all allowed perturbations of the attacked MDP component for state $s$ and action $a$), and (ii) the choice of the direction $A$ of the crafted perturbation for each state-action pair. 
        Following this, an attack is defined as $\xi_{\alpha} := ( \xi_{\alpha}^{\eta}, \, \xi^{A} \bigr)$, where:\newline
        $\bullet$ $\xi^{A} : {\cal S} \times {\cal A} \to {\cal A}^{\xi, \Omega}$ is the \emph{direction selector}, which, given a state $s$, and (optionally) an action $a$, produces a normalized perturbation direction over the parameters of the perturbed component.\newline
        $\bullet$ $\xi_\alpha^{\eta} : {\cal S} \times {\cal A} \to \mathbb{R}^+$ is the \emph{magnitude selector}, assigning to each state $s$ and (optionally) action $a$ a perturbation magnitude $\eta := \xi^{\eta}_{\alpha}(s,a)$ within $\mathbb{R}^+$. This selector aims to scale the attack so that its corresponding perturbation maintains the MDP component inside   $\Xi_\alpha(s,a) \subseteq {\cal B}(s,a)$, with ${\cal B}(s,a)$ the whole set of possible attacks for $(s,a)$ and  $\Xi_\alpha(s,a)$ the subset of those that are $\alpha$-reward-preserving. Thus, $\eta := \xi^{\eta}(s,a)$ stands as the radius of the convex core ${\cal B}_\alpha(s,a)$ of ${\Xi}_\alpha(s,a)$.

        Thus, for any state $s \in \mathcal{S}$ and action $a \in \mathcal{A}(s)$, an attack is fully specified by $\xi(s,a) = (\eta, A)$ with $\eta = \xi^\eta_\alpha(s,a)$ and $A = \xi^A(s,a)$, and acts on a perturbed MDP component $\omega$ as
$\hat{\omega}(s,a) \propto \omega(s,a) + \eta A.$ 
For dynamics attacks, taking $\mathcal{A}^{\xi,\Omega}=\Delta(\mathcal{S})$ and defining
$P_\xi(\cdot \mid s,a) \propto P^\Omega(\cdot \mid s,a) + \eta A(\cdot)$
yields a convex s--a rectangular uncertainty set $\mathcal{B}_\alpha(s,a)$ corresponding to the convex core of $\Xi_\alpha(s,a)$.  
Similarly, when perturbing a global parameter vector $\omega \in \mathbb{R}^d$ of the dynamics (e.g., some factors of physical forces in a simulator), we set $\omega=\omega_0+\eta A$ with $A \in \{-1,+1\}^d$, subject to admissible bounds.  Observation attacks can be handled analogously, e.g., by defining $O_\xi(s)=\delta(\phi(s)+\eta A)$ with $\|A\|_2=1$. In all cases, %$\eta$ is chosen as the maximal scale such that $\mathcal{B}_\alpha(s,a) \subseteq \Xi_\alpha(s,a)$, ensuring that 
$\mathcal{B}_\alpha(s,a)$ contains at least one boundary point 
of $\Xi_\alpha(s,a)$, corresponding to a worst-case attack in 
$\Xi_\alpha^{*,*}(s,a)$.

        In this work, we mainly focus on the definition of accurate magnitude selectors, as it is core for building $\alpha$-reward-preserving attacks. The direction selector can be defined from any given classical approach from the literature. For a given state-action pair $(s,a)$, we consider the $Q$ value from the attacker perspective, setting $\eta$ as its first action and following an $\alpha$-reward-preserving  attacker $\xi_\alpha \in \Xi_\alpha(\Omega)$ in the subsequent steps: 
        \begin{equation*}
            Q^{\pi}_\alpha((s,a), \eta)=\mathbb{E}_{\tau \sim \pi^{\Omega^{\xi_\alpha}}}[R(\tau)|s_0=s,a_0=a,\eta_0=\eta]
        \end{equation*}
        where $\eta_0$ corresponds to the magnitude of attack applied on the first state-action $s_0,a_0$ from the trajectory. At each state-action pair, the worst-case identification problem thus comes down at taking the highest $\eta$ that satisfies (2), then considering the worst-case direction within the corresponding ball. Specifically, we can note that all attacks $\xi_\alpha$ from ${\cal B}_\alpha(s,a)$ respect: $Q^*_\alpha((s,a),\xi_\alpha^\eta(s,a)) \geq (1-\alpha) Q^*_0((s,a),\eta_{\cal B}) + \alpha Q^*_1((s,a),0) $\footnote{We can remark that this is not necessarily true in $\Xi_\alpha(s,a)$}. Our general approach of approximated $\alpha$-preserving-attacks for deep reinforcement learning presented in next section builds on this magnitude-parametrized value function.

\section{Approximated Reward-Preserving Attacks for Robust Deep RL}
    \label{sec:deep_approx}
    
    %Leaving the tabular setting with known MDP for the more practical deep RL setting induces various difficulties to circumvent for the use of $\alpha$-preserving-attacks as defined in previous sections: 1) We do not have directly access to optimal policies and their corresponding worst-case attacks needed to define $\Xi_\alpha$, which requires to build on reference policies lagging behind current policy $\pi$; 2) Q-values must be approximated (typically through neural networks) from rollouts with a sufficiently large variety of attacks   to allow accurate magnitude selection; 3) Approximated Q-values for reference policies must be updated for the evolution of the state-action occupancy distribution of the behavior policy. We discuss each of these aspects in the following.

Moving from the tabular setting with a known MDP to the more practical deep RL regime introduces several additional challenges for the use of $\alpha$-reward-preserving attacks:
1) The construction of $\Xi_\alpha$ relies on optimal policies adapted to worst-case attacks, which are no longer accessible in the deep RL setting and must therefore be approximated using reference policies that lag behind the current policy $\pi$;
2) Q-values must be approximated (typically via neural networks) from rollouts collected under a sufficiently diverse set of attacks to enable reliable magnitude selection;
3) These Q-value estimates must be continuously updated to track the evolving state–action occupancy induced by the learning policy.
We discuss each of these challenges in the remainder of this section.

    \subsection{Reference Policies}
        \label{sec:refpol}
        
        %In the general setting we do not have directly access to the optimal policies $\pi^{*, \Omega}$ and $\pi^{*, \Omega^{\xi^*}}$ as we did in the tabular setting with available MDP. In order to build approximate reward-preserving attacks in such a more restrictive setting, we propose to ground in a reference policy $\tilde{\pi}$ in place of these optimal ones. Rather than relying on a set of attacks for which there exists a policy that preserves rewards, we propose to restrict training to a set $\tilde{\Xi}_\alpha(\Omega)$ that includes any attack for which $\tilde{\pi}$ collects at least an $\alpha$ proportion of the difference between cumulative discounted reward it could collect with minimal and maximal attacks.
        Since optimal policies under attack are not accessible in deep RL, we approximate $\alpha$-reward-preserving attacks using a reference policy $\tilde{\pi}$.
We define $\tilde{\Xi}_\alpha(\Omega)$ as the set of attacks under which $\tilde{\pi}$ preserves at least an $\alpha$ fraction of the reward gap between its minimal and maximal achievable returns. 
That is, for each (s,a), we consider attacks within: 
        \begin{align}
            \hat{Q}_\alpha(s,a)
            &\coloneqq
            (1-\alpha) Q^{\tilde{\pi}}_0\big((s,a),\eta_{\mathcal B}\big) %\nonumber 
             + \alpha 
            Q^{\tilde{\pi}}_1\big((s,a),0\big)
            , \nonumber  \\[2pt]
            \tilde{\mathcal B}_\alpha(s,a)
            &\coloneqq
            \Big\{\xi %\in\mathcal B(s,a)
            \;:\Big. 
            %&\qquad
            Q^{\tilde{\pi}}_\alpha\big((s,a),\xi_\alpha^\eta(s,a)\big)
            \ge \hat{Q}_\alpha(s,a)
            \Big\}. \label{xitilde}
        \end{align}
        which corresponds to the convex core of the corresponding set $\tilde{\Xi}_\alpha(s,a)$, which includes all so-called $(\tilde{\pi},\alpha)$-reward-preserving attacks.  In the following, we use % for the ease of notation we introduce
        $\tilde{\cal B}_\alpha^\eta(s,a)$ as the set of allowed magnitudes in ${\cal B}_\alpha(s,a)$. 
    
        We have: $\Xi_\alpha \subseteq \tilde{\Xi}_\alpha$ and $\Xi_\alpha^{\pi,*} \subseteq \tilde{\Xi}_\alpha^{\pi,*}$. That is,  approximate $\alpha$-reward-preserving attacks that are based on a reference policy $\tilde{\pi}$ are less  conservative than attacks from $\tilde{\Xi}_\alpha^{\pi,*}$. However, assuming that $\pi'$ is a policy lagging behind the currently trained policy $\pi$, inducing a two-timescale learning process, we claim we can define a process that approximately concentrates on $\Xi_\alpha^{*,*}$, to obtain an $\alpha$-robust policy. In the following, policy $\pi$ and $\tilde{\pi}$ are defined as neural networks with the same architecture, with parameters from $\tilde{\pi}$ that are periodically updated via polyak updates using weights from $\pi$.

    \subsection{Approximation of Q-values for the full range of magnitudes in $[0,\eta_{\cal B}]$} 
        \label{sec:appox}

In this setting, Q-values for magnitude selection are approximated  with a neural network taking as input the state-action pair, the candidate magnitude $\eta$, and a target level of reward preservation -- typically one of $\alpha$, $0$, or $1$ as in \eqref{xitilde}. 
To obtain accurate Q-value estimates for $\tilde{\pi}$, it is necessary to sample attack magnitudes across the full range $[0, \eta_{\cal B}]$, rather than greedily from $\tilde{\Xi}_\alpha$, to ensure sufficient diversity in the training transitions. 
This allows the Q-network to predict reliably not only the value for $Q^{\tilde{\pi}}_\alpha((s,a),\xi_\alpha^\eta(s,a))$, but also at the extrema, $Q^{\tilde{\pi}}_1((s,a),0)$ and $Q^{\tilde{\pi}}_0((s,a),\eta_{\cal B})$, which are used as effective bounds for magnitude selection.  
To achieve this, we define an $(\epsilon,\alpha)$-reward-preserving sampling distribution $p^{\tilde{\pi}}_\alpha(\cdot \mid s_t,a_t)$ over $[0, \eta_{\cal B}]$, allocating $(1-\epsilon)$ of its mass inside $\tilde{\mathcal{B}}_\alpha(s_t,a_t)$.  
At each step, with direction $A_t := \xi^A(s_t,a_t)$ and magnitude $\eta_t \sim p^{\tilde{\pi}}_\alpha(\cdot \mid s_t,a_t)$, we ensure that 
$
P\big(\xi_t=(\eta_t,A_t) \in \tilde{\mathcal{B}}_\alpha(s_t,a_t) \,\big|\, \xi_t \in \mathcal{B}(s_t,a_t)\big) = 1-\epsilon .
$

In practice, we first identify the magnitude $\eta^*(s_t,a_t)$ corresponding to the worst Q-value $Q^{\tilde{\pi}}_\alpha(s_t,a_t)$ within $\tilde{\mathcal{B}}^\eta_\alpha(s_t,a_t)$ by evaluating a discrete set of candidates in $[0, \eta_{\cal B}]$ and selecting the largest one satisfying \eqref{xitilde}\footnote{We use a geometric sequence of 40 candidates starting from $\eta_{\cal B}$ with common ratio 0.75, giving more precision for small magnitudes.}.  
We then sample $\eta_t$ from an exponential distribution with rate $\lambda_t = -\log(\epsilon)/\eta^*(s_t,a_t)$, ensuring both coverage of the full admissible range and adherence to the $(\epsilon,\alpha)$-reward-preserving requirement.
 Finally, sampled magnitudes are clipped to $[0, \eta_{{\cal B}}+]$, with $\eta_{{\cal B}}+$ the radius of an extended uncertainty set.  In our experiments, using $\eta_{{\cal B}}+ > \eta_{\cal B}$ proved beneficial: it allows occasional very challenging attacks during training while still leveraging the reward-reshaping effects induced by the use of a narrow set $\cal B$ (see Property \ref{prop:change}).
%Sampled magnitudes are finally clipped to $[0, \eta_{{\cal B}}+]$, where $\eta_{{\cal B}}+ \geq \eta_{\cal B}$ extends the uncertainty set: this allows occasional very challenging attacks while still benefiting from the reward-structure preservation of the narrower set $\tilde{\cal B}_\alpha$ (see Property \ref{prop:change}).  
This approach is of course more conservative than always using $\eta^*(s_t,a_t)$, though this effect can be mitigated by lowering $\alpha$. Alternative distributions (e.g., truncated Gaussian, mixtures, epsilon-greedy with noise) could also be considered.

%which  ensuring both adherence to the $(\epsilon,\alpha)$-reward-preserving requirement and coverage of the full admissible range. Finally, sampled magnitudes are clipped to $[0, \eta_{{\cal B}}+]$, with $\eta_{{\cal B}}+$ the radius of an extended uncertainty set. In our experiments, we found that using $\eta_{{\cal B}}+ \geq \eta_{\cal B}$ is beneficial to get some very challenging attacks during training, while leveraging from beneficial reward restructuration effects arising from using a narrow ${\cal B}$ in the definition of the set of $\alpha$-reward-preserving attacks (see property \ref{prop:change}). 
%This approach is more conservative than always using $\eta^*(s_t,a_t)$, though this can be mitigated by lowering $\alpha$.  Alternative distributions (e.g., truncated Gaussian, mixtures, epsilon-greedy with noise) could also be considered.

    \subsection{Off-policy Updates for non-Stationary State-Action Occupancy}
        \label{sec:offpol}

As $\pi$ evolves during training on the perturbed MDP, its state-occupancy distribution gradually diverges from that of the reference policy $\tilde{\pi}$, on which the Q-network was trained.  
Consequently, $Q^{\tilde{\pi}}_\alpha$, $Q^{\tilde{\pi}}_0$, and $Q^{\tilde{\pi}}_1$ can become inaccurate on new state-action pairs, causing the selected attacks $\xi$ to either collapse to null actions or saturate at $\eta_{\cal B}$, which may result in catastrophic forgetting of acquired robustness or even nominal performance.  
To prevent this, the Q-network defining the attacks must be continuously fine-tuned throughout training. %, even if the reference policy $\tilde{\pi}$ is kept fixed.

       % As $\pi$ evolves through training iterations (via any standard RL algorithm interacting with a perturbed MDP), the distribution of states reached by the behavior policy progressively shift from the state-occupancy distribution of the reference policy $\tilde{\pi}$, on which Q-values have been trained in previous iterations. The selection of $\alpha$-reward-preserving attacks can thus be strongly biased on new areas exploited by the agent, as $Q^{\tilde{\pi}}_\alpha$, $Q^{\tilde{\pi}}_0$, $Q^{\tilde{\pi}}_1$ can be inaccurate on corresponding state-action pairs of trajectories sampled from $\pi^{\Omega^\xi}$. The consequence can be dramatic as the corresponding attack $\xi$ can either collapse to null attacks or shift toward only using $\eta_{\cal B}$ at each step, causing catastrophic forgetting (either of any acquired robustness capabilities in the former case, or even the performances of the agent in the nominal MDP in the latter one). To prevent this, the Q-network we use to define attacks need to be finetuned continuously during the progress of the agent (even in the case where we would freeze the reference agent for the whole training process). 

To adapt the Q-networks during training, we use off-policy updates from transitions $(s_t,a_t,\eta_t,A_t,r_t,s_{t+1})$ collected under $\pi$.  
For each transition, we minimize the squared $\alpha$-reward-preserving temporal difference
\[
\delta_t := Q^{\tilde{\pi}}_\alpha(\cdot) - r_t - \gamma \, \mathbb{E}[Q^{\tilde{\pi}}_\alpha(\cdot)]
\]
weighted by the importance ratio $w_t = \tilde{\pi}/\pi$ to account for the fact that the Q-networks encode values for $\tilde{\pi}$ rather than $\pi$. We distinguish two settings:  
\begin{itemize}
    \item \textbf{Observation attacks:} the attack acts on the policy input and is agnostic to $a_t$, so $Q^{\tilde{\pi}}_\alpha((s_t,a_t),\eta_t)$ reduces to $Q^{\tilde{\pi}}_\alpha(s_t,\eta_t)$, with 
    \begin{equation*}
\delta_t:=Q^{\tilde{\pi}}_\alpha(s_t,\eta_t) -  r_t - \gamma \mathbb{E}_{\eta_{t+1}}\left[ Q^{\tilde{\pi}}_\alpha(s_{t+1},\eta_{t+1})\right] \, ,
        \end{equation*}
        \begin{equation*}
            w_t=\frac{\tilde{\pi}(a_t|\phi(s_t)+\eta_t A_t)}{\pi(a_t|\phi(s_t)+\eta_t A_t)} \, .
        \end{equation*}
        where the expectation on $\eta_{t+1}$ is taken according to $p^{\tilde{\pi}}_\alpha(.|s_{t+1})$.
    \item \textbf{Dynamics attacks:} the attack can exploit the agent's action, so
    \begin{align*}
            \delta_t \;:=\;&
            Q^{\tilde{\pi}}_\alpha\big((s_t,a_t),\eta_t\big) - r_t \\
            &-\; \gamma \,\mathbb{E}_{a_{t+1}} \mathbb{E}_{\eta_{t+1}}
            \Big[\, Q^{\tilde{\pi}}_\alpha\big((s_{t+1},a_{t+1}),\eta_{t+1}\big) \Big] \, ,
        \end{align*}
        \begin{equation*}
            w_t=\frac{\tilde{\pi}(a_t|s_t)}{\pi(a_t|s_t)} \, .
        \end{equation*}
        where $a_{t+1} \sim \tilde{\pi}(.|s_{t+1})$ and the expectation on  $\eta_{t+1}$ is taken according to $p^{\tilde{\pi}}_\alpha(.|s_{t+1},a_{t+1})$.
\end{itemize}

        \noindent We remark that, rather than using the magnitude $\eta_{t+1}$ from the stored transition, we define $\delta_t$ with a target that considers a recomputed $\eta_{t+1}$. This is important in order to account for the evolution of the Q-networks and to maintain consistency of the Bellman updates.
        %(see the corresponding ablation in Section~XX, which shows over-confidence of $p^{\tilde{\pi}}_\alpha$ when using future magnitudes from rollouts).
        Next, while we could use a single sample of $\eta_{t+1}$ at each step, we instead approximate the full expectation to reduce variance. In our setting, this can be done using the same sequence of magnitude candidates already used to define $\eta^*_t$ (see Section~\ref{sec:appox}): We approximate $\mathbb{E}_{\eta\sim p^{\tilde\pi}_\alpha} [Q^{\hat\pi}_\alpha((s_t,a_t),\eta)]$ by applying a trapezoidal rule on our geometric grid $\{\eta^i\}$ for the clipped exponential density on $[0,\eta_{{\mathcal B}+}]$, and adding the point mass at $\eta_{{\mathcal B}+}$ of size $e^{-\lambda_t\eta_{{\mathcal B}+}}$. 
        
        %Finally, we perform a minimization step using the gradient of $w_t(\delta^2)$ to update the Q-network for $Q^{\tilde{\pi}}_\alpha$. The auxiliary values $Q^{\tilde{\pi}}_0$ and $Q^{\tilde{\pi}}_1$ could in principle be optimized in the same way, but restricted to transitions where $\eta_t = \eta_{\mathcal B}$ and $\eta_t = 0$, respectively, and by setting $\eta_{t+1} = \eta_t$ instead of using our approximation of the expectation over $p^{\tilde{\pi}}_\alpha$ to compute the targets. However, filtering transitions by magnitude is highly inefficient in practice. Instead, we employ two separate Q-networks: (1) a \emph{dynamic} Q-network $Q_{\psi_\alpha}(\cdot)$, which models the expected cumulative return under dynamically selected $\alpha$-reward-preserving magnitudes (i.e., $Q_{\psi_\alpha}(\cdot,\eta) \approx Q^{\tilde{\pi}}_\alpha(\cdot,\eta)$); and (2) a \emph{static} Q-network $Q_{\psi_c}(\cdot)$, which always conditions on a fixed magnitude provided as input (i.e., $Q_{\psi_c}(\cdot,\eta_{\mathcal B}) \approx Q^{\tilde{\pi}}_0(\cdot,\eta_{\mathcal B})$ and $Q_{\psi_c}(\cdot,0) \approx Q^{\tilde{\pi}}_1(\cdot,0)$). The use of this static Q-network enables sharing information across all constant-magnitude values in the range $[0, \eta_{\mathcal B}]$, thereby leveraging every transition collected during training, even though only the extreme values $0$ and $\eta_{\mathcal B}$ are ultimately required for computing the $\alpha$-reward-preserving magnitude sets.

        Finally, we update the dynamic Q-network $Q_{\psi_\alpha}$ for $Q^{\tilde{\pi}}_\alpha$ by performing gradient descent on $w_t \, \delta_t^2$.  
While $Q^{\tilde{\pi}}_0$ and $Q^{\tilde{\pi}}_1$ could in principle be optimized similarly by restricting transitions to $\eta_t = \eta_{\mathcal B}$ and $\eta_t = 0$, respectively, and setting $\eta_{t+1} = \eta_t$, this filtering is highly inefficient in practice.  
Instead, we employ two separate Q-networks: 
(1) a \emph{dynamic} network $Q_{\psi_\alpha}(\cdot,\eta)$ modeling the expected return under variable $\alpha$-reward-preserving magnitudes, and 
(2) a \emph{static} network $Q_{\psi_c}(\cdot,\eta)$ conditioned on a fixed magnitude, such that $Q_{\psi_c}(\cdot,\eta_{\mathcal B}) \approx Q^{\tilde{\pi}}_0(\cdot,\eta_{\mathcal B})$ and $Q_{\psi_c}(\cdot,0) \approx Q^{\tilde{\pi}}_1(\cdot,0)$.  
The static network allows sharing information across all constant-magnitude transitions in $[0, \eta_{\mathcal B}]$, leveraging every collected transition even though only the extremes are needed to define the $\alpha$-reward-preserving magnitude sets.

        Algorithms~\ref{algo:approx_dynamics} and~\ref{algo:approx_observations} (in Appendix~\ref{DeepPreserve}) present the complete procedures for dynamics and observation attacks  respectively, following all the main steps described in this section.

\section{Experiments}
\label{sec:xps}

    We study robustness to adversarial perturbations of bounded magnitude $\eta$,  starting from a pre-trained (baseline) agent. We begin by validating the behavior of $\alpha$-preserving attacks for a fixed policy. Next, % our adaptive attack mechanism driven by a reward-preservation target $\alpha$ using a learned critic $Q((s,a),\eta)$, then 
    we experience adversarial fine-tuning of  agents using our dynamic $\alpha$-preserving attacks. % against these attacks. 
    Finally, we compare robustness of our obtained policies against agents trained using non-adaptive baselines, % trained with 
    using either a constant or a uniformly sampled magnitude.
    Experiments reported focus only on observation attacks, experiments for dynamics attacks are left for future work.  
    
    \subsection{Experimental setup}
        We evaluate on \emph{HalfCheetah-v5} using a pre-trained \emph{SAC} \citep{haarnoja2018soft} agent as our baseline. The adversary applies observation perturbations constrained by $L_2$ radius $\eta \le \eta_{\cal B}$ using the FGM\_QAC attack (results with other attacks are given in appendix \ref{sec:supplementary_material_alpha}). This attack method is a variant of the FGSM attack \citep{goodfellow2014explaining} for actor attacks of Q actor-critic agents. It targets both networks jointly by back-propagating gradients from the critic network $\tilde{q}$ through the actor network $\tilde{\pi}$ to the input observation $x$ of the actor:
        %\begin{align}
        %    x' =& \; x - \eta \frac{\nabla_x \tilde{q}\big(x^\perp,\mu_{\tilde{\pi}}(x)\big)}{||\nabla_x \tilde{q}\big(x^\perp,\mu_{\tilde{\pi}}(x)\big)||_2}
        %\end{align}
        \[
x' = x - \eta\,
\frac{
\nabla_x \tilde q\big(x^\perp,\mu_{\tilde\pi}(x)\big)
}{
\left\|\nabla_x \tilde q\big(x^\perp,\mu_{\tilde\pi}(x)\big)\right\|_2
}, 
\qquad
x^\perp := \mathrm{stopgrad}(x)
\]
where % $x^\perp=x$, $\frac{\partial x^\perp}{\partial x}=0$, and 
$\mu_{\tilde{\pi}}$ stands as the mean parameter of the gaussian distribution produced by the reference SAC actor policy ${\tilde{\pi}}$. That is, it seeks at the direction in the observation space that results in an action that most decreases the expected return. Note that  $\tilde{q}$ is trained from a buffer of non-perturbed observations, while the actor $\tilde{\pi}$ uses perturbed ones. 
        
        We report episodic return under either no perturbation ($\eta=0$) or under test-time attacks at magnitude $\eta$, averaging over \emph{20} evaluation episodes per checkpoint (plots show rolling average evaluation over the 5 last checkpoints) and per training seed (3 seeds per setting used in our experiments). Experiments were performed on Nvidia V100 GPU devices. The code will be publicly released after acceptance. All hyper-parameters used are given in appendix \ref{sec:hparams}.

    \subsection{Reward-preserving attacks via a learned $Q^{\tilde{\pi}}_\alpha((s,a),\eta)$ and reward preservation target $\alpha$}
    
        We first analyze the behavior of reward-preserving attacks for a fixed baseline agent (i.e., $\tilde{\pi}=\pi$ is constant over the whole training process). In these experiments, we %thus only
        train critics of the form $Q^{\tilde{\pi}}_\alpha((s,a),\eta)$ for various preservation levels $\alpha$, in order to assess the impact of $\alpha$ on %select %that estimate expected return when the adversary perturbs with attack magnitude $\eta$. These networks are used to sample
        selected attack magnitudes and policy performances. % given policy performances, as explained in section \ref{sec:appox}. 
        %For the following experiments we decouple the ambiguity radius $\eta_{\cal B}$ into two values $\eta_{\cal B}$ and $\eta_{\cal B+}$. $\eta_{\cal B}$ is the maximum allowed perturbation, and $\eta_{\cal B+}$ is used only to compute $Q^{\tilde{\pi}}_0((s,a),\eta_{\cal B+})$ for a lower worst q-value estimate. We tuned $\eta_{\cal B}$ and $\eta_{\cal B+}$ and selected $\eta_{\cal B}=0.3$ and $\eta_{\cal B+}=0.5$ as the best-performing values. Full experiment results are reported in Appendix~\ref{sec:supplementary_material_alpha}.
        
        Figure~\ref{fig:alpha_calibration} reports the performance of the baseline agent under various values of $\alpha$ for reward-preserving attacks (using $\eta_{\cal B} = \eta_{\cal B+}=0.3$). The results validate that the mechanism behaves as intended: for smaller $\alpha$, the attack is more aggressive, yielding lower achieved return and larger average chosen magnitude; as \linebreak[4] $\alpha \to 1$, attacks become mild, returns increase, and the average chosen magnitude decreases.
        
        \begin{figure}[t]
            \centering
            \includegraphics[width=0.9\linewidth]{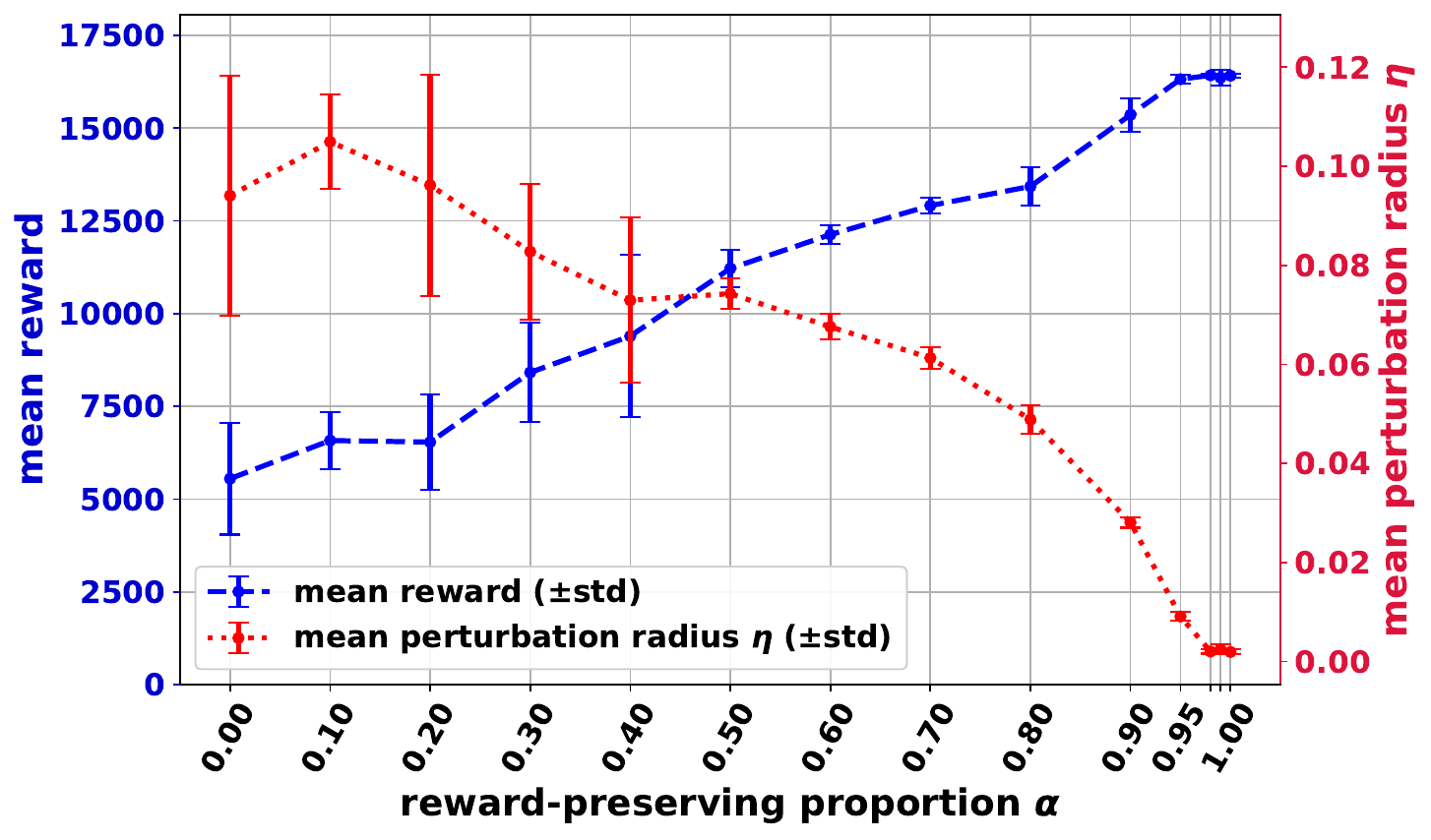}
            \caption{Calibration of reward-preserving $\alpha$-attacks on a pre-trained agent using $Q^{\tilde{\pi}}_\alpha((s,a),\eta)$.
Smaller values of $\alpha$ induce larger average perturbation magnitudes and lower returns, while $\alpha \to 1$ recovers nominal performance.}
            \label{fig:alpha_calibration}
        \end{figure}

    %\subsection{Adversarial fine-tuning with our reward-preserving $\alpha$-attacks}
    
     %   We then fine-tune several agents against the reward-preserving adversary for different values of $\alpha$. So here the adversary automatically adjusts the perturbation magnitude to match the desired $\alpha$ reward-preservation level, using $\eta_{\cal B}=0.3$ and $\eta_{\cal B+}=0.5$. The full procedure of  training used in these experiments is given in algorithm \ref{algo:approx_observations}.
     %   Fine-tuning curves are reported in Appendix~\ref{sec:supplementary_material_alpha}.

    \subsection{Adversarial $\alpha$-reward-preserving training} %fine-tuning with  $\alpha$-attacks}

In this section, we consider the full $\alpha$-reward-preserving adversarial training process introduced in Section~\ref{sec:appox}.
Starting from a pre-trained agent $\pi$, we fine-tune it against adversaries enforcing different levels of reward preservation $\alpha$, which are trained jointly with $\pi$. 
%In this setting, the adversary automatically adjusts the perturbation magnitude to meet the target reward-preservation level $\alpha$, using $\eta_{\cal B}=0.3$ and $\eta_{\cal B+}=0.5$.
The complete training procedure is summarized in Algorithm~\ref{algo:approx_observations}.

All fine-tuning learning curves are reported in Appendix~\ref{sec:supplementary_material_alpha}. They show that best performances are consistently obtained using  $\eta_{\cal B}=0.3$ and $\eta_{\cal B+}=0.5$. These values are thus used for all reward-preserving results presented in the following.

Figure~\ref{fig:alpha_profiles} shows the performance of the $\alpha$-trained agents after 30M environment steps, evaluated across the full range of test-time perturbation magnitudes $\eta$. Agents trained with intermediate $\alpha$ (within $[0.3;0.7]$) maintain strong performance throughout the range, with the best results observed for $\alpha=0.6$. % ( all $\alpha$ values within . 
Intermediate $\alpha$-reward-preserving attacks provide a level of challenge adapted to the agent’s capabilities across different regions of the environment, avoiding overly aggressive perturbations in difficult areas while still encouraging exploration in easier or more regular regions.

%Intermediate $alpha$-reward-preserving attacks allow to challenge the agent de manière adaptée à son niveau dans chacune des zones de l'environnement, en évitant d'être trop agressif dans les zones difficiles et trop conservatif dans les zones faciles / régulières. 

        \begin{figure}[t]
            \centering
            \includegraphics[width=0.9\linewidth]{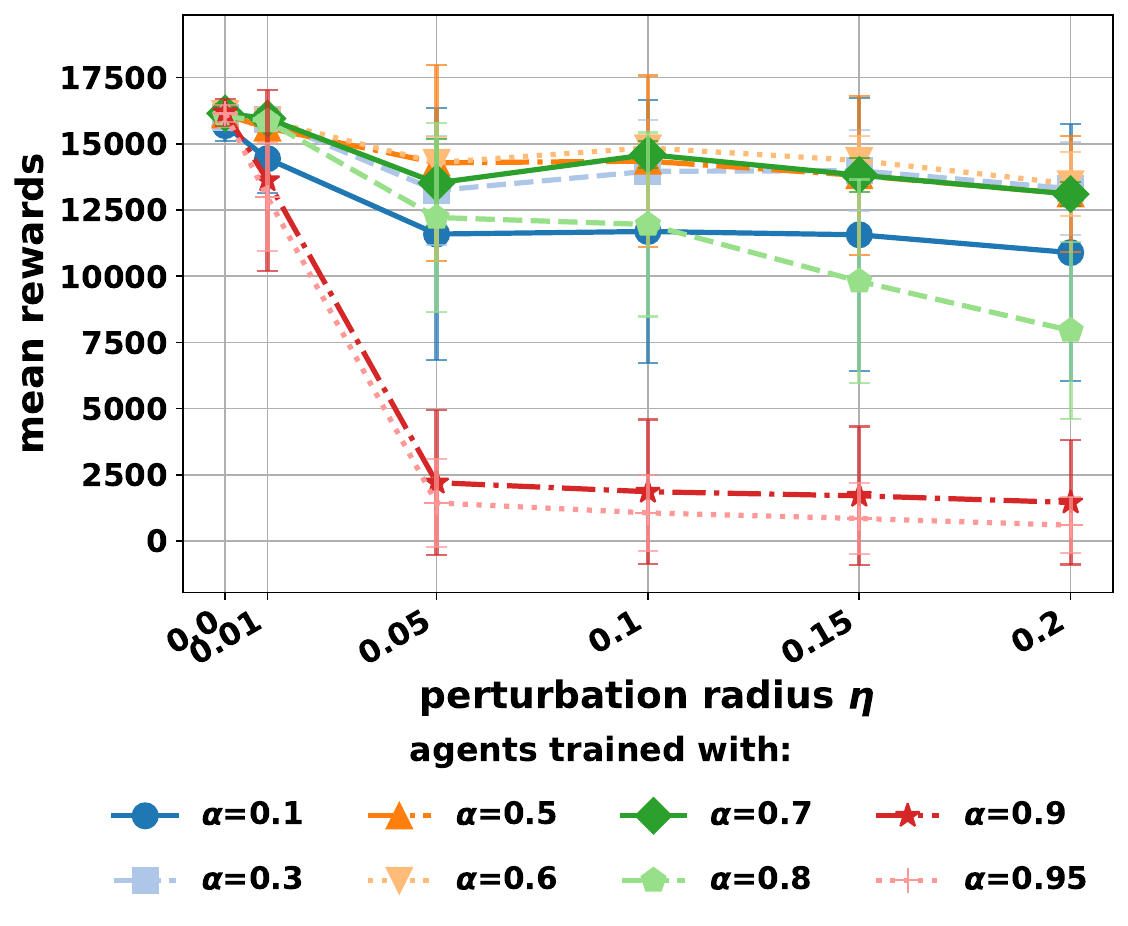}
            \caption{Robustness profiles of $\alpha$-trained agents under varying evaluation $\eta$: Agents trained with intermediate $\alpha$ are the more robust over a broader range of magnitudes $\eta$.}
            \label{fig:alpha_profiles}
        \end{figure}

    \subsection{Comparison with non-adaptive approaches}

        We then compare our method to standard adversarial training baselines with constant or random perturbation magnitudes. Starting from the pre-trained agent, we fine-tune agents against FGM\_QAC using either (i) a \emph{constant} perturbation magnitude  $\eta$, or (ii) a training magnitude  sampled uniformly as $\eta \sim \mathcal{U}(0,\eta_{\cal B})$ to encourage robustness across perturbation magnitudes at test-time. Fine-tuning curves for these baselines are reported in Appendix~\ref{sec:supplementary_material_const}.
        
        To summarize cross-$\eta$ robustness, Figure~\ref{fig:fixed_eta_profiles} evaluates the agents against a range of perturbation magnitudes $\eta$. We find that constant-$\eta$ adversarial training produces \emph{specialized} policies: each agent is robust for evaluation settings using the same perturbation magnitude as the one  it was trained on, and robustness does not transfer well to significantly smaller or larger attack strengths. In contrast, training with uniformly sampled attack magnitudes yields policies that are robust over a broader range of settings for $\eta$ in test environments.
        
        However, despite this improved coverage, these uniformly-trained agents remain consistently below our best reward-preserving $\alpha$-trained policies. The figure reports the performances of an agent trained with our $\alpha$-reward-preserving process, using $\alpha=0.6$, which achieves significantly higher returns under perturbation (for any constant $\eta > 0$ at test-time) and also maintains stronger nominal performance (when $\eta=0$). %We also note that, compared to  
        
        \begin{figure}[t]
            \centering
            \includegraphics[width=0.9\linewidth]{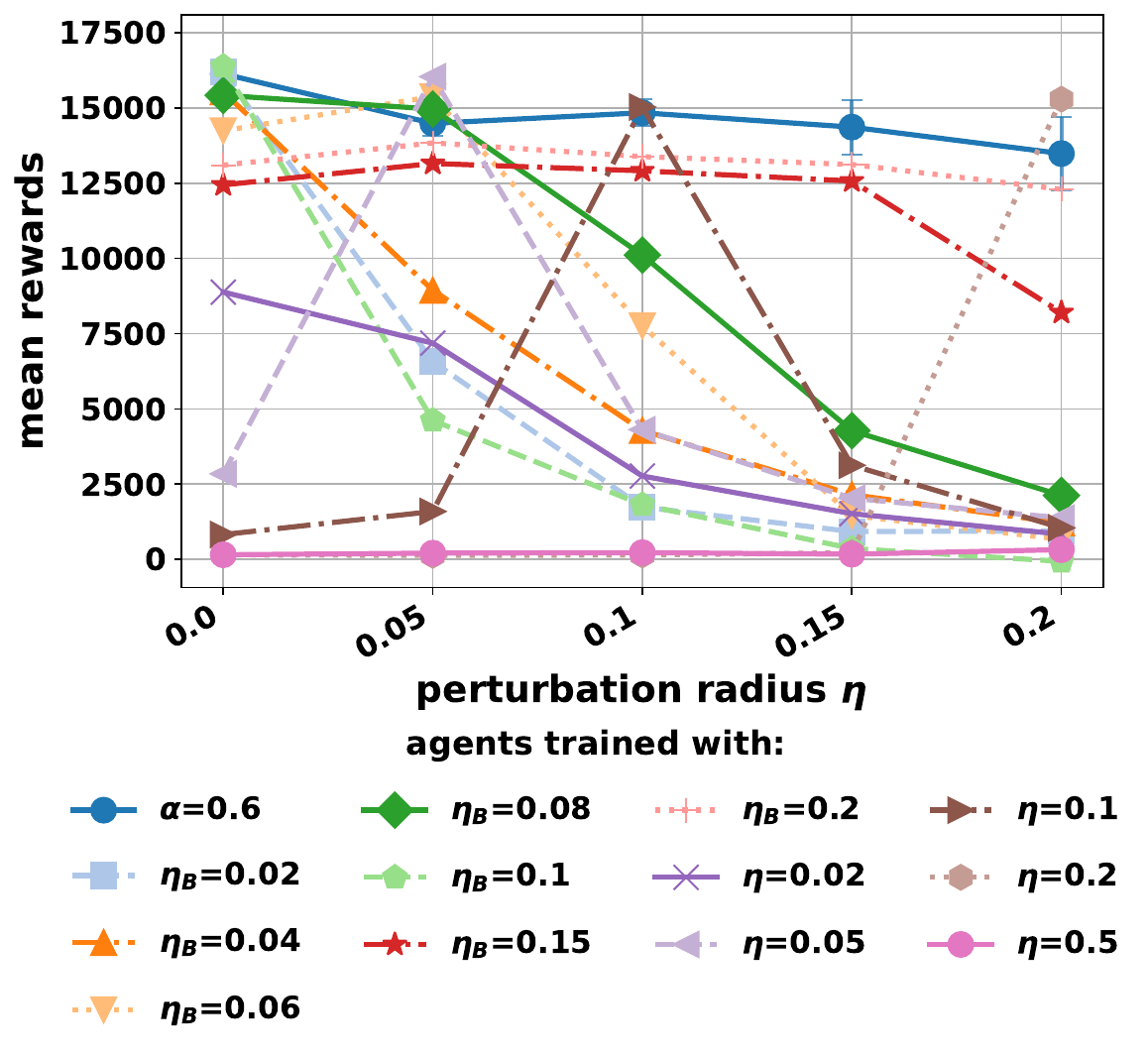}
            \caption{Robustness profiles under varying evaluation $\eta$. Agents trained with a constant magnitude  $\eta$, are denoted $\eta$ in the legend. %, specialize around their training magnitude. 
            %While a
            Agents trained with random magnitude %radius
            $\eta \sim \mathcal{U}(0,\eta_{\cal B})$ are denoted $\eta_{\cal B}$ in the legend. %figure, and are robust over a broader range of perturbations. Nevertheless, the  agent trained with o
            Our reward-preserving approach with $\alpha=0.6$ achieves higher return across perturbed cases and preserves better nominal performance.}
            \label{fig:fixed_eta_profiles}
        \end{figure}

        Experiments conducted %the same experiments
        with other adversarial attack methods, whose %; the corresponding
        results are reported in Appendix~\ref{sec:supplementary_material} report the same tendencies.
        
        \noindent Overall, reward-preserving $\alpha$-training yields the best robustness/nominal trade-off: it avoids over-specialization for a specific $\eta$, while preserving strong performance in the nominal setting.

\section{Conclusion}
    
    We introduced \emph{reward-preserving attacks} as a principled way to control adversarial strength in reinforcement learning without collapsing the learning signal. Our formulation constrains the adversary so that, at each state-action pair, an $\alpha$-fraction of the nominal-to-worst-case return gap remains achievable for a policy adapted to the attack. To apply this idea in deep RL, we proposed a magnitude-direction decomposition and an approximate procedure that learns a critic $Q((s,a),\eta)$ to adapt attack strength online to meet a reward-preservation target.

Empirically, $\alpha$-reward-preserving attacks behave as intended: decreasing $\alpha$ increases the average attack magnitude and reduces returns in a controlled way. Adversarial fine-tuning against these attacks preserves strong nominal performance while producing policies robust across a wide range of perturbations, whereas constant-magnitude or uniformly sampled attacks either overfit or sacrifice performance. Intermediate levels of preservation provide the best trade-off.  

These results demonstrate the potential of $\alpha$-reward-preserving attacks as a principled mechanism for adaptive robustness, exposing agents to challenges matched to their capabilities while avoiding catastrophic failures. Beyond the current benchmarks, these attacks offer a promising tool for improving generalization, safe exploration, and transfer to novel environments, as they expose the agent to controlled yet meaningful stress-tests. Future work  will extend this approach to high-dimensional continuous control, multi-agent scenarios, and other types of perturbations, including dynamics perturbations, further leveraging the balance between challenge and task solvability.

%These results highlight the potential of $\alpha$-reward-preserving attacks as a principled mechanism for adaptive robustness, allowing agents to experience challenges matched to their capabilities while avoiding catastrophic failure in critical states. Future work will explore their application to more complex, partially observable or high-dimensional environments, including continuous control tasks and multi-agent settings, as well as to other forms of perturbations such as dynamics attacks. %, or partial observability.

%Beyond reinforcement learning, $\alpha$-reward-preserving perturbations could inspire new approaches in safe exploration, curriculum learning, or robust policy optimization, where maintaining solvability while encouraging challenge is crucial.

\section*{Impact Statement}

    This work introduces $\alpha$-reward-preserving attacks as a principled mechanism for adaptive robustness in reinforcement learning. By controlling the magnitude of adversarial perturbations according to the agent's capabilities, these attacks allow policies to experience meaningful challenges without compromising the solvability of the task.  
    
    The immediate societal impact is primarily positive: this method can improve the reliability, safety, and generalization of RL agents in real-world applications, including robotics, autonomous systems, and other safety-critical domains. By training agents under controlled, adaptive perturbations, the approach reduces the likelihood of catastrophic failures when deployed in unpredictable environments.  
    
    Potential risks include misuse in settings where robust agents are deployed for harmful purposes, and the fact that $\alpha$-reward-preserving attacks can make the agent perceive normal conditions while it is actually being perturbed, which could be exploited in adversarial or deceptive scenarios. Overall, our work highlights a tool for safer and more resilient RL, emphasizing responsible usage.

\section*{Acknowledgments}

    This work has been supported by the French government under the “France 2030” program, as part of the SystemX Technological Research Institute within the Confiance.ai program. This work was granted access to the HPC resources of IDRIS under the allocation AD011015866 made by GENCI.

\bibliography{biblio}
\bibliographystyle{icml2026} % ou plainnat, unsrtnat, abbrvnat

%%%%%%%%%%%%%%%%%%%%%%%%%%%%%%%%%%%%%%%%%%%%%%%%%%%%%%%%%%%%%%%%%%%%%%%%%%%%%%%
%%%%%%%%%%%%%%%%%%%%%%%%%%%%%%%%%%%%%%%%%%%%%%%%%%%%%%%%%%%%%%%%%%%%%%%%%%%%%%%
% APPENDIX
%%%%%%%%%%%%%%%%%%%%%%%%%%%%%%%%%%%%%%%%%%%%%%%%%%%%%%%%%%%%%%%%%%%%%%%%%%%%%%%
%%%%%%%%%%%%%%%%%%%%%%%%%%%%%%%%%%%%%%%%%%%%%%%%%%%%%%%%%%%%%%%%%%%%%%%%%%%%%%%
\newpage
\appendix
\onecolumn

\section{Appendix}

    % =========================
    % Related Work (Appendix) — Robust RL and Regulated Adversaries
    % =========================
    \subsection{Related Work: Regulating Adversarial Attacks for Robust RL}
        \label{sec:related_work}
        
        This appendix complements Section~\ref{sec:rewpreserv_intro} by situating our reward-preserving attacks within the literature on adversarial robustness in RL. A recurring theme is that, unlike supervised learning, adversarial perturbations in RL compound along trajectories: overly aggressive attacks can make the task effectively unsolvable and collapse the learning signal, while overly weak attacks fail to induce meaningful robustness \citep{huang2017adversarial,pinto2017robust,morimoto2005robust}. We organize prior work along three axes aligned with Section~\ref{sec:rewpreserv_intro}: (1) \emph{observation vs.\ dynamics} attacks, (2) \emph{local vs.\ global} control of perturbations, including sa-rectangularity and step-level perturbations vs. episode-level perturbations, and (3) \emph{regularizations} that explicitly regulate attack strength.
        
        \subsubsection{Observation vs.\ dynamics attacks}
            \label{sec:rw_obs_vs_dyn}
            
            \paragraph{Observation-space attacks.}
                Observation attacks perturb the agent’s inputs, e.g., by adding norm-bounded noise to state features or pixels. Early work adapted supervised adversarial examples to RL and showed that policies can be highly sensitive to small input perturbations, motivating adversarial training in the observation space \citep{huang2017adversarial,goodfellow2014explaining,carlini2017towards}. In deep actor-critic settings, gradient-based perturbations are natural when the attacker has white-box access, yielding FGSM/PGD-style attacks or variants that backpropagate critic signals through the actor to craft damaging perturbations (as in our FGM\_QAC setting).
                Beyond gradient-based attacks, observation perturbations can also be generated by \emph{adversarial policies} trained with RL in a black-box manner: an attacker network observes $x_t$ (and optionally additional side information) and outputs an additive perturbation $a_t^{\xi,X}$, producing $x'_t=x_t+a_t^{\xi,X}$. Concurrent works such as OARLP \citep{russo2021towards}, AdvRL-GAN \citep{yu2022natural}, and ATLA \citep{zhang2021robust} follow this principle, typically optimizing an untargeted adversarial objective (often the negative of the protagonist return) without requiring access to gradients of the victim policy.

            \paragraph{Dynamics (transition) attacks and robust MDPs.}
                A second family targets the transition kernel (or a parameterization of it) and is closely related to robust MDPs \citep{wiesemann2013robust}. Robust value iteration and related dynamic-programming methods provide guarantees under structured uncertainty sets, but their conservatism and the realism of the induced worst-case dynamics strongly depend on the geometry and rectangularity assumptions of the uncertainty set \citep{wiesemann2013robust}. In deep RL, dynamics perturbations are often implemented through \emph{adversarial policies} that inject disturbances (forces, parameter shifts, etc.) based on the current state, resulting in a minimax game between a protagonist and an antagonist \citep{pinto2017robust}. RARL \citep{pinto2017robust} is a canonical example: the adversary is trained with reward $r_\xi=-r$ to apply disturbances that reduce the protagonist’s return. Extensions refine the adversary objective (e.g., risk-aware variants) or the perturbation mechanism, but share the same core idea: the transition dynamics are modified online by an adversarial agent.
                A complementary line of work attacks dynamics via \emph{adversarial domain randomization}, where an environment parameter vector is sampled (typically per episode) from an uncertainty set. ADR \citep{mehta2020active} learns a challenging sampling distribution over parameters (e.g., via particles and SVPG), and M2TD3 \citep{tanabe2022max} conditions the critic on environment parameters to enable gradient-based search for worst-case configurations within the parameter set. Compared to per-step adversarial policies, these approaches often yield more \say{plausible} dynamics shifts by restricting perturbations to a low-dimensional parameterization.
        
        \subsubsection{Local vs.\ global control of perturbations, sa-rectangularity}
            \label{sec:rw_local_vs_global}
            
            A key distinction for robust RL is whether perturbations can be chosen \emph{locally} at each time step and state-action pair, or only \emph{globally} at the episode (or slower) timescale.
        
            \paragraph{Local (per-step) adversaries and sa-rectangularity.}
                When the adversary can select perturbations at every step conditioned on the current state (and possibly action), the induced uncertainty is effectively sa-rectangular: the adversary can independently choose worst-case disturbances for each encountered state-action pair. This is the implicit setting of many adversarial-policy methods (including RARL) and of robust MDP formulations with sa-rectangular uncertainty sets \citep{pinto2017robust,wiesemann2013robust}. While this yields strong worst-case robustness, it can be excessively pessimistic: compounding, locally worst-case perturbations may drive the agent into unrecoverable regions and make training unstable or uninformative.
            
            \paragraph{Global (episode-level or slowly varying) perturbations.}
                In contrast, domain randomization methods typically choose an environment configuration once per episode, reducing the adversary’s ability to ``chase'' the agent online. Adversarial domain randomization (ADR) \citep{mehta2020active} and worst-parameter search methods such as M2TD3 \citep{tanabe2022max} fall into this category. These global perturbations often preserve solvability more readily, but they may miss critical state-local vulnerabilities that only appear when an attacker can time perturbations precisely. A middle ground constrains how quickly perturbations can change over time, producing disturbances that are neither fully local nor fully global; for instance, TC-RMDP bounds the rate of change of an adversarially controlled dynamics parameter, enforcing temporally correlated perturbations and mitigating the drawbacks of fully rectangular uncertainty \citep{zouitine2024time}.
        
        \subsubsection{Regularizations that prevent overly aggressive attacks}
            \label{sec:rw_regularizations}
            
            The above distinctions motivate a third axis: \emph{how} methods regulate the adversary so it remains challenging yet does not break learnability.
            
            \paragraph{Magnitude penalties and constrained adversaries.}
                A direct approach is to penalize attack magnitude or constrain adversary updates. SC-RARL adds a penalty term to the adversary objective to discourage large perturbations while still minimizing the protagonist return, thereby reducing training collapse under overly aggressive disturbances \citep{ma2018improved}. Related ideas in distributionally robust RL constrain the adversary within a Wasserstein ball (or similar trust-region) around previous dynamics, producing adversarial-yet-plausible shifts rather than arbitrarily destructive ones \citep{abdullah2019wasserstein}. In generative environment design, DRAGEN learns a generator of environments and performs adversarial search under a distributional constraint, again aiming to avoid unrealistic worst-case jumps \citep{ren2022distributionally}.
                
            \paragraph{Adaptive curricula and performance-based regulation.}
                Another family regulates attack strength according to the protagonist’s learning progress. For example, A2P-SAC adaptively modulates the effective attacker influence so that attacks strengthen when the agent is performing well and weaken when training destabilizes, reducing the need for manual tuning and avoiding loss of learning signal \citep{liu2024robust}. More broadly, curriculum and teacher-student environment design methods propose tasks/environments of increasing difficulty, attempting to keep training near the frontier of the agent’s capabilities; in UED/PAIRED, a teacher proposes environments that maximize regret between agents, implicitly controlling difficulty to remain informative \citep{dennis2020emergent}.
                
            \paragraph{Positioning reward-preserving attacks.}
                Most prior regulation mechanisms act either (i) \emph{globally} (episode-level parameter shifts, slowly varying disturbances), or (ii) via \emph{generic} regularizers (penalizing magnitude or constraining adversary updates) that do not explicitly account for \emph{state criticality}. Our setting highlights that solvability can be destroyed in specific regions (e.g., on a narrow bridge), while large perturbations may be tolerable elsewhere. Reward-preserving attacks address this by regulating adversarial strength \emph{locally} through a value-based feasibility constraint: at each state-action pair, the adversary is restricted so that an $\alpha$ fraction of the nominal-to-worst-case return gap remains achievable (Definition~\ref{def_preserv}). This yields a state-conditional attack magnitude that is strong in ``safe'' regions and automatically reduced in critical regions where aggressive perturbations would eliminate any viable recovery strategy. In this sense, our approach can be seen as a learnability-preserving alternative to fixed-radius (often too destructive) and uniformly sampled-radius (often too diffuse) adversarial training, while remaining compatible with both observation attacks (as in our experiments) and dynamics uncertainty sets (as in robust MDP formulations).

    \newpage

        \subsection{Hyper-parameters}
        \label{sec:hparams}
        
        %\paragraph{Baseline SAC and environment.}
        
        \begin{table}[!ht]
            \centering
            \caption{Environment, initialization, and training budgets.}
            \label{tab:hparams_budget}
            \begin{tabular}{@{}ll@{}}
            \toprule
            \textbf{Item} & \textbf{Value} \\
            \midrule
            Environment & HalfCheetah-v5 \\
            RL framework & StableBaselines3 \\
            RL algorithm & SAC \\
            Nominal pre-training steps & 20M \\
            \bottomrule
            \end{tabular}
        \end{table}
        
        \begin{table}[H]
            \centering
            \caption{Architecture of the base SAC agent and optimizer settings}
            \label{tab:hparams_sac}
            \begin{tabular}{@{}ll@{}}
            \toprule
            \textbf{Hyper-parameter} & \textbf{Value} \\
            \midrule
            Policy & \texttt{MlpPolicy} \\
            Actor/Critic MLP & \texttt{net\_arch=[256,256,256]} \\
            Activations & ReLU \\
            Log std init & -3 \\
            Discount & 0.99 \\
            Batch size & 256 \\
            Train envs & 8 \\
            Replay buffer size & 1e6 \\
            Polyak & 0.005 \\
            Entropy coef & \texttt{auto} \\
            Train freq & 1 \\
            grad steps & 1 \\
            SDE exploration & True \\
            Learning rate & 1e-3\\
            \bottomrule
            \end{tabular}
        \end{table}
        
        %\paragraph{Reward-preserving training loop (Q-based magnitude selection).}
        \begin{table}[H]
            \centering
            \caption{Cycle-based training, and $Q_\alpha$-learning settings.}
            \label{tab:hparams_loop}
            \begin{tabular}{@{}ll@{}}
            \toprule
            \textbf{Hyper-parameter} & \textbf{Value} \\
            \midrule
            Q MLP & \texttt{net\_arch=[256,256,256,256]} \\
            %Seeds & \{0, 1\} \\
            Adversarial training steps & 30M \\
            Train envs & 8 \\
            Replay buffer & 1e5 \\
            Q sample reuse & 10 \\
            Q batch size & 1,000 \\
            Protagonist LR & 1e-3 \\
            Q-network LR & 1e-3 \\
            Polyak (target Q) & \texttt{tau\_q}=0.1 \\
            Polyak (reference protagonist) & \texttt{tau\_ref}=0.1 \\
            Reward-preservation tail prob. & $\epsilon=0.01$ \\
            \bottomrule
            \end{tabular}
        \end{table}
        
        %\paragraph{Attacks and evaluation protocol.}
        %\begin{table}[H]
        %    \centering
        %    \caption{Observation attacks and evaluation radii.}
        %    \label{tab:hparams_attacks}
        %    \begin{tabular}{@{}ll@{}}
        %    \toprule
        %    \textbf{Hyper-parameter} & \textbf{Value} \\
        %    \midrule
       %     Attack space & Observations \\
       %     Perturbation norm & $L_2$ \\
       %     Reward-preservation tail prob. & %$\epsilon=0.01$ \\
       %     \bottomrule
       %     \end{tabular}
       % \end{table}

    \newpage
        
    \subsection{Additional results}
        \label{sec:supplementary_material}
        To better understand which design choices drive performance and robustness, we report a set of supplementary studies.
        All runs use the same environment \emph{HalfCheetah-v5} and the same pre-trained \emph{SAC} agent. We compare methods using evaluation curves recorded during training as well as the final-policy evaluation over a grid of perturbation magnitudes.
    
        \paragraph{Attacks considered.}
            
            \ \\We evaluate robustness under three observation-space attacks :

            RUA denotes a random baseline where we sample a uniform random perturbation direction :
            \begin{equation}
                x' = x + \eta \frac{u}{\|u\|_2} \quad \text{with} \quad u \sim \mathcal{U}\big(\!-\!1,1\big)^{\dim(x)}
            \end{equation}
            
            FGM\_C %is another % \emph{targeted} 
            %gradient-based attack, it
            is a variant of the FGSM attack for continuous-action policies: it perturbs the observation in the direction that minimizes an MSE loss between the deterministic actor output $\mu_{\tilde{\pi}}(x)$ and a target action $a$. In our experiments, we use a random target sampled uniformly around the original action of the actor output :
            \begin{equation}
                x' = x - \eta \,\frac{\nabla_x \|\mu_{\tilde{\pi}}(x)-a\|_2^2}{\Big\|\nabla_x \|\mu_{\tilde{\pi}}(x)-a\|_2^2\Big\|_2} \quad \text{with} \quad a \sim {\cal U}\big(\:\mu_{\tilde{\pi}}(x)\!-\!5e^{-5}\,,\,\mu_{\tilde{\pi}}(x)\!+\!5e^{-5}\:\big)
            \end{equation}
            This produces perturbations with high stochasticity, while focusing on noise directions that have impacts on the policy's decision. 
            
            FGM\_QAC is an \emph{untargeted} gradient-based attack. It is a variant of the FGSM attack for Q Actor-Critic architectures: it back-propagates the critic signal through the actor (freezing the critic observation input) to decrease the estimated $q$-value of the actor’s action :
            \begin{equation}
                x' = x - \eta \,\frac{\nabla_x \tilde{q}(x^\perp,\mu_{\tilde{\pi}}(x))}{\|\nabla_x \tilde{q}(x^\perp,\mu_{\tilde{\pi}}(x))\|_2} \quad \text{with} \quad %\frac{\partial x^\perp}{\partial x}=0
                x^\perp := stopgrad(x)
            \end{equation}
            This attack thus produces perturbations in the direction that looks the most detrimental from the critic point of view (hence, with long term impact).

        \subsubsection{Training with alpha reward preserving strategy for different attack methods}
            \label{sec:supplementary_material_alpha}

            In this first ablation, we study the effect of the \emph{$\alpha$-reward-preserving} training strategy under different adversarial attack models.
            For each attack (FGM\_QAC, FGM\_C and RUA), we perform a grid search over $\alpha$, $\eta_B$, and $\eta_{B+}$, and report in Figures~\ref{fig:alpha-training-curves-fgm-qac},~\ref{fig:alpha-training-curves-fgm-c} and ~\ref{fig:alpha-training-curves-rua} evaluation curves collected during training at fixed perturbation magnitudes $\eta \in \{0, 0.05, 0.15\}$, and in Figure~\ref{fig:alpha-evaluations} evaluation of the final agent over a broader range $\eta \in \{0, 0.01, 0.05, 0.1, 0.15, 0.2\}$.

            We observe that using $\alpha=0.6$ with $\eta_{\cal B}=0.3$ and $\eta_{\cal B+}=0.5$ consistently yields strong performance across all attack settings. This configuration provides a favorable trade-off between challenging the policy and preserving task solvability.

            %We observe that using $\alpha=0.6$, $\eta_{\cal B}=0.3$ and $\eta_{\cal B}=0.5$ obtains consistently good results in every setting for every attack. It stands as a good trade-off for dealing with challenging the policy and preserving solvability of the task.  
    
            % ==================== FGM-QAC ====================
            \begin{figure}[!htp]
                \centering
                \begin{adjustwidth}{-0cm}{-1.3cm} % extend left/right by 1.5cm (tune this)
                \centering
                \caption{Evaluation curves during training with alpha reward preserving strategy under \textbf{FGM-QAC} attacks for evaluation settings $\eta \in \{0, 0.05, 0.15\}$.}
                \label{fig:alpha-training-curves-fgm-qac}
                \begin{subfigure}[t]{0.355\textwidth}
                    \centering
                    \includegraphics[width=\linewidth]{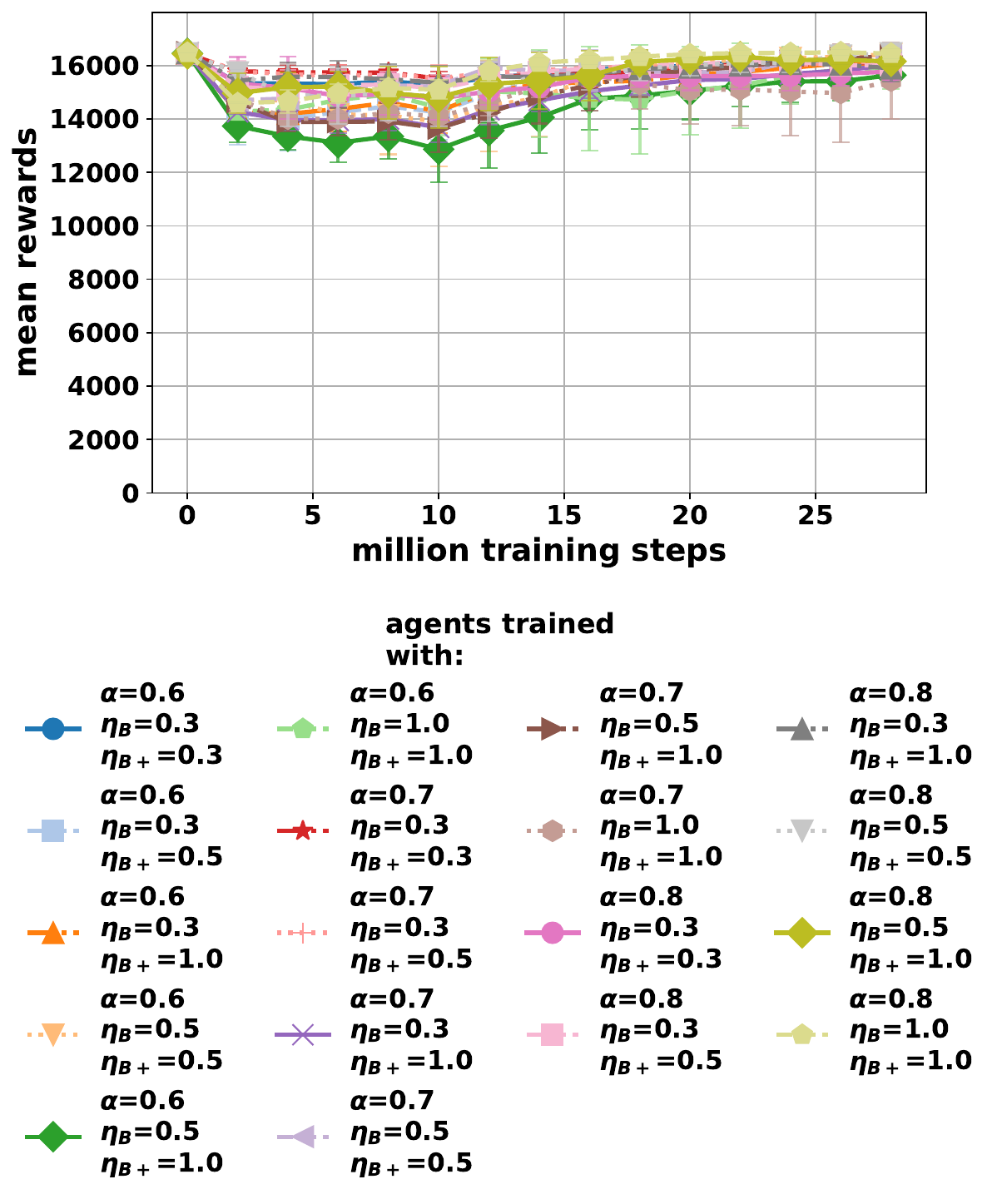}
                    \caption{training; eval $\eta=0$}
                \end{subfigure}\hfill
                \begin{subfigure}[t]{0.355\textwidth}
                    \centering
                    \includegraphics[width=\linewidth]{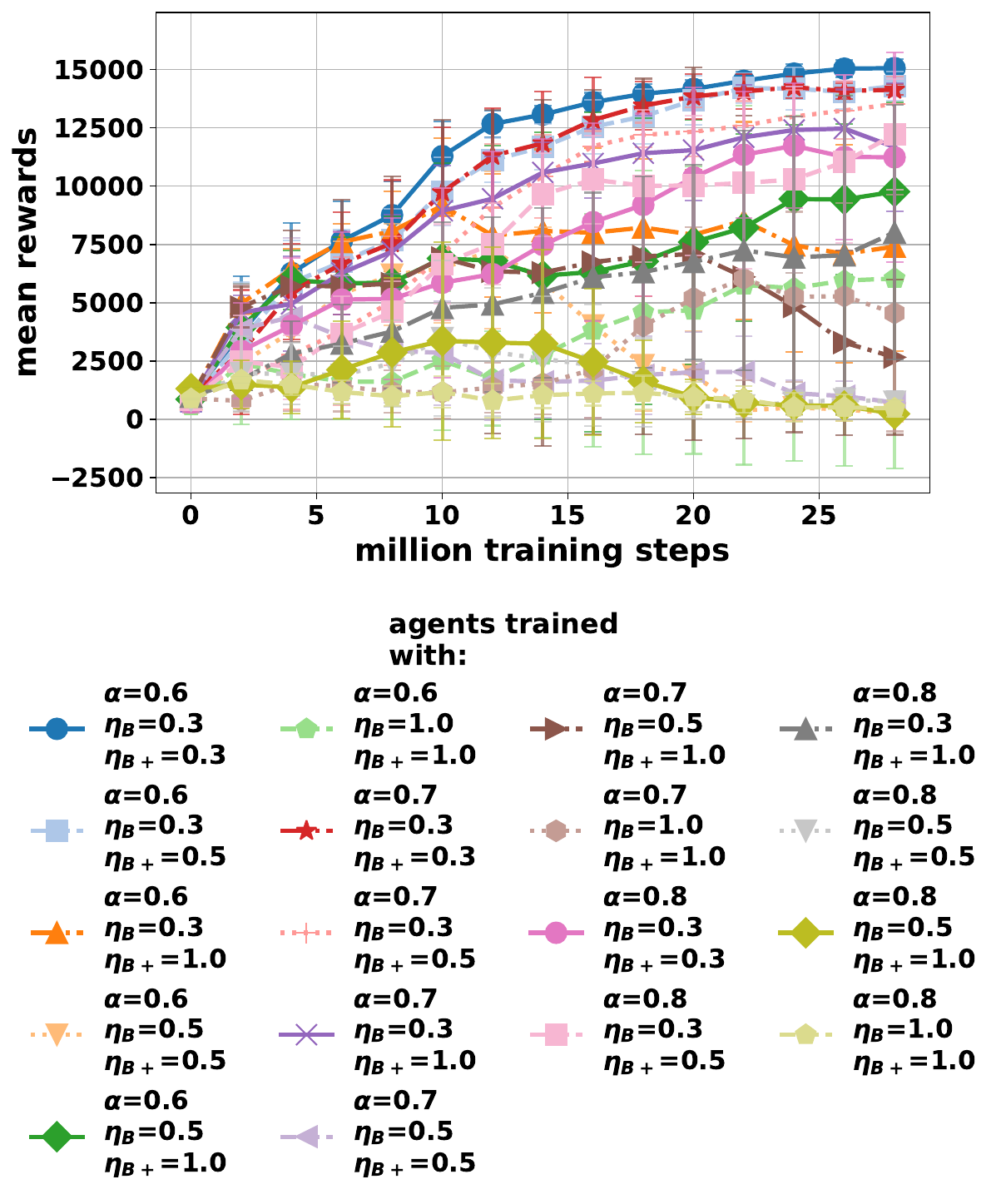}
                    \caption{training; eval $\eta=0.05$}
                \end{subfigure}\hfill
                \begin{subfigure}[t]{0.355\textwidth}
                    \centering
                    \includegraphics[width=\linewidth]{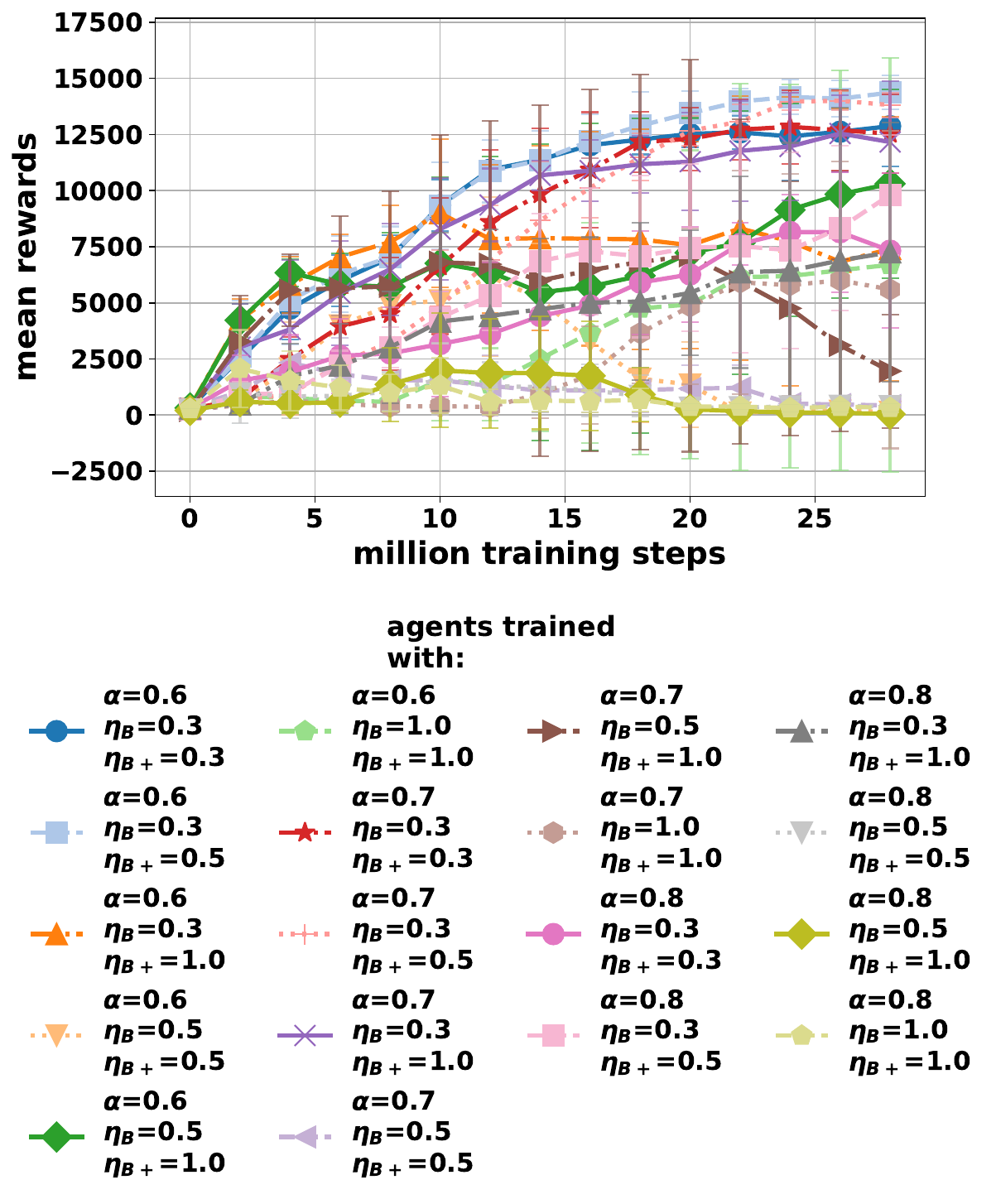}
                    \caption{training; eval $\eta=0.15$}
                \end{subfigure}
                \end{adjustwidth}
            \end{figure}
                            
            % ==================== FGM-C ====================
            \begin{figure}[!htp]
                \centering
                \begin{adjustwidth}{-0cm}{-1.3cm} % extend left/right by 1.5cm (tune this)
                \centering
                \caption{Evaluation curves during training with alpha reward preserving strategy under \textbf{FGM-C} attacks for evaluation settings $\eta \in \{0, 0.05, 0.15\}$.}
                \label{fig:alpha-training-curves-fgm-c}
                \begin{subfigure}[t]{0.355\textwidth}
                    \centering
                    \includegraphics[width=\linewidth]{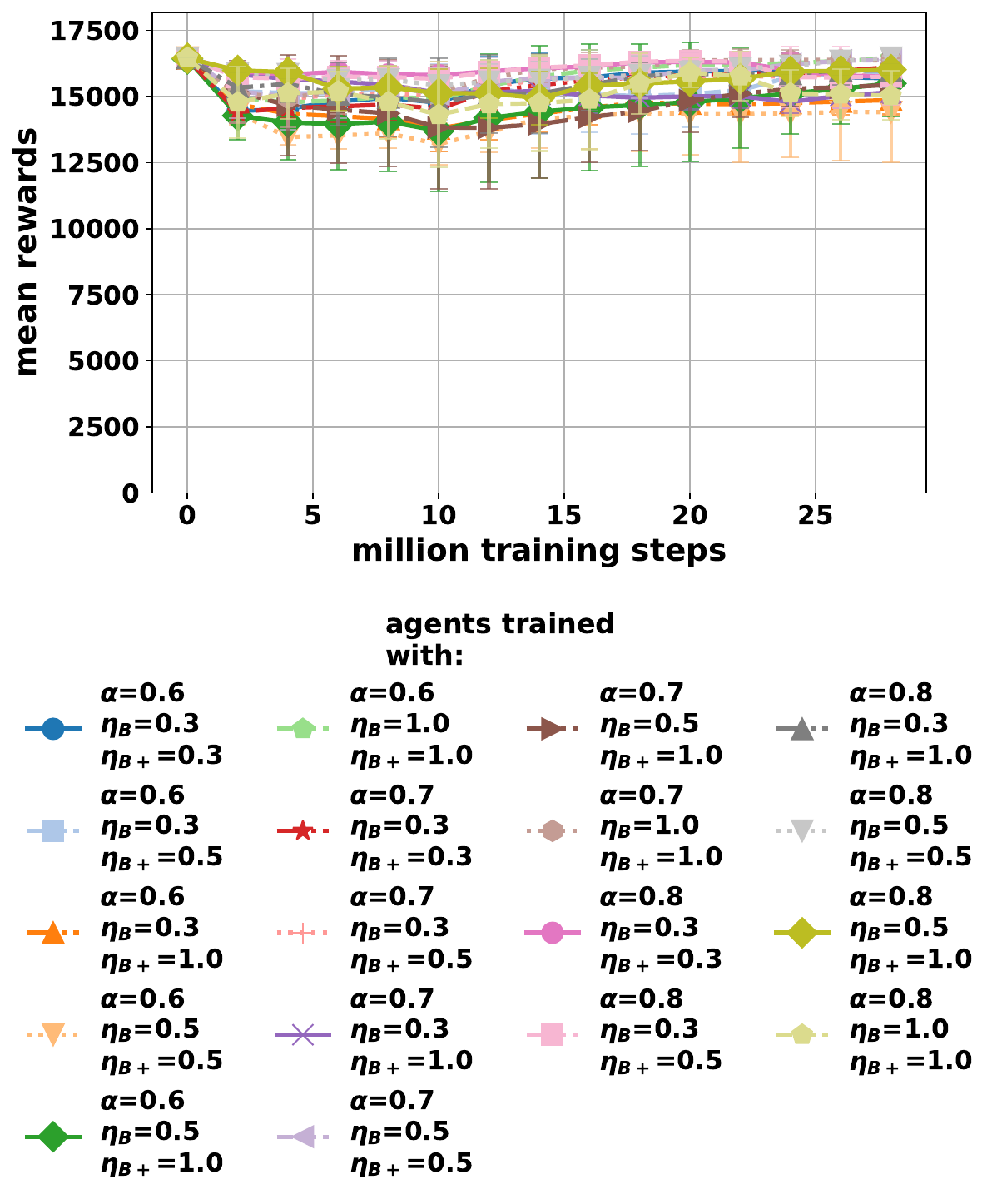}
                    \caption{training; eval $\eta=0$}
                \end{subfigure}\hfill
                \begin{subfigure}[t]{0.355\textwidth}
                    \centering
                    \includegraphics[width=\linewidth]{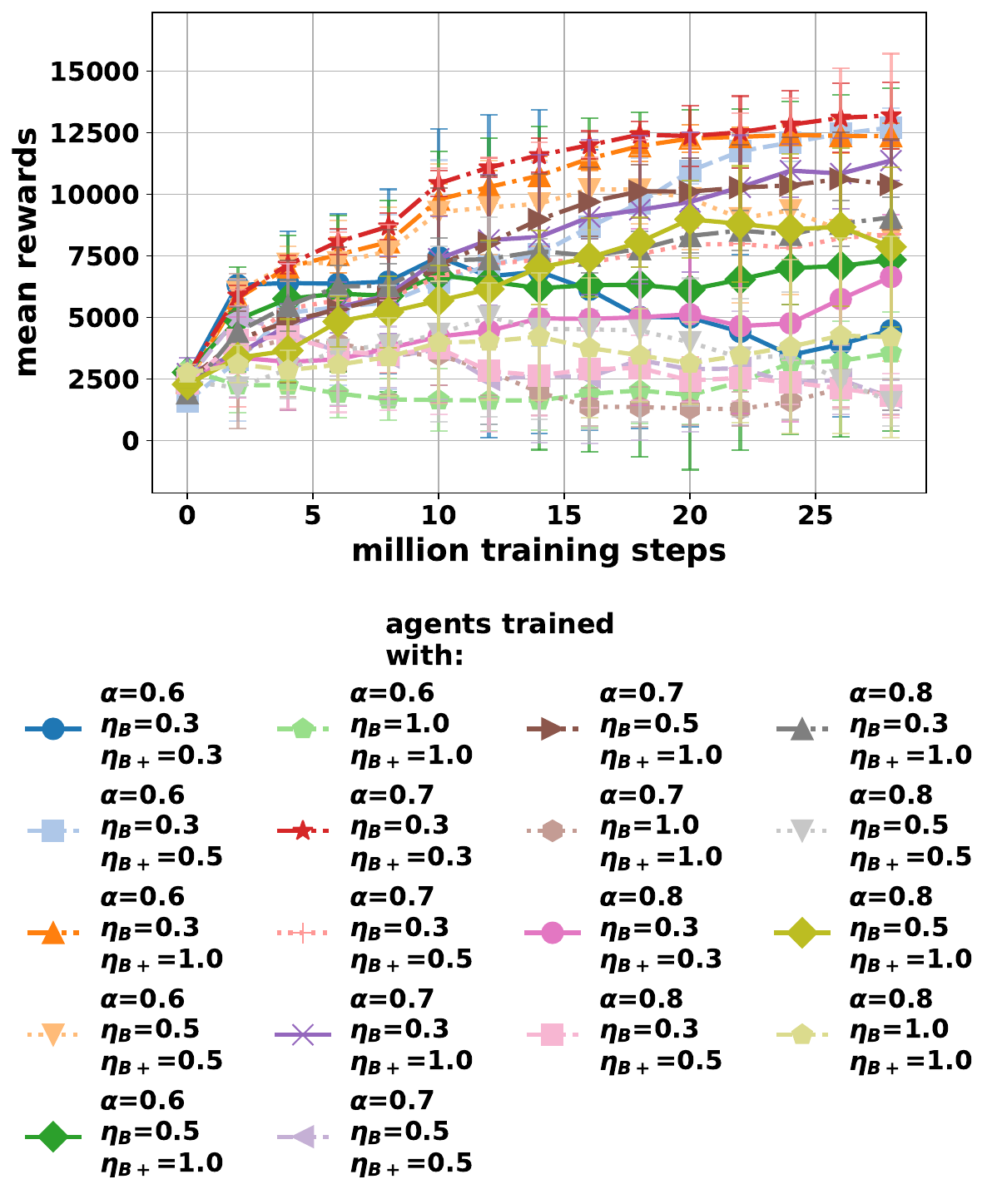}
                    \caption{training; eval $\eta=0.05$}
                \end{subfigure}\hfill
                \begin{subfigure}[t]{0.355\textwidth}
                    \centering
                    \includegraphics[width=\linewidth]{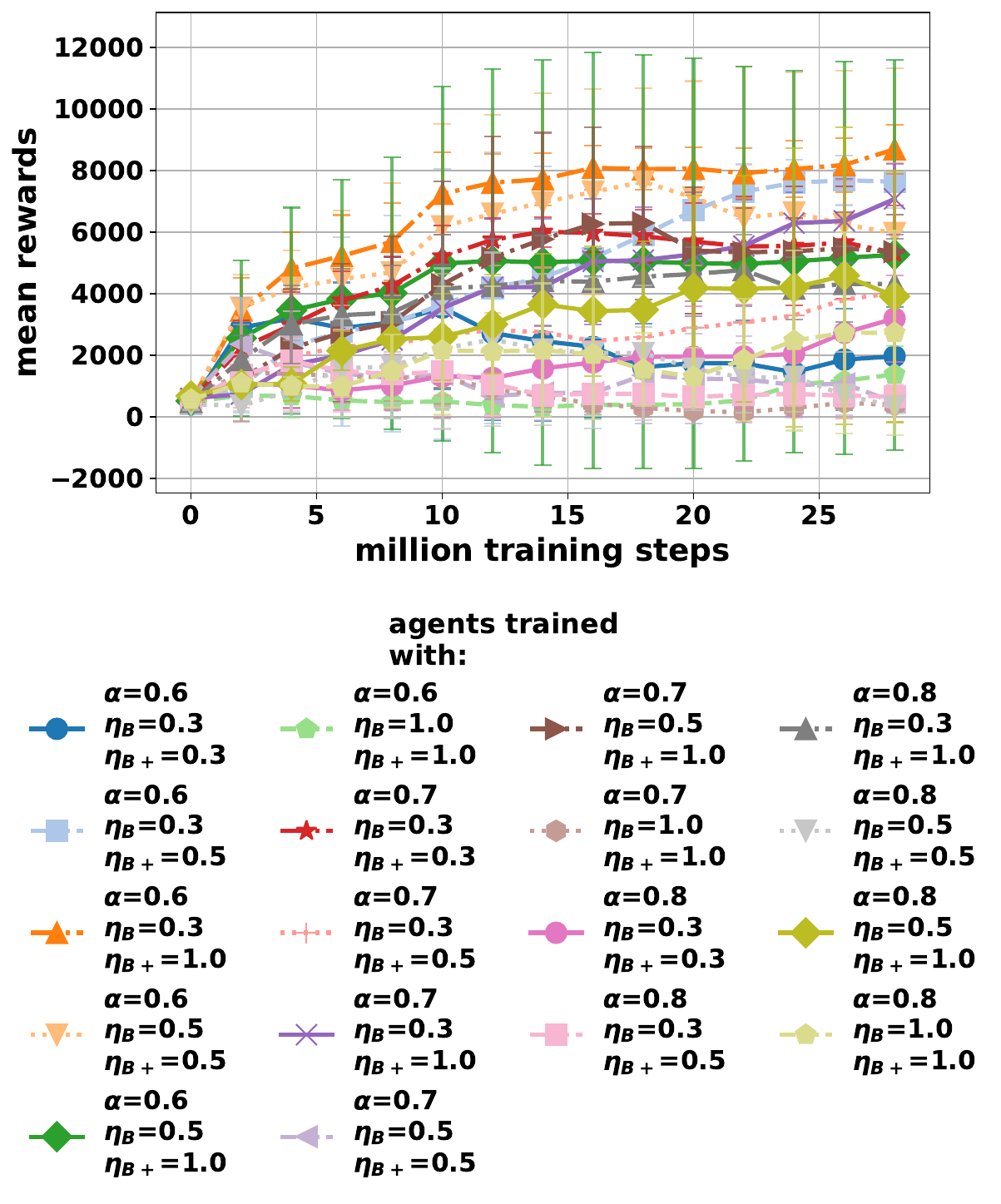}
                    \caption{training; eval $\eta=0.15$}
                \end{subfigure}
                \end{adjustwidth}
            \end{figure}

            % ==================== RUA ====================
            \begin{figure}[!htp]
                \centering
                \begin{adjustwidth}{-0cm}{-1.3cm} % extend left/right by 1.5cm (tune this)
                \centering
                \caption{Evaluation curves during training with alpha reward preserving strategy under \textbf{RUA} attacks for evaluation settings $\eta \in \{0, 0.05, 0.15\}$.}
                \label{fig:alpha-training-curves-rua}
                \begin{subfigure}[t]{0.355\textwidth}
                    \centering
                    \includegraphics[width=\linewidth]{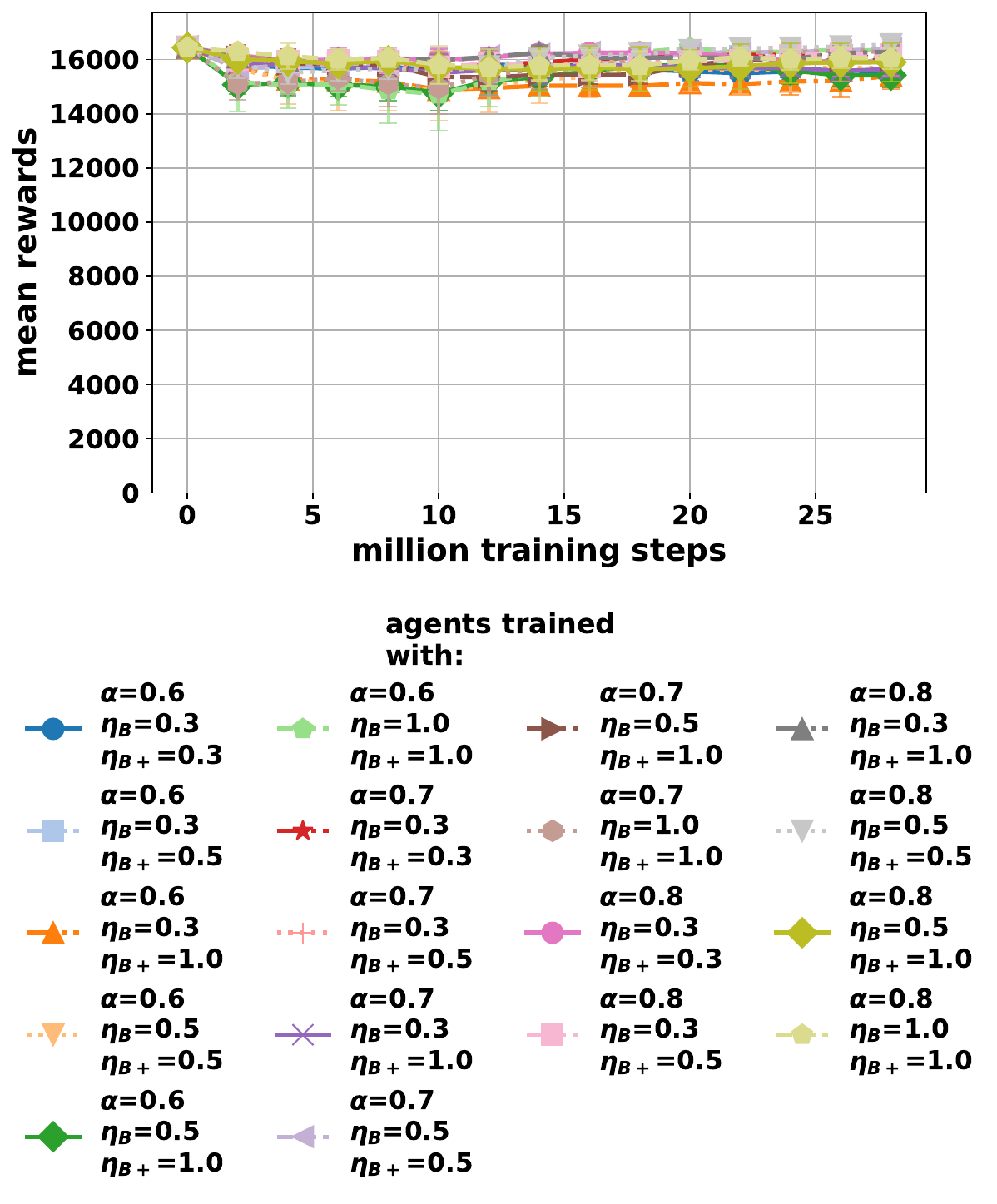}
                    \caption{training; eval $\eta=0$}
                \end{subfigure}\hfill
                \begin{subfigure}[t]{0.355\textwidth}
                    \centering
                    \includegraphics[width=\linewidth]{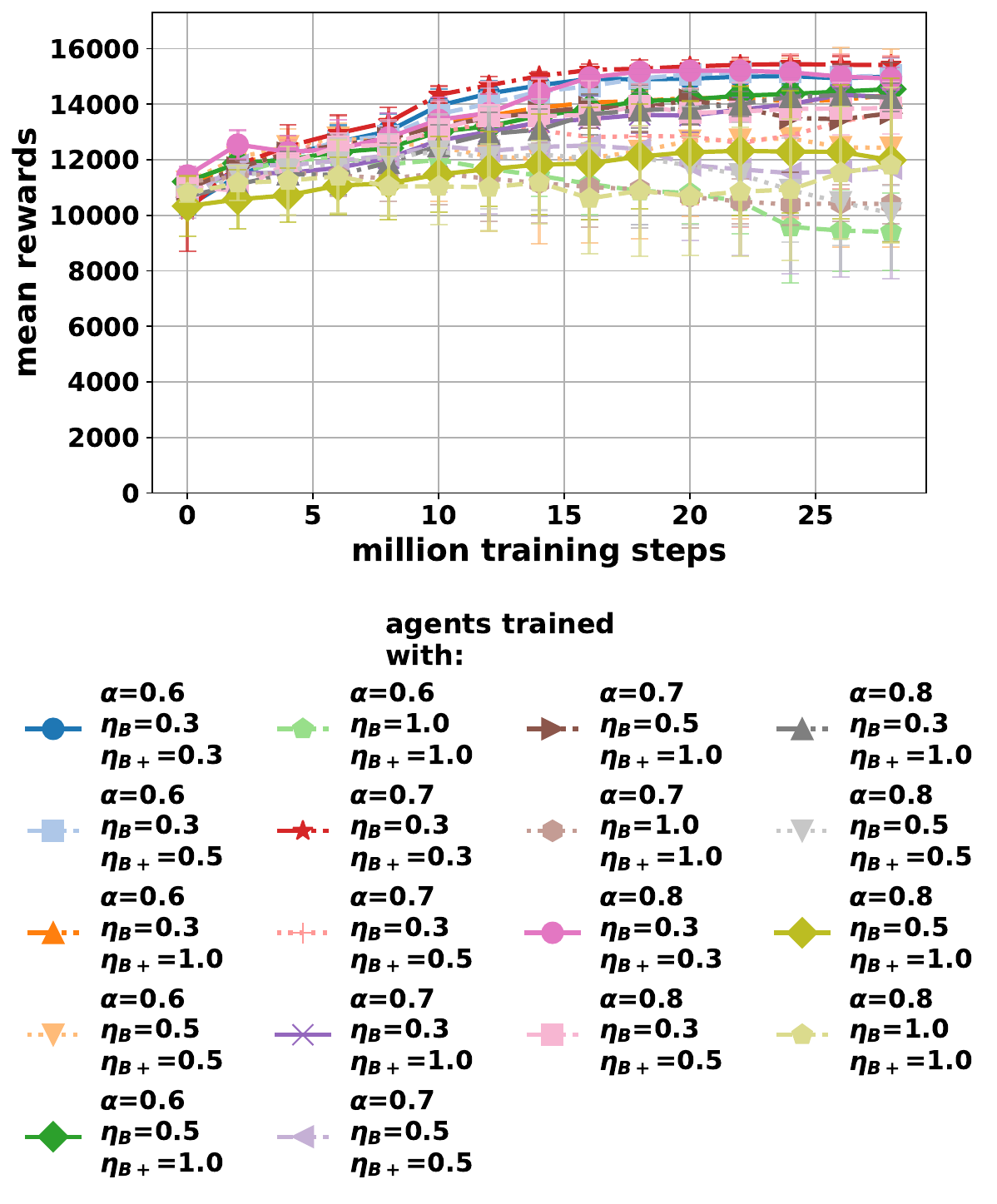}
                    \caption{training; eval $\eta=0.05$}
                \end{subfigure}\hfill
                \begin{subfigure}[t]{0.355\textwidth}
                    \centering
                    \includegraphics[width=\linewidth]{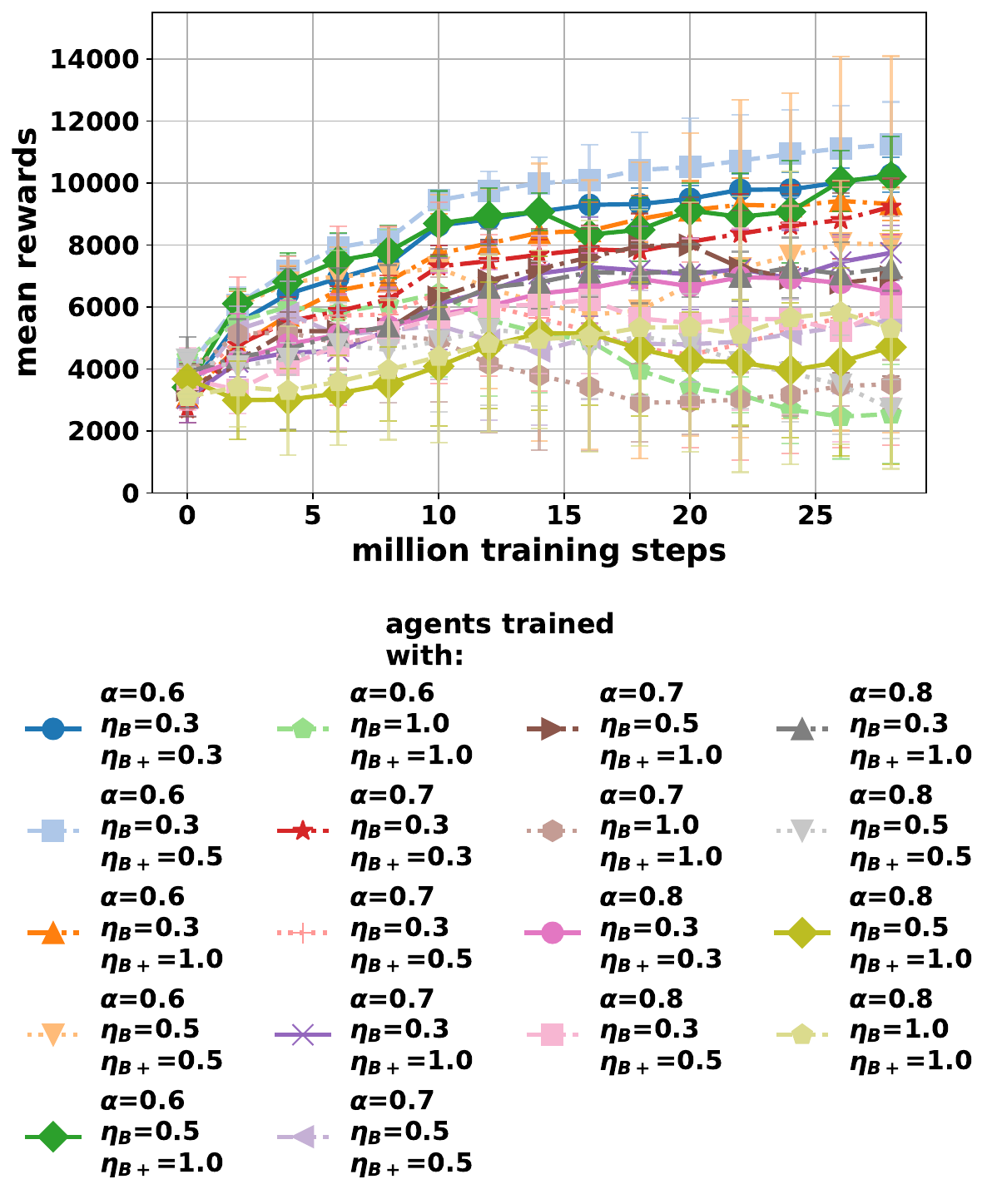}
                    \caption{training; eval $\eta=0.15$}
                \end{subfigure}
                \end{adjustwidth}
            \end{figure}

            % ==================== Evals ====================
            \begin{figure}[!htp]
                \centering
                \begin{adjustwidth}{-0cm}{-1.3cm} % extend left/right by 1.5cm (tune this)
                \centering
                \caption{Evaluation of the agents trained with alpha reward preserving strategy against \textbf{FGM\_QAC}, \textbf{FGM\_C} and \textbf{RUA} for $\eta \in \{0, 0.01, 0.05, 0.1, 0.15, 0.2\}$.}
                \label{fig:alpha-evaluations}
                \begin{subfigure}[t]{0.355\textwidth}
                    \vspace{0pt}\centering
                    \includegraphics[width=\linewidth]{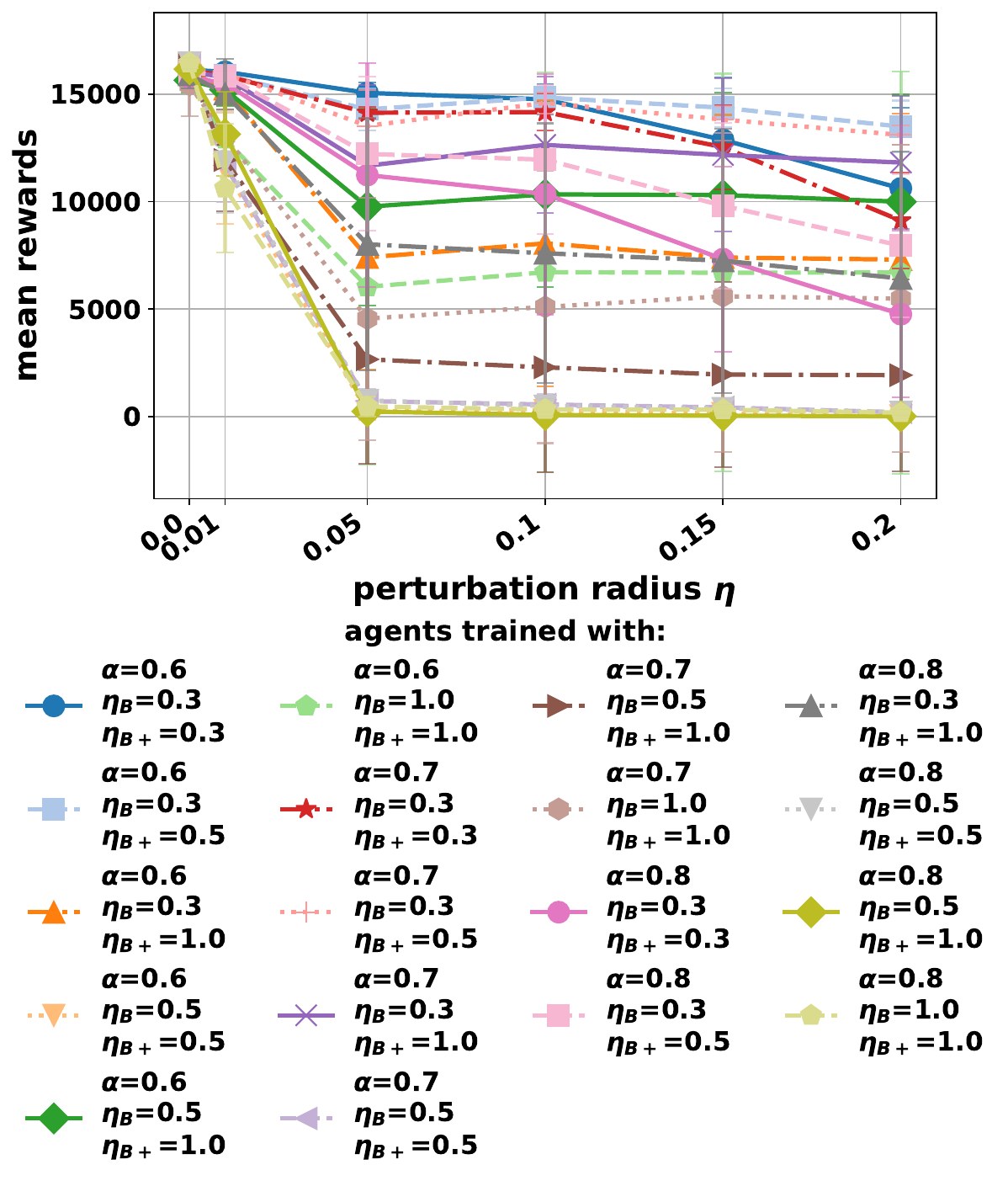}
                    \caption{FGM\_QAC : eval $\eta \in [0,0.2]$}
                \end{subfigure}\hfill
                \begin{subfigure}[t]{0.355\textwidth}
                    \vspace{0pt}\centering
                    \includegraphics[width=\linewidth]{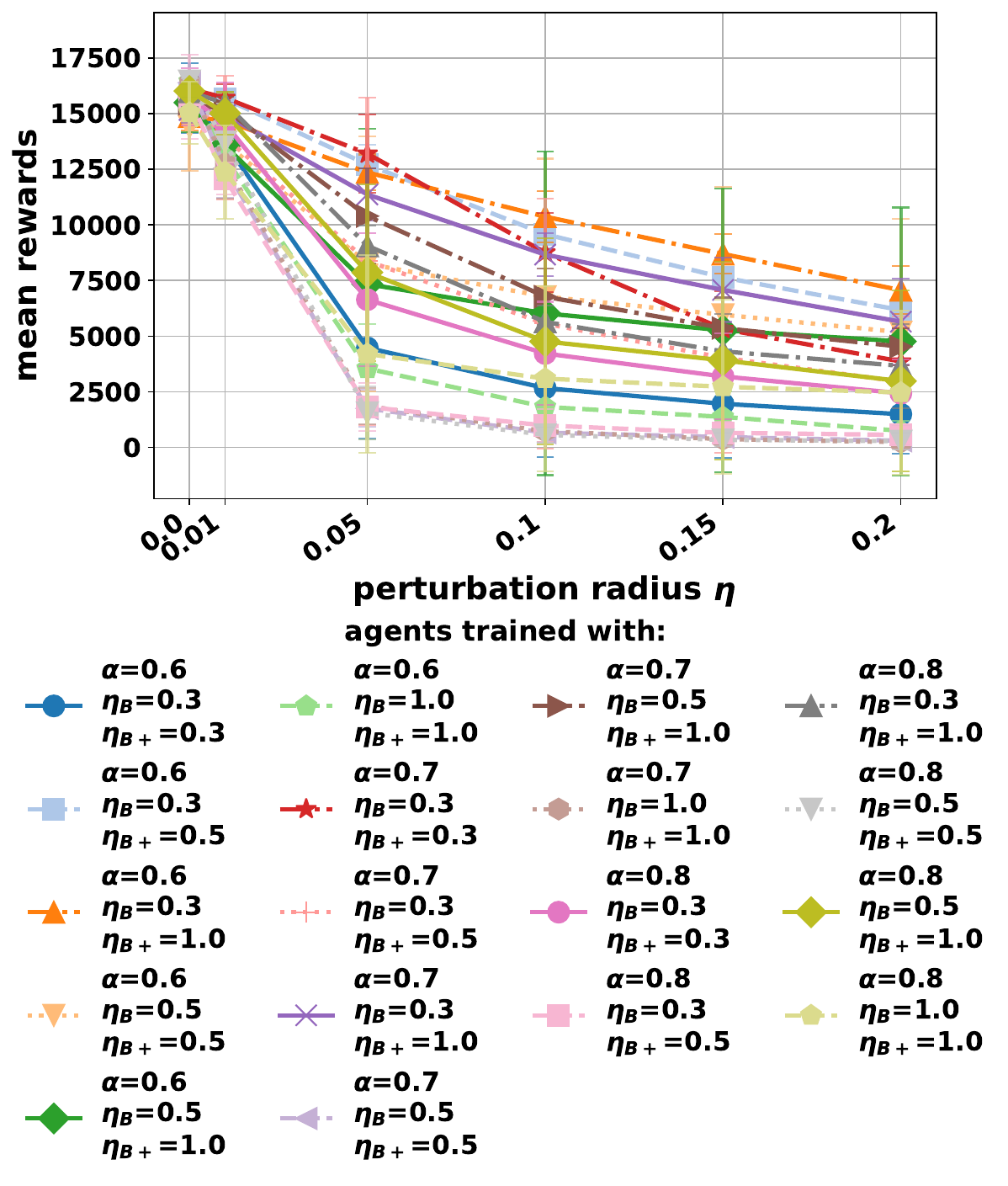}
                    \caption{FGM\_C : eval $\eta \in [0,0.2]$}
                \end{subfigure}\hfill
                \begin{subfigure}[t]{0.355\textwidth}
                    \vspace{0pt}\centering
                    \includegraphics[width=\linewidth]{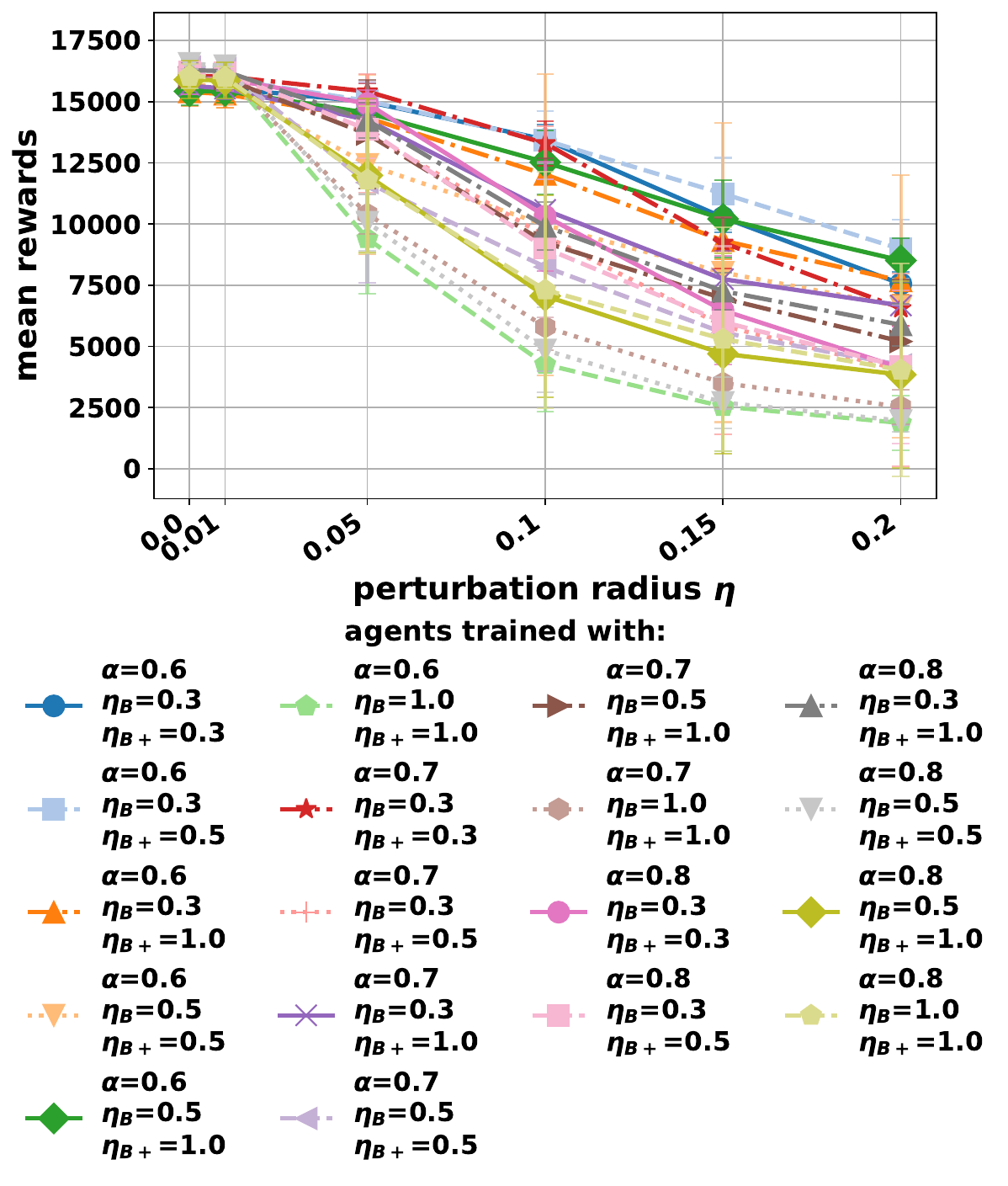}
                    \caption{RUA : eval $\eta \in [0,0.2]$}
                \end{subfigure}
                \end{adjustwidth}
            \end{figure}

        \newpage
            
        \subsubsection{Training with constant and random perturbation strength for different attack methods}
            \label{sec:supplementary_material_const}
            We then compare the $\alpha$-reward-preserving strategy against simpler baselines that use a \emph{random uniform} perturbation radius during training.
            For each attack (FGM\_QAC, FGM\_C and RUA), we grid-search over the perturbation parameters $\eta_B$ and include, as a reference point, the best $\alpha$ configuration identified in the previous ablation.
            As above, we report in Figures~\ref{fig:eta-training-curves-fgm-qac},~\ref{fig:eta-training-curves-fgm-c} and ~\ref{fig:eta-training-curves-rua} evaluation curves collected during training at fixed perturbation magnitudes $\eta \in \{0, 0.05, 0.15\}$, and in Figure~\ref{fig:eta-evaluations} evaluation of the final agent over a broader range $\eta \in \{0, 0.01, 0.05, 0.1, 0.15, 0.2\}$. 
            
            %We observe that, while our $\alpha$-reward-preserving trained agent is one of the best in most settings. It is only slightly surpassed by specific baselines in some particular test attack magnitudes, but such baselines that are performant for specific values are then dominated in the nominal MDP or for more strong attacks.  

            We observe that our $\alpha$-reward-preserving trained agent performs among the best across most settings. In some specific test attack magnitudes, it is slightly outperformed by baselines tailored to those particular conditions. However, these specialized baselines are then outperformed by our agent in the nominal environment or under stronger attacks, highlighting the overall robustness and generality of $\alpha$-reward-preserving training.

            % ==================== FGM-QAC ====================
            \begin{figure}[!htp]
                \centering
                \begin{adjustwidth}{-0cm}{-1.3cm} % extend left/right by 1.5cm (tune this)
                \centering
                \caption{Evaluation curves during training with random uniform perturbation radius strategy under \textbf{FGM-QAC} attacks for evaluation settings $\eta \in \{0, 0.05, 0.15\}$.}
                \label{fig:eta-training-curves-fgm-qac}
                \begin{subfigure}[t]{0.355\textwidth}
                    \centering
                    \includegraphics[width=\linewidth]{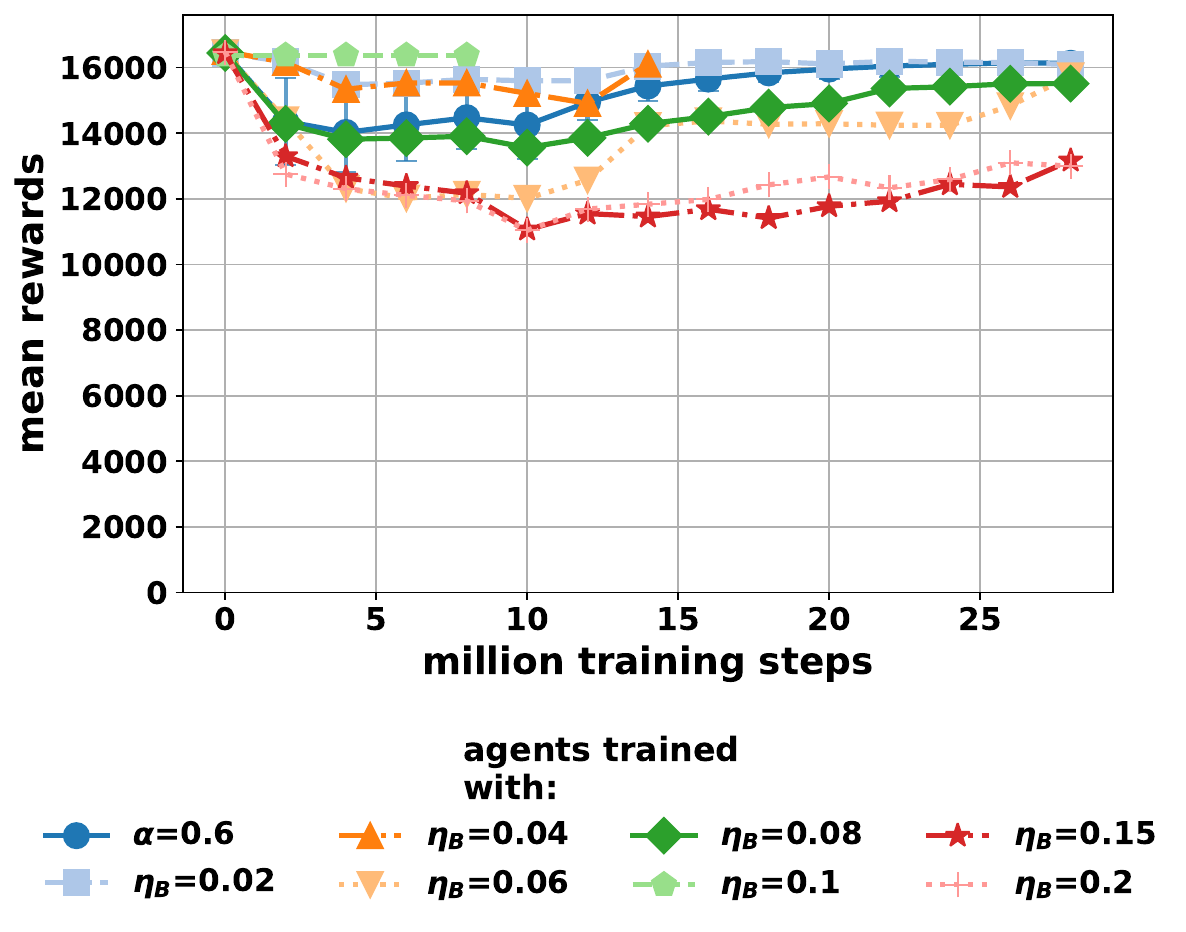}
                    \caption{eval $\eta=0$}
                \end{subfigure}\hfill
                \begin{subfigure}[t]{0.355\textwidth}
                    \centering
                    \includegraphics[width=\linewidth]{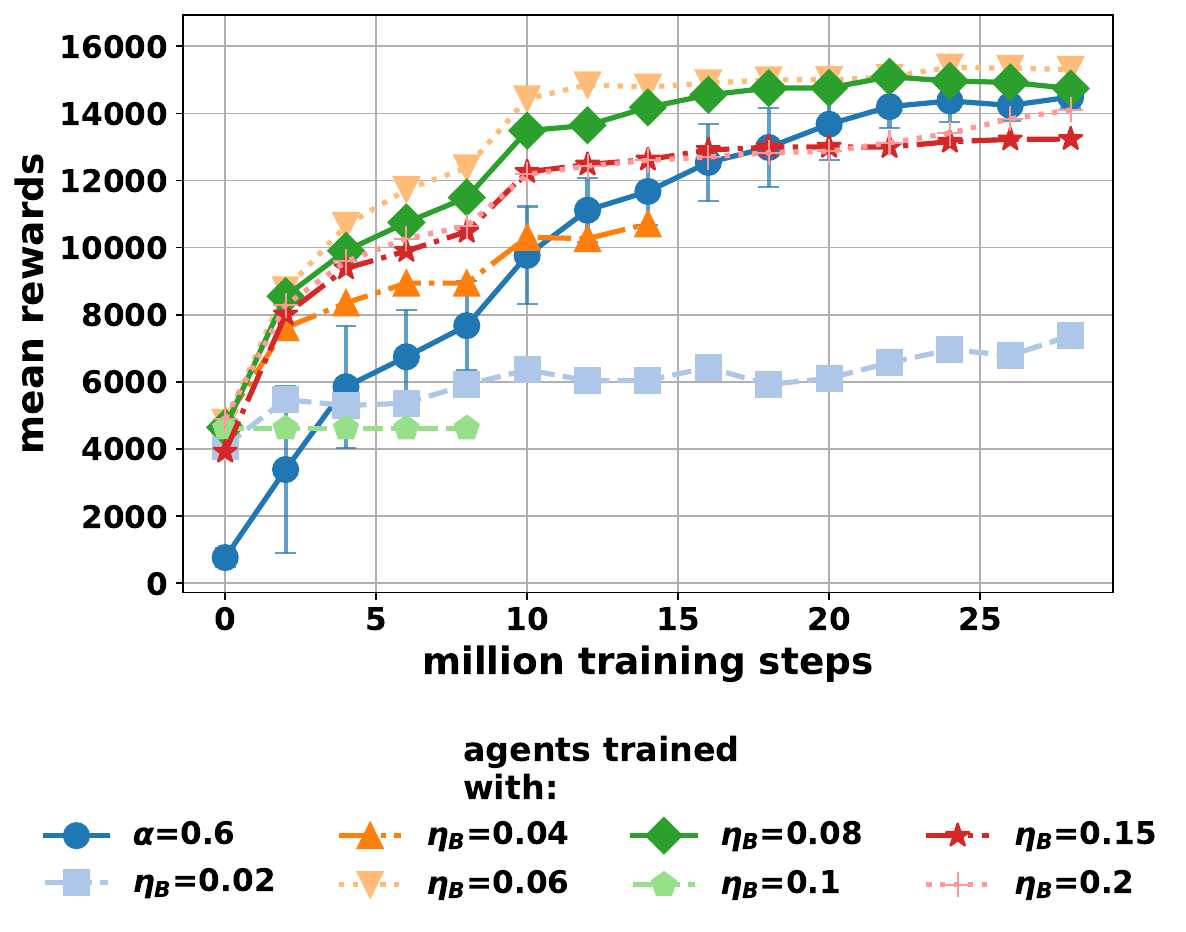}
                    \caption{eval $\eta=0.05$}
                \end{subfigure}\hfill
                \begin{subfigure}[t]{0.355\textwidth}
                    \centering
                    \includegraphics[width=\linewidth]{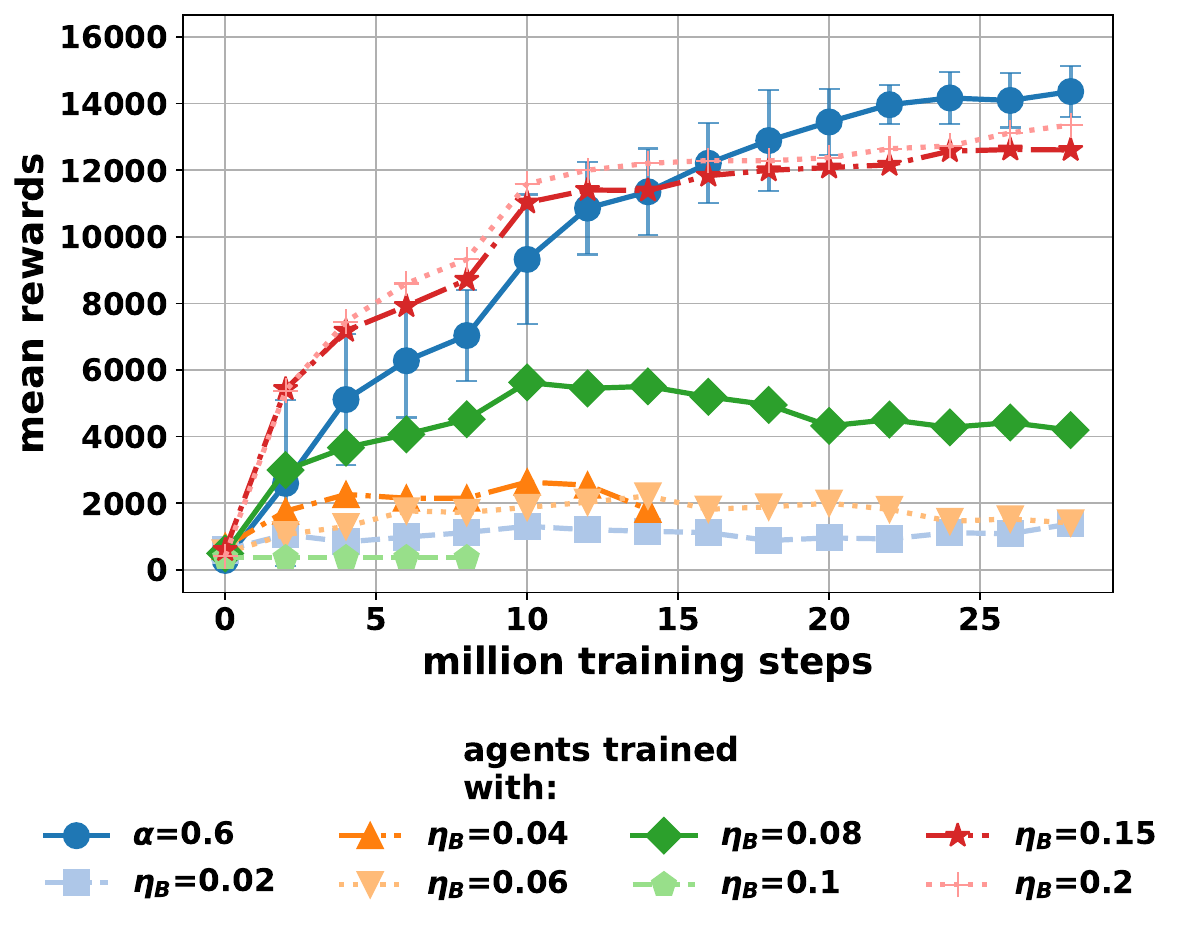}
                    \caption{eval $\eta=0.15$}
                \end{subfigure}
                \end{adjustwidth}
            \end{figure}
                            
            % ==================== FGM-C ====================
            \begin{figure}[!htp]
                \centering
                \begin{adjustwidth}{-0cm}{-1.3cm} % extend left/right by 1.5cm (tune this)
                \centering
                \caption{Evaluation curves during training with random uniform perturbation radius strategy under \textbf{FGM-C} attacks for evaluation settings $\eta \in \{0, 0.05, 0.15\}$.}
                \label{fig:eta-training-curves-fgm-c}
                \begin{subfigure}[t]{0.355\textwidth}
                    \centering
                    \includegraphics[width=\linewidth]{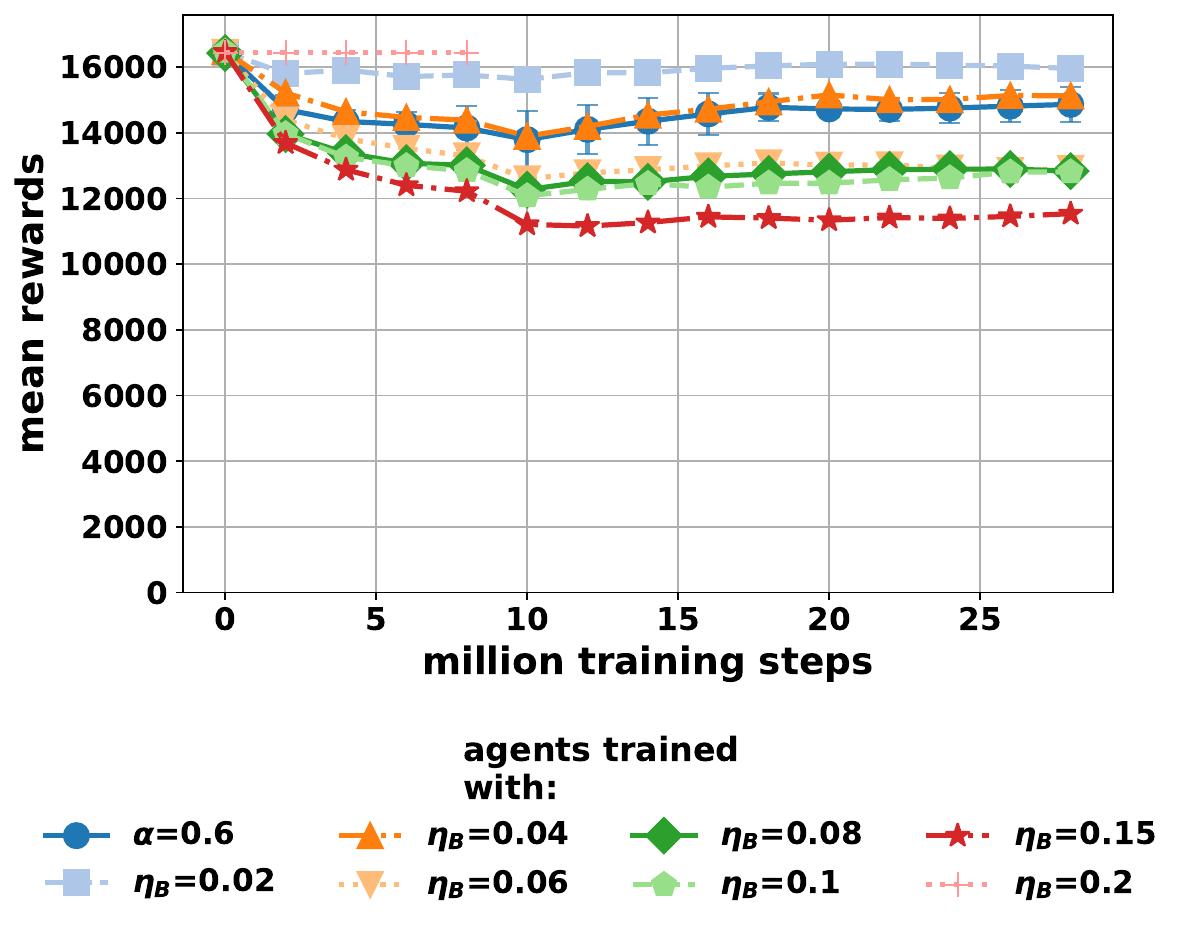}
                    \caption{eval $\eta=0$}
                \end{subfigure}\hfill
                \begin{subfigure}[t]{0.355\textwidth}
                    \centering
                    \includegraphics[width=\linewidth]{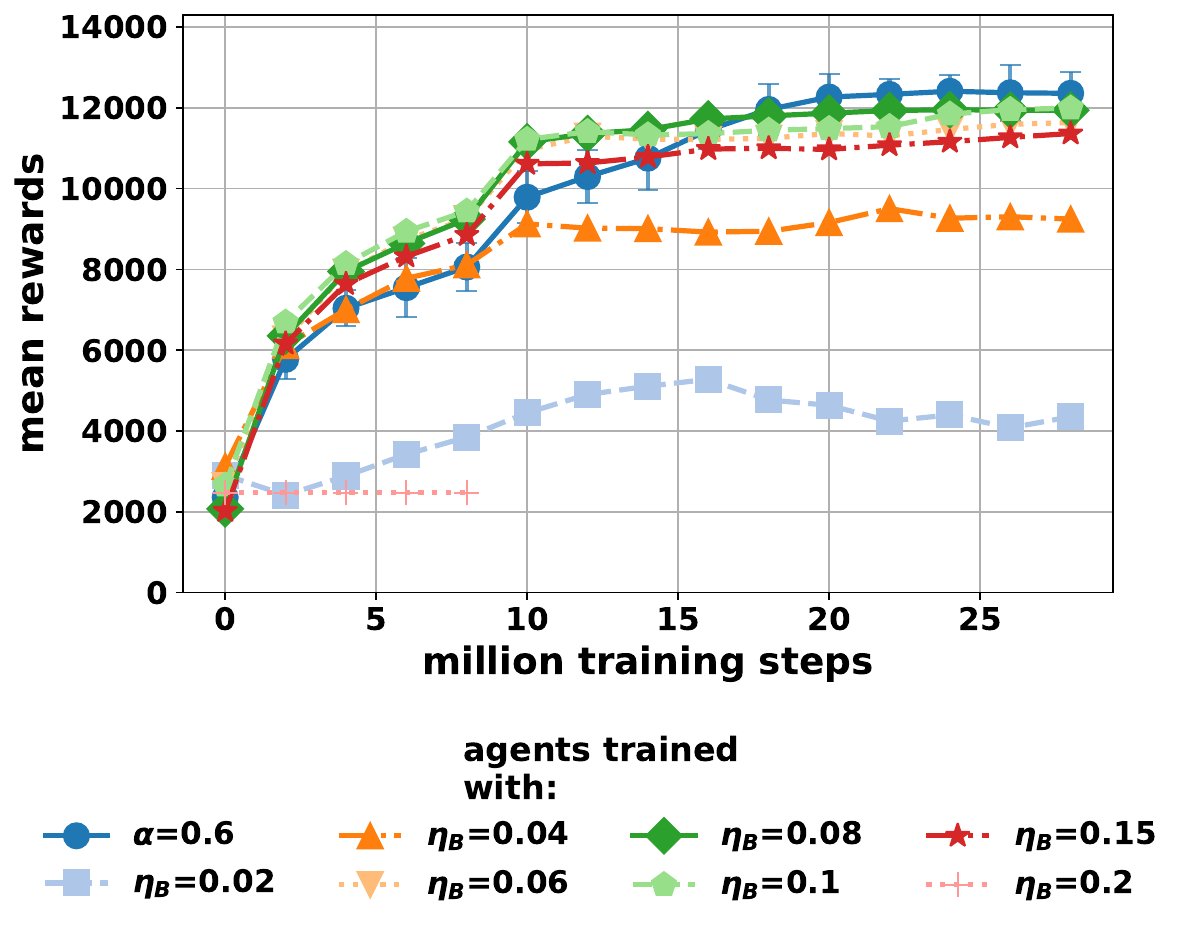}
                    \caption{eval $\eta=0.05$}
                \end{subfigure}\hfill
                \begin{subfigure}[t]{0.355\textwidth}
                    \centering
                    \includegraphics[width=\linewidth]{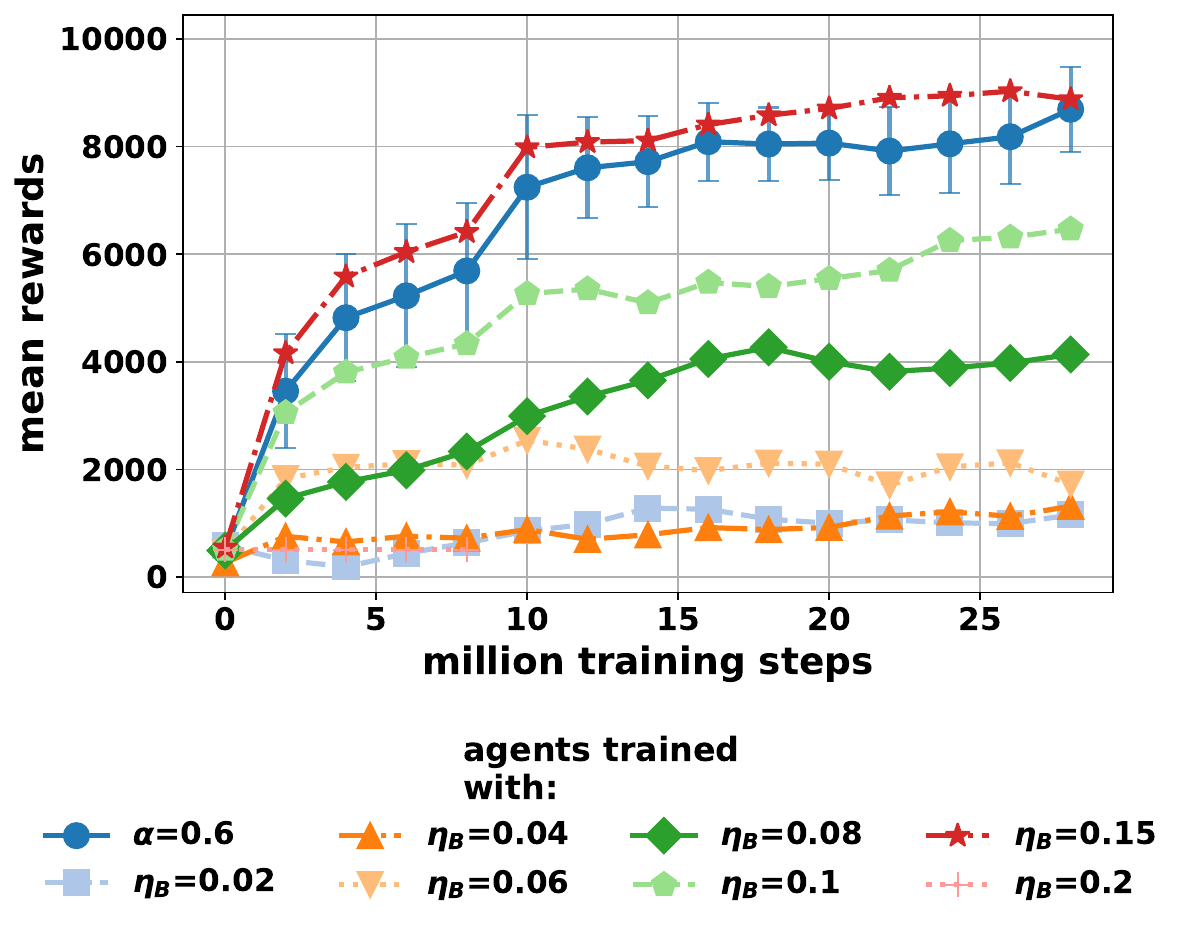}
                    \caption{eval $\eta=0.15$}
                \end{subfigure}
                \end{adjustwidth}
            \end{figure}

            % ==================== RUA ====================
            \begin{figure}[!htp]
                \vspace*{-\topskip}
                \centering
                \begin{adjustwidth}{-0cm}{-1.3cm} % extend left/right by 1.5cm (tune this)
                \centering
                \caption{Evaluation curves during training with random uniform perturbation radius strategy under \textbf{RUA} attacks for evaluation settings $\eta \in \{0, 0.05, 0.15\}$.}
                \label{fig:eta-training-curves-rua}
                \begin{subfigure}[t]{0.355\textwidth}
                    \centering
                    \includegraphics[width=\linewidth]{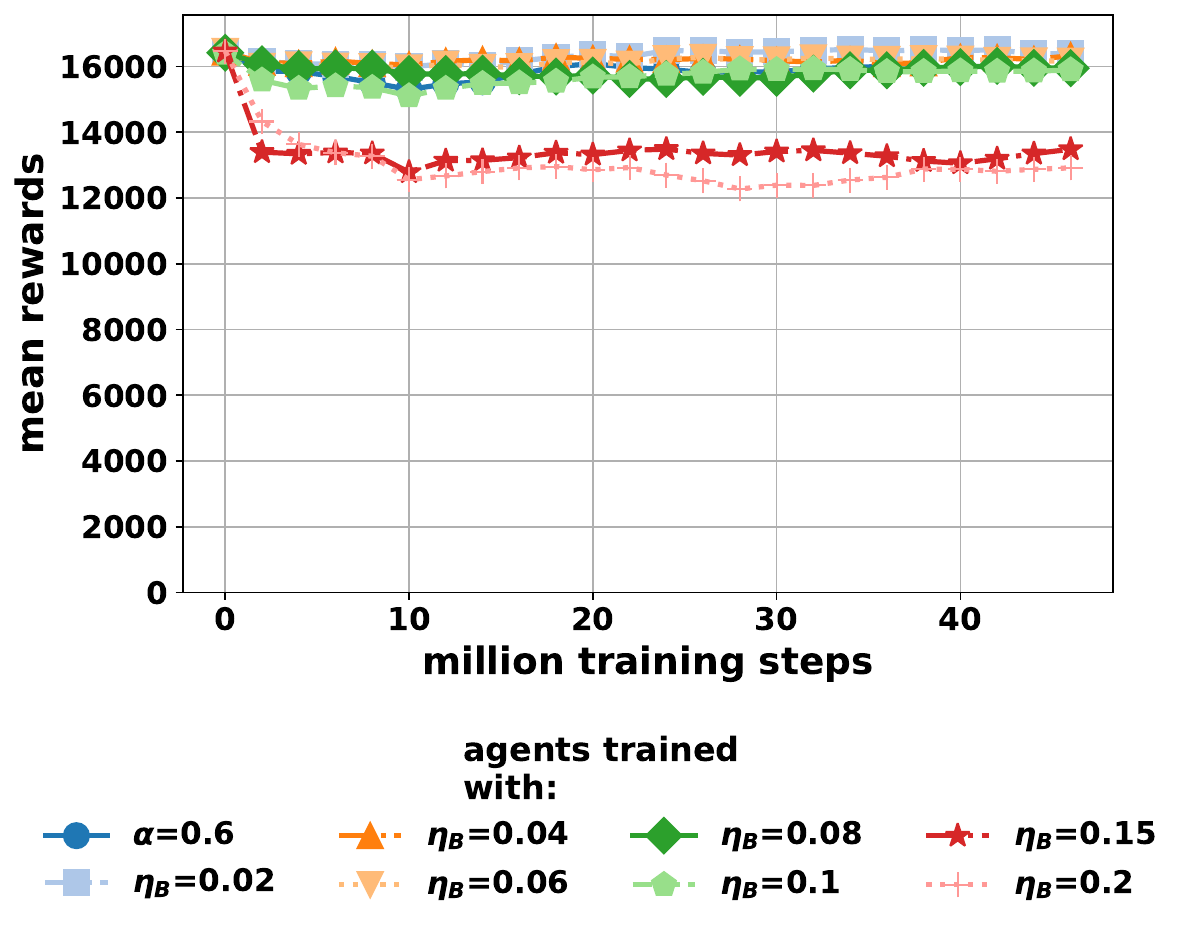}
                    \caption{eval $\eta=0$}
                \end{subfigure}\hfill
                \begin{subfigure}[t]{0.355\textwidth}
                    \centering
                    \includegraphics[width=\linewidth]{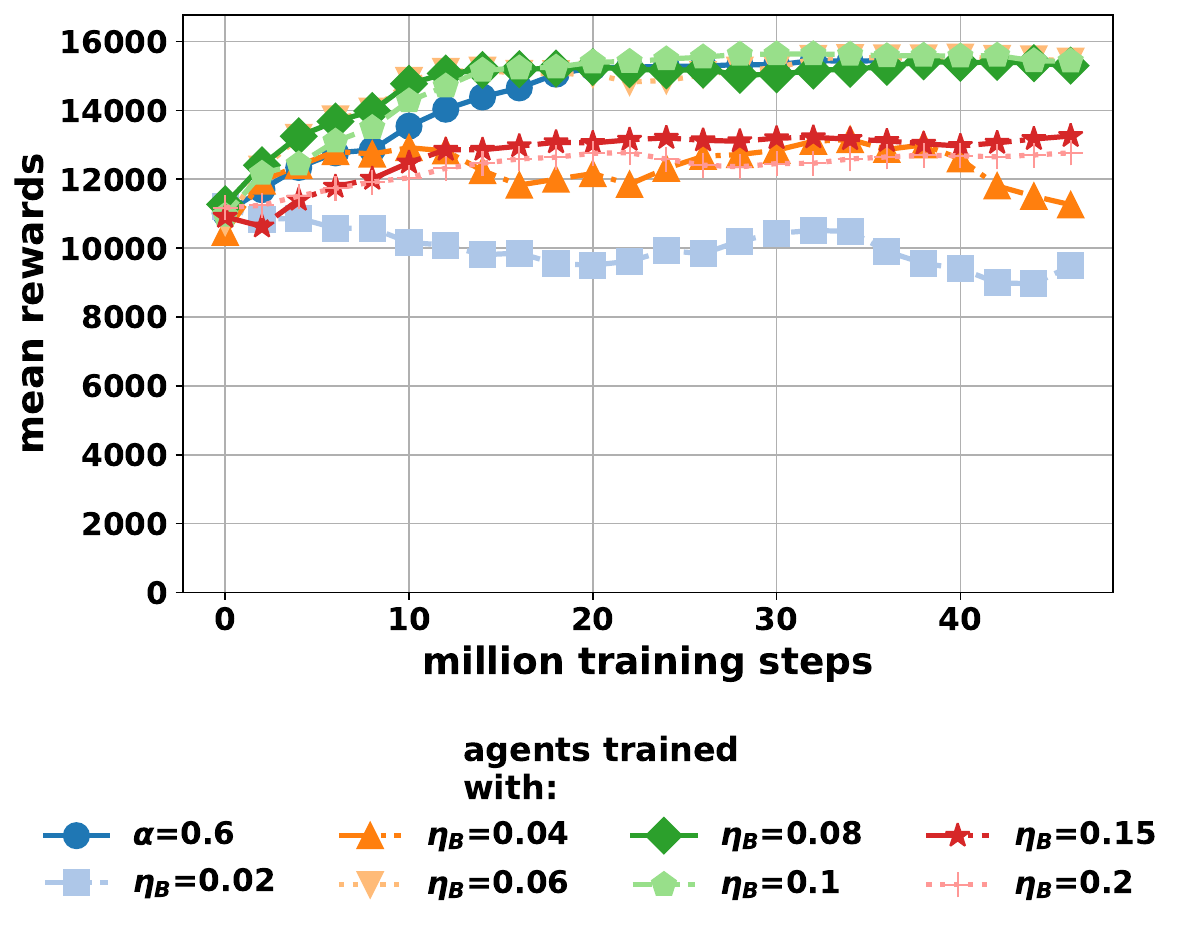}
                    \caption{eval $\eta=0.05$}
                \end{subfigure}\hfill
                \begin{subfigure}[t]{0.355\textwidth}
                    \centering
                    \includegraphics[width=\linewidth]{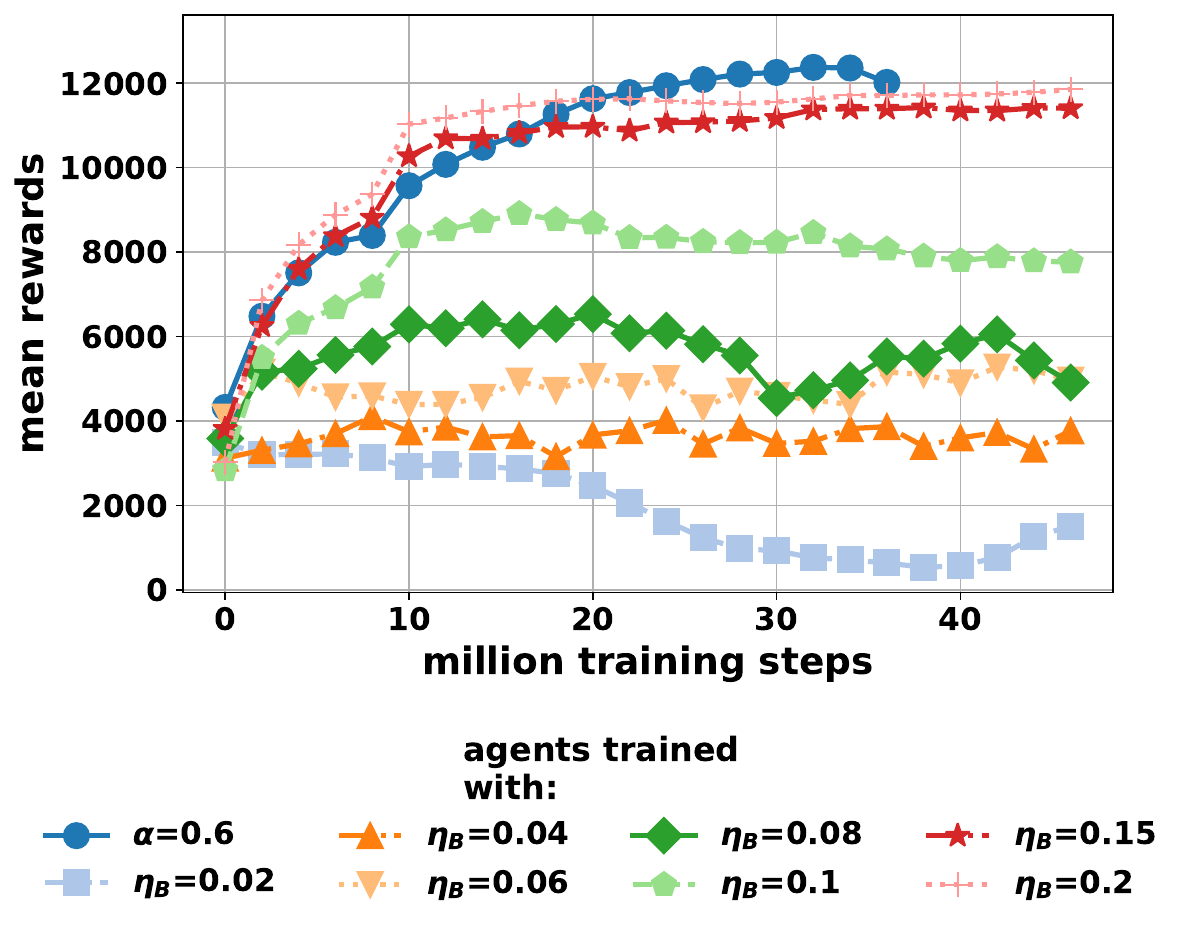}
                    \caption{eval $\eta=0.15$}
                \end{subfigure}
                \end{adjustwidth}
            \end{figure}

            % ==================== Evals ====================
            \begin{figure}[!htp]
                \vspace*{-\topskip}
                \centering
                \begin{adjustwidth}{-0cm}{-1.3cm} % extend left/right by 1.5cm (tune this)
                \centering
                \caption{Evaluation of the agents trained with constant or random uniform perturbation radius strategy against \textbf{FGM\_QAC}, \textbf{FGM\_C} and \textbf{RUA} for $\eta \in \{0, 0.01, 0.05, 0.1, 0.15, 0.2\}$.}
                \label{fig:eta-evaluations}
                \begin{subfigure}[t]{0.355\textwidth}
                    \vspace{0pt}\centering
                    \includegraphics[width=\linewidth]{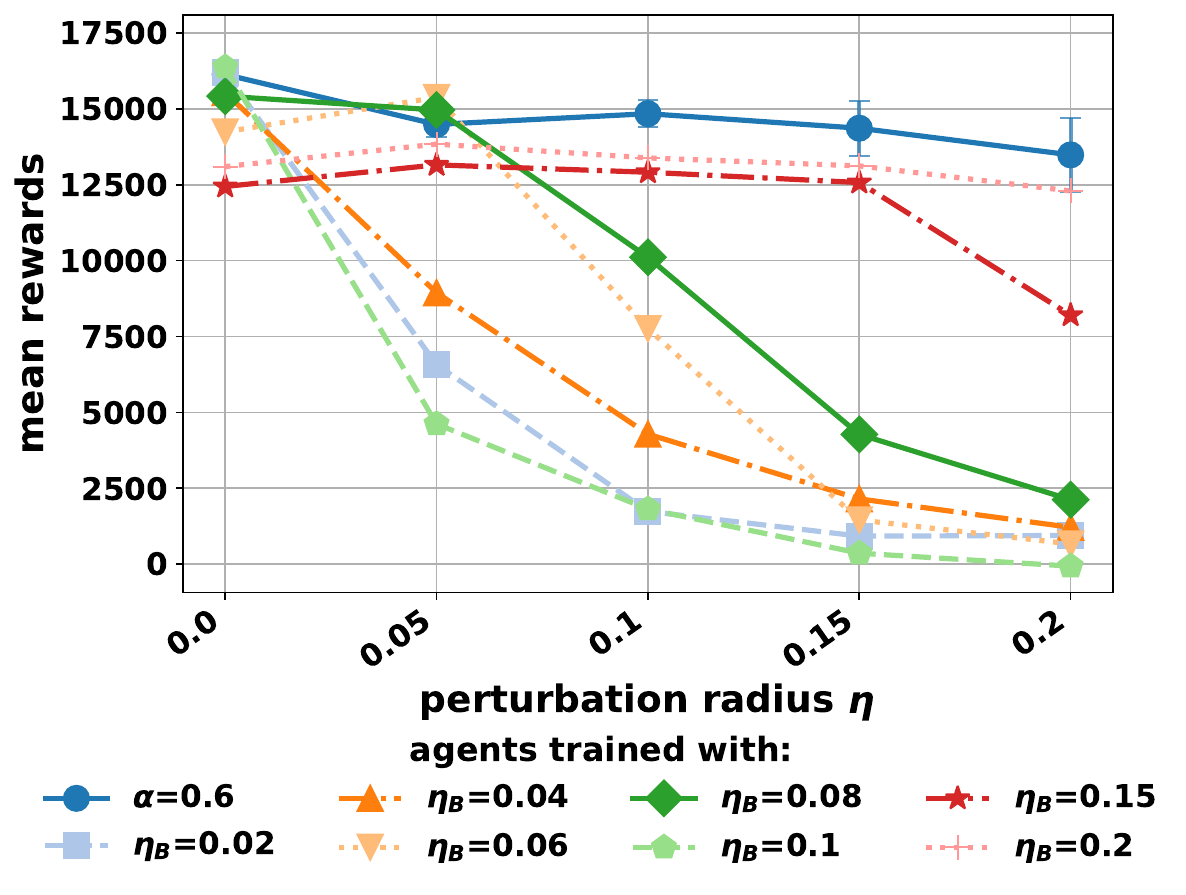}
                    \caption{FGM\_QAC : eval $\eta \in [0,0.2]$}
                \end{subfigure}\hfill
                \begin{subfigure}[t]{0.355\textwidth}
                    \vspace{0pt}\centering
                    \includegraphics[width=\linewidth]{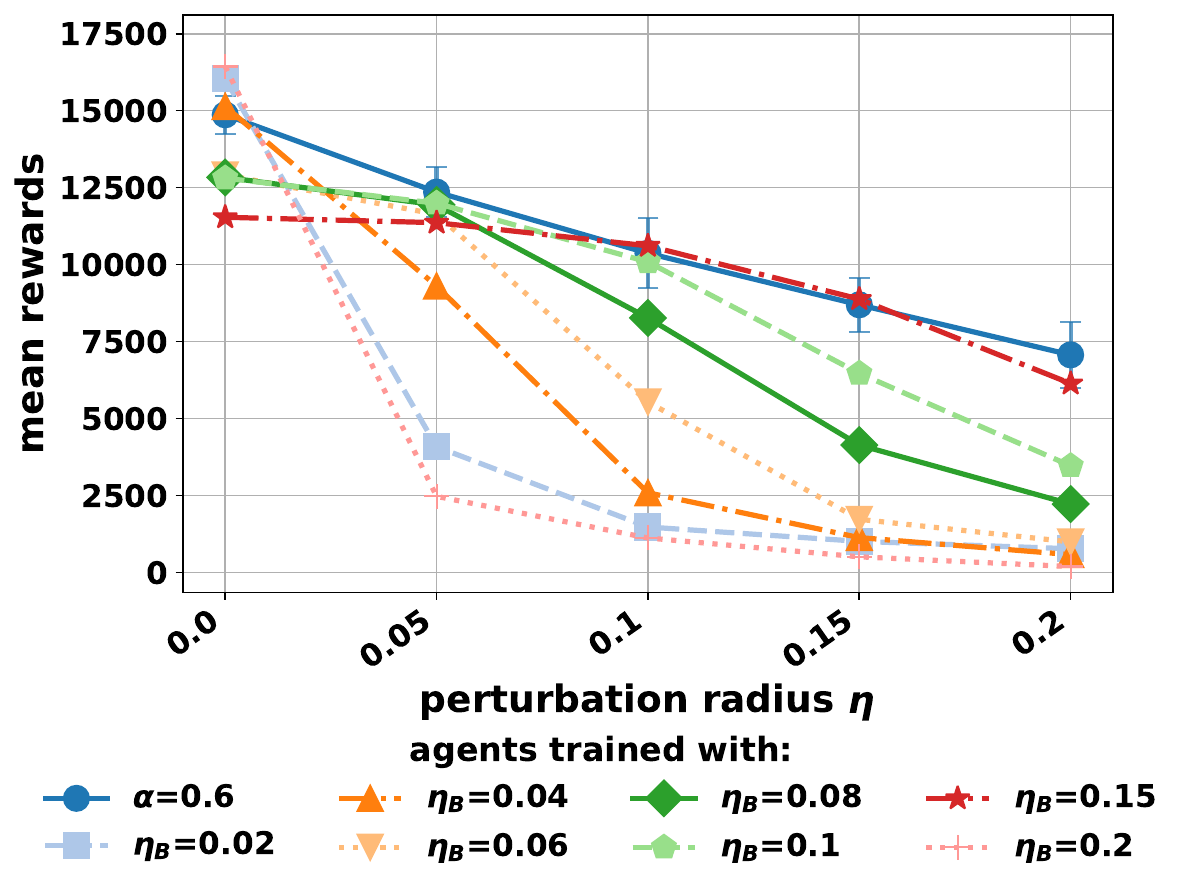}
                    \caption{FGM\_C : eval $\eta \in [0,0.2]$}
                \end{subfigure}\hfill
                \begin{subfigure}[t]{0.355\textwidth}
                    \vspace{0pt}\centering
                    \includegraphics[width=\linewidth]{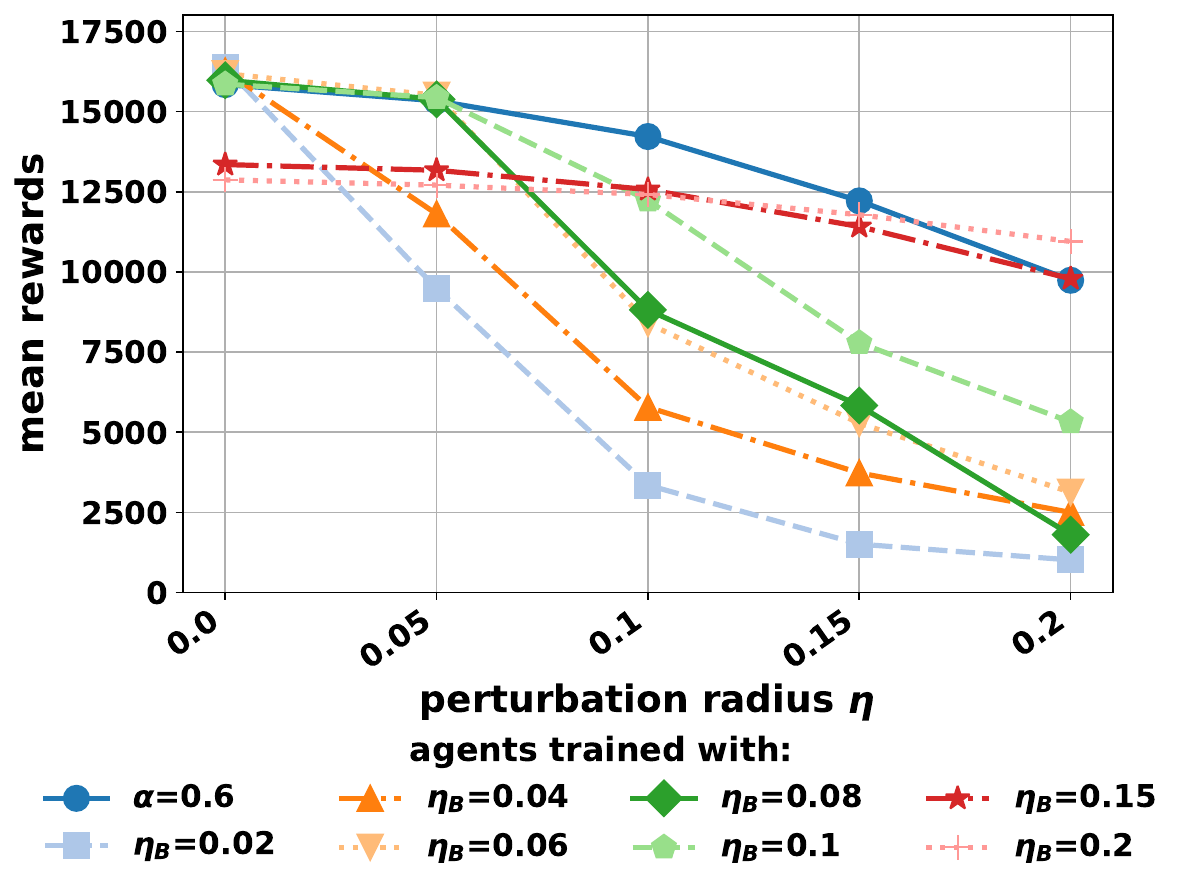}
                    \caption{RUA : eval $\eta \in [0,0.2]$}
                \end{subfigure}
                \end{adjustwidth}
            \end{figure}

\newpage
        
    \subsection{Extension of RVI for $\alpha$-reward-preserving attacks}
        \label{firstRVI}
        
        This section first recalls the classical Robust Value Iteration algorithm and then presents its extension for $\alpha$-reward-preserving attacks. 
        
        \subsubsection{Robust Value Iteration}
        
            First, we recall the classical Robust Value Iteration in algorithm \ref{algo:classicalRVI}.

            \begin{algorithm}[!ht]
                \caption{Robust Value Iteration with $r(s,a,s')$ and action-value function $Q$}
                \label{algo:classicalRVI}
                \begin{algorithmic}[1]
                    \REQUIRE State space $\mathcal{S}$, action space $\cal A$, reward $r(s,a,s')$, nominal MDP $\Omega$, uncertainty set $\mathcal{B}(s,a)$, discount $\gamma$
                    \ENSURE Robust Q-function $Q^{*,\Omega^{\xi^*}}$
                    
                    \STATE Initialize $Q(s,a) \leftarrow 0$ for all $(s,a)$
                    
                    \FOR{each iteration $k = 0,1,2,\dots$}
                      \FOR{each state $s \in \mathcal{S}$}
                        \FOR{each action $a \in \cal A$}
                          \STATE Compute worst-case expected Q over uncertainty set:
                          \STATE $Q_{\text{new}}(s,a) \leftarrow \min\limits_{P \in \mathcal{B}(s,a)} \sum\limits_{s'} P(s' \mid s,a)\Big[r(s,a,s') + \gamma \max\limits_{a'} Q(s',a')\Big]$
                        \ENDFOR
                      \ENDFOR
                      \STATE Update $Q(s,a) \leftarrow Q_{\text{new}}(s,a)$ for all $(s,a)$
                    \ENDFOR
                    
                    \STATE \textbf{return} $Q$
                \end{algorithmic}
            \end{algorithm}

            This algorithm is guaranteed to converge asymptotically towards $Q^{*,\Omega^{\xi^*}}$ (i.e., the robust Q-value for the optimal policy in the worst-case MDP for this policy given uncertainty sets ${\cal B}(s,a)$ in each pair $s,a$).  
            For completeness, we  recall the proof of convergence of this algorithm below.

            \paragraph{Convergence of Robust Value Iteration.} 
                Consider a discounted MDP with state space $\mathcal{S}$, action space $\cal A$, reward $r(s,a,s')$ and uncertainty sets $\mathcal{B}(s,a)$.
                Define the robust Bellman operator for action-values:
                \[
                (\mathcal{T} Q)(s,a) = \min_{P \in \mathcal{B}(s,a)} \sum_{s'} P(s'|s,a) \big[ r(s,a,s') + \gamma \max_{a'} Q(s',a') \big], \quad \gamma \in [0,1).
                \]

            \begin{proposition}
                $\mathcal{T}$ is a $\gamma$-contraction in the supremum norm $\|\cdot\|_\infty$, hence it admits a unique fixed point $Q^{*,\Omega^{\xi^*}}$, and iterating
                \[
                Q_{k+1} = \mathcal{T} Q_k
                \]
                converges to $Q^{*,\Omega^{\xi^*}}$ with
                \[
                \| Q_k - Q^{*,\Omega^{\xi^*}} \|_\infty \le \gamma^k \| Q_0 - Q^{*,\Omega^{\xi^*}} \|_\infty.
                \]
            \end{proposition}

            \begin{proof} 
                For any fixed transition $P \in \mathcal{B}(s,a)$, define the classical Bellman operator
                \[
                (\mathcal{T}^P Q)(s,a) = \sum_{s'} P(s'|s,a) \big[ r(s,a,s') + \gamma \max_{a'} Q(s',a') \big].
                \]
                It is well-known that $\mathcal{T}^P$ is $\gamma$-contractant in $\|\cdot\|_\infty$, since for any $Q_1, Q_2$:
                \[
                \| \mathcal{T}^P Q_1 - \mathcal{T}^P Q_2 \|_\infty \le \gamma \| Q_1 - Q_2 \|_\infty.
                \]
                
                \noindent Then, we remark that the robust operator is the pointwise minimum over $P \in \mathcal{B}(s,a)$:
                \[
                (\mathcal{T} Q)(s,a) = \min_{P \in \mathcal{B}(s,a)} (\mathcal{T}^P Q)(s,a).
                \]
                For any two functions $Q_1, Q_2$ and for each $(s,a)$, we thus have
                \begin{eqnarray}
                \big| \min_{P} (\mathcal{T}^P Q_1)(s,a) - \min_{P} (\mathcal{T}^P Q_2)(s,a) \big|
                &\le& \max_{P} \big| (\mathcal{T}^P Q_1)(s,a) - (\mathcal{T}^P Q_2)(s,a) \big| \nonumber\\ &\le&  \gamma \| Q_1 - Q_2 \|_\infty \nonumber.
                \end{eqnarray}
                
                \noindent Therefore, $\mathcal{T}$ is also $\gamma$-contractant in $\|\cdot\|_\infty$.
                
                \noindent By the Banach fixed-point theorem, $\mathcal{T}$ has a unique fixed point $Q^{*,\Omega^{\xi^*}}$ and the iterates $Q_k$ converge to $Q^{*,\Omega^{\xi^*}}$ exponentially fast.
            \end{proof}

        \subsubsection{$\alpha$-reward-preserving RVI for dynamics sa-rectangular attacks with known MDP}
            \label{sec:preservingRVI}
            
            In this section, we give an extension of the classical RVI presented in previous section, for our case of $\alpha$-reward-preserving attacks. First, we can remark that the bound $\hat{Q}_\alpha(s,a)= (1-\alpha) Q^{*,\Omega^{\xi^*}}(s,a) + \alpha Q^{*,\Omega}(s,a)$  that is used to define $\alpha$-reward-preserving sets as defined in definition \ref{def_preserv} only uses static components, which can be obtained by an initial application of the classical value iteration algorithm (i.e., for $Q^{*,\Omega}$) and the classical robust value iteration algorithm (i.e., for $Q^{*,\Omega^{\xi^*}}$) respectively. This allows to get thresholds on admissible Q-values during the process. However, explicitly characterizing the uncertainty sets $\Xi_\alpha(s,a) \subseteq \mathcal{B}(s,a)$ using these bounds is intractable, as it would require solving a complete robust planning problem for each candidate perturbed transition model.
            
            Rather, we adopt an incremental process in our proposed algorithm \ref{algo:preservRVI}, that employs a two-timescale stochastic approximation
            \citep{borkar1997stochastic}, in order to craft attacks that lie in the convex core ${\cal B}_\alpha$ of $\Xi_\alpha$ (see section \ref{sec:decomposition}), by progressively tuning admissible magnitudes of attacks throughout the process, starting with all magnitudes $\eta(s,a)$ set to $\eta_{\cal B}$. On the fast timescale, Q-values are updated, with a decreasing rate $c_k$, using a robust Bellman operator restricted 
            to transition kernels within distance $\eta(s,a)$ of the nominal dynamics (given a specific metric $d(.,P)$ regarding nominal dynamics $P$). 
            On the slow timescale, the admissible magnitudes $\eta(s,a)$ are adjusted so as to enforce
            the $\alpha$-reward-preserving constraint defined by $\hat{Q}_\alpha(s,a)$. At each step, every $\eta(s,a)$  is updated with quantity $\beta_k \Delta_{\eta(s,a)}$, with $\beta_k$ a decreasing learning rate and $\Delta_{\eta(s,a)}$ the update direction. In algorithm \ref{algo:preservRVI}, we set $\Delta_{\eta(s,a)}=-1$ if the attack for $(s,a)$ is too strong regarding the threshold $\hat{Q}_\alpha$ (i.e., the result of the robust bellman operator with magnitude $\eta(s,a)$ is lower than $\hat{Q}_\alpha$) and $\Delta_{\eta(s,a)}=-1$ otherwise. 
            
            This separation of timescales allows the Q-iterates to track, almost surely, the fixed point
            of the robust Bellman operator corresponding to quasi-static values of $\eta$,
            while $\eta$ is progressively tuned to identify the largest admissible subset of
            $\mathcal B$ compatible with $\alpha$-reward preservation.

            The stepsizes $(c_k)$ and $(\beta_k)$ are chosen to satisfy the Robbins--Monro conditions
            \citep{robbins1951stochastic}, with an explicit two-timescale separation:
            \[
            \sum_k c_k = \infty, \quad \sum_k c_k^2 < \infty, \qquad
            \sum_k \beta_k = \infty, \quad \sum_k \beta_k^2 < \infty, \qquad
            \frac{\beta_k}{c_k} \to 0.
            \]
            A canonical choice is
            \[
            c_k = k^{-\tau_Q}, \qquad \beta_k = k^{-\tau_\eta},
            \]
            with $1/2 < \tau_Q < \tau_\eta \le 1$.
            Such two-timescale schemes are standard in stochastic approximation and reinforcement learning
            \citep{borkar1997stochastic,konda1999actor}. Under these conditions, the fast-timescale recursion sees $\eta(s,a)$ as quasi-static, 
            and the $Q$-updates track the fixed point of the robust Bellman operator associated with the current admissible radius $\eta(s,a)$.  
            Conversely, the slow-timescale recursion for $\eta(s,a)$ sees the $Q$-values as essentially equilibrated, 
            and adjusts the admissible attack magnitude so as to satisfy the $\alpha$-reward-preservation constraints.

            Because of the nonlinear and nonconvex mapping $\eta \mapsto Q^{*,\eta}$ (arising from the $\arg\min$ over adversarial transitions and the $\max$ over actions), global convergence cannot be guaranteed. However, by classical results on two-timescale stochastic approximation \citep{borkar2008stochastic}, the joint iterates $(Q_k, \eta_k)$ converge to \textit{locally stable equilibria}, corresponding to solutions that satisfy the $\alpha$-reward-preserving constraints.  While we do not claim global optimality, these locally stable solutions ensure that the resulting policy respects the desired fraction of reward preservation and adapts to safer regions of the state-action space.

            \begin{algorithm}[!ht]
                \caption{$\alpha$-reward-preserving Robust Value Iteration with $r(s,a,s')$ and action-value function $Q$}
                \label{algo:preservRVI}
                \begin{algorithmic}[1]
                    \REQUIRE State space $\mathcal{S}$, action space $\cal A$, reward $r(s,a,s')$, nominal MDP $\Omega$,
                    uncertainty set $\mathcal{B}(s,a)$, discount $\gamma$, preservation rate $\alpha$, scheduled magnitude learning rates $\beta_k$, scheduled Q update rates $c_k$.
                    \ENSURE Robust $\alpha$-reward-preserving Q-function $Q^{*,\Omega^{\xi_\alpha^*}}$
                    
                    \STATE Compute $Q^{*,\Omega}$ using classical Value Iteration until convergence
                    \STATE Compute $Q^{*,\Omega^{\xi^*}}$ using Robust Value Iteration until convergence
                    \STATE Define $\alpha$-reward-preserving uncertainty sets $\Xi_\alpha(s,a) \subseteq {\cal B}(s,a)$ implicitly, by computing the thresholds:
                    \STATE $$\hat{Q}_\alpha(s,a)= (1-\alpha) Q^{*,\Omega^{\xi^*}}(s,a) + \alpha Q^{*,\Omega}(s,a)$$
                    \STATE Initialize $Q(s,a) \leftarrow 0$ for all $(s,a)$
                    \STATE Initialize $\eta(s,a) \leftarrow \eta_{\cal B}$ for all $(s,a)$
                    
                    \FOR{each iteration $k = 0,1,2,\dots$}
                      \FOR{each state $s \in \mathcal{S}$}
                        \FOR{each action $a \in \cal A$}
                          \STATE Compute worst-case expected Q over uncertainty sets of radius $\eta(s,a)$:
                          \STATE $Q_{\text{new}}(s,a) \leftarrow
                          \min\limits_{\substack{P^\xi \in {\cal B}(s,a)\\ d(P^\xi,P)\leq \eta(s,a)}}
                          \sum\limits_{s'} P^\xi(s' \mid s,a)\Big[r(s,a,s') + \gamma \max\limits_{a'} Q(s',a')\Big]$
                    
                          \STATE $\Delta_{\eta(s,a)} \leftarrow +1$
                          \IF{$Q_{\text{new}}(s,a) < \hat{Q}_{\alpha}(s,a)$}
                            \STATE $\Delta_{\eta(s,a)} \leftarrow -1$
                          \ENDIF
                          \STATE $\eta(s,a) \leftarrow \eta(s,a) + \beta_k \Delta_{\eta(s,a)}$
                        \ENDFOR
                      \ENDFOR
                      \STATE Update $Q(s,a) \leftarrow Q(s,a) + c_k\big(Q_{\text{new}}(s,a) - Q(s,a)\big)$ for all $(s,a)$
                    \ENDFOR
                    
                    \STATE \textbf{return} $Q$
                \end{algorithmic}
            \end{algorithm}

    \subsection{Robust Value Iteration and $\alpha$-reward-preserving RVI via Sinkhorn in Gridworlds}
        \label{sec:gridworld} 
        
        This section details the experimental setting used to produce figure \ref{fig:preserve_RVI}, in particular the way worst-case attacks are computed over each uncertainty set in our experimental design for dynamics attacks in the tabular setting with known nominal dynamics. The environment is a deterministic discrete Gridworld where the agent navigates from the bottom-left corner to the upper-left corner, receiving a reward $+1$, while avoiding terminal states with reward $-1$ along its path.
        
        We consider a robust Markov Decision Process (RMDP) on this environment, under uncertainty in the transition dynamics, modeled via a Wasserstein ball of radius $\eta_{\cal B}$ around the nominal distribution.

        \paragraph{Classical Robust Minimization.} 
            For each state-action pair $(s,a)$, the robust Q-value is defined by minimizing the expected return over all admissible distributions in the uncertainty set:
            \begin{equation}
                Q_{\text{robust}}(s,a) = \min_{P^\xi \in {\cal B}(s,a)} \sum_{s'} P^\xi(s'|s,a) \big[r(s,a,s') + \gamma \max_{a'} Q(s',a') \big],
            \end{equation}
            subject to
            \begin{equation}
                d(P^\xi,P) \le \eta(s,a),
            \end{equation}
            where $d(P^\xi,P)$ is the Wasserstein-2 distance with squared Euclidean transport costs between successor states.  
            
            While this formulation directly captures the worst-case expectation, it is computationally challenging: multiple distributions may achieve the same minimal value, leading to discontinuities and instability in iterative algorithms.
            \paragraph{Approximate Optimization via Entropic Sinkhorn.} 
            To address this issue, we replace the classical robust minimization with a smooth entropic-regularized optimal transport problem. For each state-action pair $(s,a)$, the worst-case transition distribution is approximated by solving the following optimization over transport plans:
            \begin{equation}
                \begin{aligned}
                    \pi^* = \arg\min_{\pi \ge 0}
                    \;&
                    \sum_{s',s''} \pi(s',s'') \Big[ V' + \omega D(s',s'') \Big]
                    - \lambda \sum_{s',s''} \pi(s',s'') \log \pi(s',s'')
                    \\
                    \text{s.t.}\;&
                    \sum_{s'} \pi(s',s'') = p_0(s''|s,a),
                \end{aligned}
            \end{equation}
            where $\pi(s',s'')$ is the transport plan from $s''$ to $s'$ and 
            $V' = r(s,a,s') + \gamma V(s')$, $p_0(\cdot|s,a)$ is the nominal transition distribution, $D$ is the squared Euclidean distance between successor states, $\lambda>0$ is the entropic regularization parameter, and $\omega = 1/\eta(s,a)$ controls the strength of the transport cost.
            
            The resulting worst-case transition distribution is given by the first marginal of the optimal transport plan:
            \begin{equation}
                p^*(s'|s,a) = \sum_{s''} \pi^*(s',s'').
            \end{equation}
            
            This formulation yields a smooth approximation of the original Wasserstein-robust minimization, avoiding discontinuities caused by multiple equivalent worst-case distributions while still favoring transitions toward low-value successor states.

        \paragraph{Sinkhorn Iterations.}
            The entropic optimal transport problem is solved via the classical Sinkhorn algorithm. Let
            \begin{equation}
                K = \exp\!\left(- \omega D / \lambda \right)
            \end{equation}
            be the entropic transport kernel, and
            \begin{equation}
                w = \exp(- V / \lambda)
            \end{equation}
            encode the influence of the value function on the final distribution.
            
            The algorithm introduces two multiplicative vectors \(u\) and \(v\) to enforce the marginal constraints and incorporate the value weighting:
            \begin{itemize}
                \item \(v\) adjusts the plan to satisfy the fixed nominal marginal \(p_0\) (i.e., the distribution from which the mass originates),
                \item \(u\) adjusts the plan to produce a final marginal weighted by \(V\), yielding the soft-min distribution.
            \end{itemize}
            
            Starting from \(u = \mathbf{1}\) and \(v = \mathbf{1}\), the updates are iterated as
            \begin{align}
                u &\gets \frac{w}{K v + \epsilon}, \\
                v &\gets \frac{p_0}{K^\top u + \epsilon},
            \end{align}
            where \(\epsilon > 0\) prevents division by zero. After convergence, the robust transition distribution is recovered as the first marginal of the transport plan:
            \begin{equation}
                p^* = \frac{u \cdot (K v)}{\sum_{s'} u \cdot (K v)}.
            \end{equation}
            
            Intuitively, the alternating updates of \(u\) and \(v\) ensure that \(p^*\) both respects the nominal distribution \(p_0\) and shifts probability mass toward low-value successor states, while maintaining smoothness due to the entropic regularization.
            
            Finally, this $p^*$ distribution is the transition kernel used in the robust Bellman operators in algorithms \ref{algo:classicalRVI} and \ref{algo:preservRVI}.

    \subsection{Properties of $\alpha$-Reward-Preserving MDPs}
        \label{sec:properties}
        
        \begin{property} \textbf{Reward Structure Preservation} \\
            Suppose that for a sufficiently large uncertainty set ${\cal B}$, $Q^{*, \xi^*}$ is equal to a given constant minimal value $Rmin$ for every state $s \in {\cal S}$ and action $a \in {\cal A}(s)$ (i.e., the worst-case attacks fully destroy the reward signal). In that setting, worst-case  $\alpha$-reward-preserving attacks preserve the structure of the reward: $\forall ((s,a),(s',a')) \in ({\cal S}\times {\cal A})^2: Q^{*,\Omega}(s,a) > Q^{*,\Omega}(s',a') \implies Q^{*,\Omega^{\xi_\alpha^{*}}}(s,a) > Q^{*,\Omega^{\xi_\alpha^{*}}}(s',a')$ 
        \end{property}
        
        \begin{proof}
            Assuming that for a given large uncertainty set ${\cal B}$, the worst case attack fully destroys the reward signal. That is, $Q^{*, \xi^*}(s,a)=Rmin$ for every state $s \in {\cal S}$ and action $a \in {\cal A}(s)$. In that setting, we get for every $s \in {\cal S}$ and $a \in {\cal A}(s)$ that:
            \begin{multline}
                \Xi_\alpha(s,a) := \Big\{ \xi\in\mathcal B(s,a) \;:\;  \\  Q^{*,\Omega^\xi}(s,a) \geq Q^{*,\Omega^{\xi^*}}(s,a) + \alpha \left( Q^{*,\Omega}(s,a) - Q^{*,\Omega^{\xi^*}}(s,a) \right) \Big\}\nonumber, 
            \end{multline}
            can be rewritten as:
            \begin{multline}
                \Xi_\alpha(s,a) := \Big\{ \xi\in\mathcal B(s,a) \;:\;   Q^{*,\Omega^\xi}(s,a)\geq Rmin + \alpha \left( Q^{*,\Omega}(s,a) - Rmin \right) \Big\}\nonumber, 
            \end{multline}
            
            \noindent Let us define for any $(s,a)$: $$\hat{Q}(s,a):=Q^{*,\Omega^{\xi^*}}(s,a) + \alpha \left( Q^{*,\Omega}(s,a) - Q^{*,\Omega^{\xi^*}}(s,a)  \right)$$. 
            
            \noindent Contrary to the general case where we cannot guarantee that there exists an attack $\xi_\alpha$ from $\Xi_\alpha$ that respects $Q^{*,\Omega^\xi}(s,a)=\hat{Q}(s,a)$ for any $(s,a)$, we show below that this is the case when $Q^{*,\Omega^{\xi^*}}(s,a)=R{min}$. 
            
            \noindent In that setting we have for any $(s,a)$: 
            \begin{eqnarray*}
                \hat{Q}(s,a) & = & Q^{*,\Omega^{\xi^*}}(s,a) + \alpha \left( Q^{*,\Omega}(s,a) - Q^{*,\Omega^{\xi^*}}(s,a)  \right) \\
                &=&Rmin + \alpha \left( Q^{*,\Omega}(s,a) - Rmin \right) \\
                &=&(1-\alpha) Rmin + \alpha  Q^{*,\Omega}(s,a)
            \end{eqnarray*}
            
            Thus, we have: $ Q^{*,\Omega} = (\hat{Q}(s,a) - (1-\alpha) Rmin)/\alpha$, and thus, using the fixed point property of the optimal bellman operator for $Q^{*,\Omega}$: 
            \begin{eqnarray*}
                \hat{Q}(s,a) & = & (1-\alpha) Rmin + \alpha  \mathbb{E}_{s' \sim \Omega} \left[R(s,a,s') + \gamma \max_{a'} Q^{*,\Omega}(s',a') \right] \\
                & = & (1-\alpha) Rmin + \alpha  \mathbb{E}_{s' \sim \Omega} \left[R(s,a,s') + \gamma \max_{a'} \frac{(\hat{Q}(s',a') - (1-\alpha) Rmin)}{\alpha}\right] \\
                & = & (1-\alpha) (1-\gamma) Rmin +  \mathbb{E}_{s' \sim \Omega} \left[ \alpha R(s,a,s') + \gamma \max_{a'} \hat{Q}(s',a')\right] \\
                & = & \mathbb{E}_{s' \sim \Omega} \left[ \hat{R}_\alpha(s,a,s') + \gamma \max_{a'} \hat{Q}(s',a')\right] \\
            \end{eqnarray*}
            where $\hat{R}_\alpha(s,a,s'):=\alpha R_\alpha(s,a,s')+(1-\alpha) (1-\gamma) Rmin  $. Thus, the use of an $\alpha$-reward-preserving attack in large ${\cal B}$ comes down to acting in the nominal MDP with rescaled rewards. 
            
            Since $\hat{Q}(s,a)$ is the lower bound of the Q-value for any $\xi \in \Xi_\alpha$, and since %, using classical contraction properties for ,
            it can be reached for all $(s,a)$ using iterative classical bellman updates using rescaled rewards for $\gamma \in [0;1)$, we can state that all  $\xi^*_\alpha \in \Xi_\alpha^{*,*}(s,a)$ respect $Q^{*,\Omega^{\xi_\alpha^*}}(s,a)=\hat{Q}(s,a)$.   
            
            To conclude, we can remark that the reward rescaling $\hat{R}_\alpha$ is the same for any $(s,a)$. Thus,  if for any pair $(s,a)$ and $(s',a')$ we have that $Q^{*,\Omega}(s,a)>Q^{*,\Omega}(s',a')$, we also have  $\hat{Q}(s,a) > \hat{Q}(s',a')$, or equivalently $Q^{*,\Omega^{\xi_\alpha^{*}}}(s,a) > Q^{*,\Omega^{\xi_\alpha^{*}}}(s',a')$. 
            
            For sufficiently large sets ${\cal B}$, $\alpha$-reward-preserving attacks preserve the structure of the reward. There exists an optimal policy for $\Omega^{\xi_\alpha^*}$ that is also the optimal policy for the  nominal MDP $\Omega$.    
        \end{proof}

        \begin{property}\textbf{Condition for Preferred State–Action Change}  
            Consider two state–action pairs $(s,a)\in\mathcal{S}\times{\cal A}(s)$ and $(s',a')\in\mathcal{S}\times{\cal A}(s')$.  
            Assume that, in the nominal MDP $\Omega$, $(s,a)$ is preferred to $(s',a')$, i.e., 
            $d_\Omega((s,a),(s,a')) := Q^{*,\Omega}(s,a) - Q^{*,\Omega}(s',a') > 0$.  
            Under any worst-case $\alpha$-reward-preserving attack $\xi^*_\alpha$ defined for a given  uncertainty set ${\cal B}$, the preference is reversed — namely, $(s',a')$ becomes preferred to $(s,a)$ (i.e., $Q^{*,\Omega^{\xi^*_\alpha}}(s',a') > Q^{*,\Omega^{\xi^*_\alpha}}(s,a)$) — if and only if
            \[
                d_{\Omega^{\xi^*}}((s',a'),(s,a)) \;>\; \frac{\alpha}{1-\alpha}\, d_\Omega((s,a),(s',a')) \;+\;
                \delta((s',a'),(s,a)), 
            \]
            where $d_{\Omega^{\xi^*}}((s',a'),(s,a)):=Q^{*,\Omega^{\xi^*}}(s',a')-Q^{*,\Omega^{\xi^*}}(s,a)$, and $\delta((s',a'),(s,a)):=\frac{(\epsilon_{s',a'}-\epsilon_{s,a})}{1-\alpha}$, with $\epsilon_{s,a}$ the gap between $Q^{*,\Omega^{\xi_\alpha^*}}(s,a)$ and its $\alpha$-reward-preserving lower-bound $\hat{Q}(s,a):=(1-\alpha) Q^{*,\Omega^{\xi_\alpha^*}}(s,a) + \alpha  Q^{*,\Omega}(s,a)$.
            Defining the total-variation diameter of $\mathcal B$ at any $(s,a)$ by $\eta_{\mathcal B}$, standard Lipschitz bounds imply $\delta((s',a'),(s,a)) = \mathcal{O}(\eta_{\mathcal B})$. While $\delta((s',a'),(s,a)) \to 0$ as $\eta_{\mathcal B}\to 0$, the actual variation of $Q^{*,\Omega^{\xi^*}}(s,a)$ can be amplified by local gaps in successor actions, so $\delta$ variations may be dominated by the effective sensitivity of $Q$ in “dangerous” zones, which induce preference changes under $\alpha$-reward-preserving attacks. 
        \end{property}

        \begin{proof}
            By definition of an $\alpha$-reward-preserving attack $\xi_\alpha^*$, we have for any state-action pair $(x,u)$:
            \[
            Q^{*,\Omega^{\xi_\alpha^*}}(x,u) \ge \hat Q(x,u) := (1-\alpha) Q^{*,\Omega^{\xi^*}}(x,u) + \alpha Q^{*,\Omega}(x,u).
            \]
            Define the gap
            \[
            \epsilon_{x,u} := Q^{*,\Omega^{\xi_\alpha^*}}(x,u) - \hat Q(x,u) \ge 0.
            \]
            
            Consider two state-action pairs $(s,a)$ and $(s',a')$. Let
            \[
            \delta((s',a'),(s,a)) := \frac{\epsilon_{s',a'} - \epsilon_{s,a}}{1-\alpha}.
            \]
            
            The preference of $(s',a')$ over $(s,a)$ under $\xi_\alpha^*$ is expressed as
            \[
            Q^{*,\Omega^{\xi_\alpha^*}}(s',a') > Q^{*,\Omega^{\xi_\alpha^*}}(s,a).
            \]
            
            Using the definition of $\epsilon$ and $\hat Q$, we can rewrite this as
            \[
            \hat Q(s',a') + \epsilon_{s',a'} > \hat Q(s,a) + \epsilon_{s,a}.
            \]
            
            Subtracting $\hat Q(s,a)$ from both sides and rearranging terms gives
            \[
            \hat Q(s',a') - \hat Q(s,a) > \epsilon_{s,a} - \epsilon_{s',a'} = - (\epsilon_{s',a'} - \epsilon_{s,a}) = -(1-\alpha)\delta((s',a'),(s,a)).
            \]
            
            By definition of $\hat Q$, we have
            \[
            \hat Q(s',a') - \hat Q(s,a) = (1-\alpha)\big( Q^{*,\Omega^{\xi^*}}(s',a') - Q^{*,\Omega^{\xi^*}}(s,a) \big) + \alpha \big( Q^{*,\Omega}(s',a') - Q^{*,\Omega}(s,a) \big).
            \]
            
            Let $d_\Omega := Q^{*,\Omega}(s,a) - Q^{*,\Omega}(s',a') > 0$. Then
            \[
            Q^{*,\Omega}(s',a') - Q^{*,\Omega}(s,a) = -d_\Omega.
            \]
            
            Plugging this in gives
            \[
            (1-\alpha)\big( Q^{*,\Omega^{\xi^*}}(s',a') - Q^{*,\Omega^{\xi^*}}(s,a) \big) - \alpha d_\Omega > -(1-\alpha) \delta((s',a'),(s,a)).
            \]
            
            Dividing both sides by $(1-\alpha)$ yields the stated condition:
            \[
            d_{\Omega^{\xi^*}}((s',a'),(s,a)) >   \frac{\alpha}{1-\alpha} d_\Omega((s,a),(s',a')) + \delta((s',a'),(s,a)).
            \]
            
            The following of the proof relies on the decomposition of $\epsilon_{s,a}$ relative to the linear lower-bound $\hat Q$:
            \[
            \epsilon_{s,a} = Q^{*,\Omega^{\xi_\alpha^*}}(s,a) - \hat Q(s,a) 
            = (1-\alpha)\big(Q^{*,\Omega^{\xi_\alpha^*}}(s,a) - Q^{*,\Omega^{\xi^*}}(s,a)\big) + \alpha \big(Q^{*,\Omega^{\xi_\alpha^*}}(s,a) - Q^{*,\Omega}(s,a)\big).
            \]
            
            Each term measures the sensitivity of the robust Q-value under $\xi_\alpha^*$ to changes in transitions (or observations) compared to the reference MDPs. Using standard results (e.g., from \citep{wiesemann2013robust}):
            
            \[
            |Q^{*,\Omega^{\xi_\alpha^*}} - Q^{*,\Omega^{\xi^*}}| \le L_1 \eta_\mathcal B, 
            \quad 
            |Q^{*,\Omega^{\xi_\alpha^*}} - Q^{*,\Omega}| \le L_2 \eta_\mathcal B,
            \]
            
            where the constants $L_1,L_2$ depend on $R_{\max}$, $\gamma$, and the propagation of the max over successors. Therefore,
            \[
            |\epsilon_{s,a}| \le (1-\alpha)L_1 \eta_\mathcal B + \alpha L_2 \eta_\mathcal B = \mathcal{O}(\eta_\mathcal B).
            \]
            
            Consequently, 
            \[
            \delta((s',a'),(s,a)) = \frac{\epsilon_{s',a'} - \epsilon_{s,a}}{1-\alpha} = \mathcal{O}(\eta_\mathcal B),
            \]
            and $\delta((s',a'),(s,a)) \to 0$ as $\eta_\mathcal B \to 0$. 
            
            \medskip
            \noindent
            \textbf{Remark:} The bound holds for attacks on transitions (SA-rectangular), on observations, or combined, as long as $\eta_\mathcal B$ correctly measures the total-variation diameter of the perturbed distributions. For non-SA-rectangular transition sets, the same order-of-magnitude bound remains an approximation; correlations between state-action perturbations may amplify the effective variation of $Q$, so $\delta$ can underestimate local sensitivity in “dangerous zones.”
        \end{proof}
        
\medskip
    
    \subsection{Approximated Reward-Preserving Robust Deep RL Algorithms}
    \label{DeepPreserve}
    
    This section presents the complete robust training procedures proposed in this work, for both dynamics and observation attacks, following all the main steps described in section \ref{sec:deep_approx} of the paper, and used for our experiments. 
    
        \begin{algorithm}[H]
            \caption{$\alpha$-Reward-Preserving Training (on dynamics)}
            \label{algo:approx_dynamics}
            \begin{algorithmic}[1]
                \REQUIRE Environment with unknown dynamics $\Omega$; Maximal attack magnitude $\eta_{\cal B}$; Reward preservation rate $\alpha$; Protagonist agent $\pi_{\theta}$; discount $\gamma$; Dynamic Q-network $Q_{\psi_\alpha}$; Static Q-network $Q_{\psi_c}$; Tail of the magnitude distribution $\epsilon$; Capacity of the replay buffer $c$; Number of cycles $nb\_cycles$; Batch size $bsize$; Number of steps $nsteps$; Number of Q updates per cycle $nQiter$; Attack direction crafter $\xi^A_{\tilde{\pi}}$; Polyak update parameters $\tau^{\pi}$ and $\tau^{Q}$; Learning rates $\beta^\pi$ and $\beta^Q$.
                
                \STATE $\tilde{\theta} \leftarrow copy(\theta)$ \hfill \textit{// Reference policy $\tilde{\pi}$}
                \STATE $\hat{\psi}_\alpha \leftarrow copy(\psi_\alpha)$ \hfill \textit{// Target Dynamic Q-network $Q^{\tilde{\pi}}_\alpha$}
                \STATE $\hat{\psi}_c \leftarrow copy(\psi_c)$ \hfill \textit{// Target Static Q-network $Q^{\tilde{\pi}}_c$}
                \STATE Initialize buffer $B \leftarrow \emptyset$
                
                \FOR{each $nb\_cycles$ cycles}
                  \STATE \textit{// Collect $nsteps$ transitions in $\Omega^{\xi_\alpha}$}
                  \STATE $s' \leftarrow$ first state from $\Omega$
                  \FOR{$nsteps$ environment interactions}
                    \STATE $s \leftarrow s'$
                    \STATE Sample $a \sim \pi_\theta(\cdot \mid s)$
                    \STATE \textit{// $\alpha$-reward-preserving set}
                    \STATE Define $\tilde{\mathcal B}^\eta_\alpha(s,a)$ using \eqref{xitilde} with $Q_{\psi_\alpha}$ and $Q_{\psi_c}$
                    \STATE $\eta^*(s,a) \leftarrow \arg\min\limits_{\eta \in \tilde{\mathcal B}^\eta_\alpha(s,a)} Q_\alpha^{\tilde{\pi}}((s,a),\eta)$
                    \STATE $\lambda \leftarrow -\log(\epsilon) / \eta^*(s,a)$
                    \STATE Sample $\eta \sim p^{\tilde{\pi}}_\alpha(\cdot \mid s,a) \propto \lambda e^{-\lambda \eta}$
                    \STATE Craft attack direction $A \leftarrow \xi_{\tilde{\pi}}^A(s,a)$ for $\tilde{\pi}$
                    \STATE $(s', r, done) \leftarrow \Omega^{\xi=(\eta,A)}(s,a)$ \hfill \textit{// Perform adversarial step}
                    \STATE $B \leftarrow B \cup \{(s,a,\eta,A,r,s', \pi_{\theta}(a\mid s))\}$
                  \ENDFOR
                
                  \STATE Improve $\pi_\theta$ on new transitions using any RL algorithm (e.g., SAC)
                
                  \FOR{$nQiter$ iterations}
                    \STATE Sample a batch $\{(s_i,a_i,\eta_i,A_i,r_i,s'_i,p_i)\}_{i=1}^{bsize}$ from $B$
                    \STATE Sample next actions $\{a'_i \sim \pi_{\tilde{\theta}}(\cdot \mid s_i)\}_{i=1}^{bsize}$
                    \STATE $\{w_i\}_{i=1}^{bsize} \leftarrow \{\dfrac{\pi_{\tilde{\theta}}(a_i|s_i)}{p_i}\}_{i=1}^{bsize}$ \hfill \textit{// Importance Sampling Weights}
                    \STATE $\{\delta^\alpha_i\}_{i=1}^{bsize} \leftarrow \{Q_{\psi_\alpha}((s_i,a_i),\eta_i) - r_i -\gamma  \mathbb{E}_{\eta'_i}\left[ Q_{\tilde{\psi}_\alpha}((s'_i,a'_i),\eta'_i)\right]\}_{i=1}^{bsize}$
                    \STATE $\{\delta^c_i\}_{i=1}^{bsize} \leftarrow \{Q_{\psi_c}((s_i,a_i),\eta_i) - r_i -\gamma   Q_{\tilde{\psi}_c}((s'_i,a'_i),\eta_i)\}_{i=1}^{bsize}$
                    \STATE $\psi_\alpha \leftarrow \psi_\alpha - \beta^Q  \sum_{i=1}^{bsize} \nabla_{\psi_\alpha} w_i (\delta_i^\alpha)^2$
                    \STATE $\psi_c \leftarrow \psi_c - \beta^Q  \sum_{i=1}^{bsize} \nabla_{\psi_c} w_i (\delta_i^c)^2$
                  \ENDFOR
                  \STATE $\tilde{\theta} \leftarrow (1 - \tau^{\pi})\,\tilde{\theta} + \tau^{\pi}\,\theta$ \hfill \textit{// Polyak update of $\tilde{\pi}$}
                  \STATE $\hat{\psi}_\alpha \leftarrow (1 - \tau^{Q})\,\hat{\psi}_\alpha + \tau^{Q}\,\psi_\alpha$ \hfill \textit{// Polyak update of $Q^{\tilde{\pi}}_\alpha$}
                  \STATE $\hat{\psi}_c \leftarrow (1 - \tau^{Q})\,\hat{\psi}_c + \tau^{Q}\,\psi_c$ \hfill \textit{// Polyak update of $Q^{\tilde{\pi}}_c$}
                \ENDFOR
                
                \STATE \textbf{return} $Q$
            \end{algorithmic}
        \end{algorithm}

        \begin{algorithm}[H]
            \caption{$\alpha$-Reward-Preserving Training (on observations)}
            \label{algo:approx_observations}
            \begin{algorithmic}[1]
                \REQUIRE Environment with unknown dynamics $\Omega$; Maximal attack magnitude $\eta_{\cal B}$; Reward preservation rate $\alpha$; Protagonist agent $\pi_{\theta}$; discount $\gamma$; Dynamic Q-network $Q_{\psi_\alpha}$; Static Q-network $Q_{\psi_c}$; Tail of the magnitude distribution $\epsilon$; Capacity of the replay buffer $c$; Number of cycles $nb\_cycles$; Batch size $bsize$; Number of steps $nsteps$; Number of Q updates per cycle $nQiter$; Attack direction crafter $\xi^A_{\tilde{\pi}}$; Polyak update parameters $\tau^{\pi}$ and $\tau^{Q}$; Learning rates $\beta^\pi$ and $\beta^Q$.
                
                \STATE $\tilde{\theta} \leftarrow copy(\theta)$ \hfill \textit{// Reference policy $\tilde{\pi}$}
                \STATE $\hat{\psi}_\alpha \leftarrow copy(\psi_\alpha)$ \hfill \textit{// Target Dynamic Q-network $Q^{\tilde{\pi}}_\alpha$}
                \STATE $\hat{\psi}_c \leftarrow copy(\psi_c)$ \hfill \textit{// Target Static Q-network $Q^{\tilde{\pi}}_c$}
                \STATE Initialize buffer $B \leftarrow \emptyset$
                
                \FOR{each $nb\_cycles$ cycles}
                  \STATE \textit{// Collect $nsteps$ transitions in $\Omega^{\xi_\alpha}$}
                  \STATE $s' \leftarrow$ first state from $\Omega$
                  \FOR{$nsteps$ environment interactions}
                    \STATE $s \leftarrow s'$
                    \STATE \textit{// $\alpha$-reward-preserving set}
                    \STATE Define $\tilde{\mathcal B}^\eta_\alpha(s)$ using \eqref{xitilde} with $Q_{\psi_\alpha}$ and $Q_{\psi_c}$
                    \STATE $\eta^*(s) \leftarrow \arg\min\limits_{\eta \in \tilde{\mathcal B}^\eta_\alpha(s)} Q_\alpha^{\tilde{\pi}}((s,a),\eta)$
                    \STATE $\lambda \leftarrow -\log(\epsilon) / \eta^*(s)$
                    \STATE Sample $\eta \sim p^{\tilde{\pi}}_\alpha(\cdot \mid s, a) \propto \lambda e^{-\lambda \eta}$
                    \STATE Craft attack direction $A \leftarrow \xi_{\tilde{\pi}}^A(s,a)$ for $\tilde{\pi}$
                    \STATE Sample $a \sim \pi_\theta(\cdot \mid \phi(s) + \eta A)$
                    \STATE $s', r, done \leftarrow \Omega^{\xi=(\eta,A)}(s)$ \hfill \textit{// Perform adversarial step}
                    \STATE $B \leftarrow B \cup \{(s,a,\eta,A,r,s', \pi_{\theta}(a\mid \phi(s)+\eta A))\}$
                  \ENDFOR
                
                  \STATE Improve $\pi_\theta$ on new transitions using any RL algorithm (e.g., SAC)
                
                  \FOR{$nQiter$ iterations}
                    \STATE Sample a batch $\{(s_i,a_i,\eta_i,A_i,r_i,s'_i,p_i)\}_{i=1}^{bsize}$ from $B$
    
                    \STATE $\{w_i\}_{i=1}^{bsize} \leftarrow \{\dfrac{\pi_{\tilde{\theta}}(a_i|\phi(s_i)+\eta_i A_i)}{p_i}\}_{i=1}^{bsize}$ \hfill \textit{// Importance Sampling Weights}
                    \STATE $\{\delta^\alpha_i\}_{i=1}^{bsize} \leftarrow \{Q_{\psi_\alpha}(s_i,\eta_i) - r_i -\gamma  \mathbb{E}_{\eta'_i}\left[ Q_{\tilde{\psi}_\alpha}(s'_i,\eta'_i)\right]\}_{i=1}^{bsize}$
                    \STATE $\{\delta^c_i\}_{i=1}^{bsize} \leftarrow \{Q_{\psi_c}(s_i,\eta_i) - r_i -\gamma   Q_{\tilde{\psi}_c}(s'_i,\eta_i)\}_{i=1}^{bsize}$
                    \STATE $\psi_\alpha \leftarrow \psi_\alpha - \beta^Q  \sum_{i=1}^{bsize} \nabla_{\psi_\alpha} w_i (\delta_i^\alpha)^2$
                    \STATE $\psi_c \leftarrow \psi_c - \beta^Q  \sum_{i=1}^{bsize} \nabla_{\psi_c} w_i (\delta_i^c)^2$
                  \ENDFOR
                
                  \STATE $\tilde{\theta} \leftarrow (1 - \tau^{\pi})\,\tilde{\theta} + \tau^{\pi}\,\theta$ \hfill \textit{// Polyak update of $\tilde{\pi}$}
                  \STATE $\hat{\psi}_\alpha \leftarrow (1 - \tau^{Q})\,\hat{\psi}_\alpha + \tau^{Q}\,\psi_\alpha$ \hfill \textit{// Polyak update of $Q^{\tilde{\pi}}_\alpha$}
                  \STATE $\hat{\psi}_c \leftarrow (1 - \tau^{Q})\,\hat{\psi}_c + \tau^{Q}\,\psi_c$ \hfill \textit{// Polyak update of $Q^{\tilde{\pi}}_c$}
                \ENDFOR
                
                \STATE \textbf{return} $Q$
            \end{algorithmic}
        \end{algorithm}

\end{document}